\documentclass[twoside,11pt]{article}

\usepackage[preprint]{jmlr2e_new}

\usepackage{graphicx}
\usepackage[ruled,vlined]{algorithm2e}
\usepackage[english]{babel} 
\usepackage{bookmark}
\usepackage{graphicx} 
\usepackage{hyperref} 
\usepackage{natbib}
\usepackage{thmtools} 
\usepackage{thm-restate}

\usepackage{amsmath,amssymb,bbm,amsfonts} 
\usepackage{url}
\usepackage[toc,page]{appendix}
\usepackage{color}
\usepackage{comment}
\usepackage{booktabs}
\usepackage{xcolor}
\usepackage{IEEEtrantools}
\usepackage{booktabs,subcaption,dcolumn,tabularx}

\newcommand{\new}[1]{{\color{blue}}}

\newcommand{\bbP}{\mathbb{P}}
\newcommand{\bbR}{\mathbb{R}}

\newcommand{\bbN}{\mathbb{N}}
\newcommand{\bbE}{\mathbb{E}}

\newcommand{\bbC}{\mathbb{C}}
\newcommand{\calX}{\mathcal{X}}
\newcommand{\calY}{\mathcal{Y}}

\newcommand{\calQ}{\mathcal{Q}}

\newcommand{\KL}{\operatorname{KL}}

\usepackage{hyperref}
\usepackage{cleveref}

\renewcommand{\Pr}{\operatorname{Proj}}

\DeclareMathOperator*{\argmin}{arg\,min}
\makeatother

\makeatletter
\newcommand{\myitem}[1]{%
\item[#1]\protected@edef\@currentlabel{#1}%
}
\makeatother

\definecolor{blue_plot}{RGB}{55,  126, 184}
\definecolor{orange_plot}{RGB}{255, 127, 0}
\definecolor{green_plot}{RGB}{77,  175, 74}
\definecolor{pink_plot}{RGB}{247, 129, 191}

\usepackage{lastpage}
\jmlrheading{23}{2022}{1-\pageref{LastPage}}{1/21; Revised 5/22}{9/22}{21-0000}{Veit Wild, James Wu, Dino Sejdinovic, and Jeremias Knoblauch}

\ShortHeadings{Near-Optimal Approximations for Function Space Inference}{}
\firstpageno{1}

\author{\name Veit D. Wild \email veit.wild@stats.ox.ac.uk \\
       \addr Department of Statistics, University of Oxford, UK \\
\name James Wu \email jswu18@gmail.com \\        
        \addr Department of Computer Science, University College London, UK;\\
        \addr Gridmatic Inc., Cupertino, CA 95014, USA\\
\name Dino Sejdinovic \email dino.sejdinovic@adelaide.edu.au  \\
       \addr School of Computer and Mathematical Sciences, University of Adelaide, AUS \\
\name Jeremias Knoblauch \email j.knoblauch@ucl.ac.uk  \\
       \addr Department of Statistical Sciences, University College London, UK
       }

\begin{document}

\title{Near-Optimal Approximations for \\ Bayesian Inference in Function Space}

       
\editor{}

\maketitle

\begin{abstract}

We propose a scalable inference algorithm for Bayes posteriors defined on a reproducing kernel Hilbert space (RKHS).
Given a likelihood function and a Gaussian random element representing the prior, the corresponding Bayes posterior measure $\Pi_{\text{B}}$ can be obtained as the stationary distribution of an RKHS-valued Langevin diffusion.
We approximate the infinite-dimensional Langevin diffusion via a projection onto the first $M$ components of the Kosambi–Karhunen–Loève expansion.
Exploiting the thus obtained approximate posterior for these $M$ components, we perform inference for $\Pi_{\text{B}}$ by relying on the law of total probability and a sufficiency assumption.
The resulting method scales as $O(M^3+JM^2)$, where $J$ is the number of samples produced from the posterior measure $\Pi_{\text{B}}$.
Interestingly, the algorithm  recovers the posterior arising from the sparse variational Gaussian process (SVGP) \citep{titsias2009variational} as a special case---owed to the fact that the sufficiency assumption underlies both methods.
However, whereas the SVGP is parametrically constrained to be a Gaussian process, our method is based on a non-parametric variational family $\mathcal{P}(\bbR^M)$ consisting of all  probability measures on $\bbR^M$.
As a result, our method is provably close to the optimal $M$-dimensional variational approximation of the Bayes posterior $\Pi_{\text{B}}$ in $\mathcal{P}(\bbR^M)$ for convex and Lipschitz continuous negative log likelihoods, and coincides with SVGP for the special case of a Gaussian error likelihood.
\end{abstract}

\section{Introduction}


Bayesian inference in function spaces crucially relies on the construction of a prior distributions on an infinite-dimensional function space. 
Gaussian processes (GPs) are the standard tool deployed for this purpose, and have proven a valuable framework for principled uncertainty quantification over the past few decades \citep{rasmussen2003gaussian}. 
A notable shortcoming of GPs is their cubic computational cost with respect to data size, and their intractability for the case of non-Gaussian likelihoods.
Numerous approaches have been proposed to overcome these challenges \citep[see e.g.][]{gneiting2002compactly,quinonero2005unifying,chalupka2013framework,wilson2015kernel,gardner2018gpytorch,wang2019exact,liu2020gaussian}.
Amongst them, sparse variational Gaussian processes (SVGPs) \citep{titsias2009variational} arguably are considered to be the gold standard for scalable GP approximations. 
SVGPs rely on the introduction of inducing features that are assumed to follow a Gaussian distribution with learnable variational parameters, and which can be interpreted as an approximate low-dimensional summary of the posterior GP. 
To learn these variational parameters, one then typically performs gradient based maximisation of the Evidence Lower Bound (ELBO) \citep[see e.g.][]{hensman2013gaussian,hensman2015scalable}. 

While SVGPs often perform surprisingly well in practice, their approximation quality is  inadequate whenever the Gaussianity imposed upon the inducing points is too restrictive.
In the hope of overcoming this shortcoming, this paper explores the possibility of performing Bayesian inference by using a version of gradient descent \textit{directly} in the space of probability measures---rather than on variational parameters indexing the family of approximating measures \citep[see][]{wild2023rigorous}.
We achieve this in three steps:
\begin{itemize}
    \item[1.] In \Cref{sec:WGF-for-FSI}, we derive the required Wasserstein Gradient Flow (WGF) \citep{otto2001geometry,ambrosio2005gradient} as a natural analogue for gradient descent over the space of probability measures;
    \item[2.] We augment the resulting infinite-dimensional diffusion via an $M$-dimensional projection operation in \Cref{sec:projected-langevin} to make it computationally tractable;
    \item[3.] We study the consequences of this projection in \Cref{sec:theory}, and find that relative to the theoretically optimal approximation (see \Cref{thm:characterisation-optimal-approximation}), the proposed algorithm is close in Kullback-Leibler divergence (\Cref{thm:PLS-optimal-close}), with a discrepancy that vanishes at least as $\mathcal{O}(M^{-1})$, and even as $\mathcal{O}(\exp(-M))$ for certain special cases of interest (\Cref{lemma:decay-eigenvalues}). 
\end{itemize}   
Before we state the results in full detail, \Cref{sec:setting} elaborates on the setting studied in this paper, and \Cref{sec:road_map} provides a high-level summary of our results.

\section{Bayesian Inference in Function Space}
\label{sec:setting}

Throughout, we will be concerned with Bayesian inference in a function space.
To this end, we have to assign a prior distribution over an  unknown function $f$.
This function $f$ acts as the infinite-dimensional parameter of the likelihood function $p(y_{1:N}|f)$ which we use to describe the distribution of $N$ observations $y_{1:N}=(y_1,\hdots,y_N) \in \bbR^N$.

\subsection{Computation for GP Regression}

Generally speaking, Bayesian inference in function spaces is computationally challenging.
As a result, the most widely studied special case of this setting is Gaussian process (GP) regression \citep{rasmussen2003gaussian}, which leads to comparatively tractable inference algorithms.
As indicated by its name, GP regression is characterised by placing a GP prior over the function $f$.
Most commonly, the likelihood in GP regression is also Gaussian, resulting in the log likelihood
%
%
\begin{align}
    \log p(y_{1:N}|f) =  \sum_{n=1}^N \log \mathcal{N}\big(y_n ;  f(x_n), \sigma^2\big). 
    \nonumber
\end{align}
Here, $\sigma^2 >0$ denotes the variance in the observation noise, $x_{1:N}=(x_1,\hdots,x_N) \in \calX^N$  are known input features, and $\mathcal{N}(y; \mu, \sigma^2)$ denotes a Gaussian density with mean $\mu$ and variance $\sigma^2$ evaluated at $y$. 
Given the GP prior on $f$ and a likelihood such as the above, GP regression finds the Bayes posterior over $f$ given  $y_{1:N}$. 

For the special case of Gaussian likelihoods, the Bayes posterior is available in closed form, and computable via basic linear algebra operations.
While everything would seem to be fine in this setting, there is an important complication: even in this best case scenario, the computational burden of obtaining the exact GP posterior is  of order $\mathcal{O}(N^3)$, and thus prohibitive for even moderately large sample sizes.
As a result, the Bayes posterior will almost always have to be approximated in practice---even if closed forms are available.
Arguably the most common way of doing this is via sparse variational GPs (SVGPs) first proposed by \citet{titsias2009variational}.
SVGPs provide a low-level stochastic summary of the GP using so-called \textit{inducing features}.
These features are assumed to follow a Gaussian distribution with finite-dimensional variational parameters that are typically learnt via gradient-based variational inference techniques \citep[see e.g.][]{hensman2013gaussian,hensman2015scalable}. 
While SVGPs can  perform  well in practice, they also often lead to subpar performance.
One of the main reason is that the Gaussianity assumed for the inducing features is too restrictive, and leads to poor approximations.
More generally speaking,
SVGPs do not provide universal guarantees for their approximation quality, and are really only appropriate if the exact posterior is itself a Gaussian process or close to it.

\subsection{Beyond GPs and SVGPs}

GP regression and its sparse variational approximations are computationally viable but restrictive.
In this paper, we aim at building a more general framework for function space inference.
To this end, we apply the Wasserstein Gradient Flow (WGF) for a variational reformulation of Bayesian inference (see \Cref{sec:opt-cen-perspec}), resulting in an infinite-dimensional version of gradient descent over the space of probability measures on the desired function space $H$.
Like SVGP, this is a variational method. But unlike SVGP, it is \textit{not} parametrically constrained to be a Gaussian process. 
%
As a result, one may hope that the resulting algorithms can provide superior approximation quality---a suspicion we will verify in \Cref{sec:theory}.
Applying the WGF will enforce an important distinction to a GP-centred view on function space inference: to construct it, we require explicit knowledge of the function space $H$ over which the prior and posterior measures are defined.\footnote{
    The necessary calculations exploit various concepts from infinite-dimensional analysis such as Fréchet derivatives or the integration by parts for Gaussian measures on Hilbert spaces (cf. Appendix \ref{ap:sec:WGF}), and thus necessitate precise understanding of the underlying function space $H$. 
}
This is in stark contrast to Bayesian inference schemes with GPs, which only rely on evaluating $f$ at a finite set of points \citep{rasmussen2003gaussian}, and therefore are agnostic to the exact function space $H$ which draws from the GP realise in. 
Consequently, the family of distributions we consider throughout will be Gaussian Random Elements \citep[see][]{van2008stochastic}: unlike for GPs,
the function space on which draws from a Gaussian Random Element (GRE) realise are explicitly part of their definition.

While GREs and GPs are different families of 
measures over function spaces, they have important connections---and we use these to link a sub-class of GREs to GPs in \Cref{sec:cov-op-gp-interpretation}. 
Importantly, our inference algorithm does \textit{not} rely on this connection: it can be employed for Bayesian computation of a class of GP posteriors, but is applicable to more general GRE posteriors  \textit{beyond} GPs.
Thus, we mainly provide the  connection to GPs to make prior specification easier and posterior inferences more interpretable.
Additionally, this will
allow us a natural way for comparing our (more general) approach with existing (more narrowly defined) function space inference algorithms such as SVGPs.
%
%
This special case is particularly interesting  
%
because it leads to a set of optimality results for our method in \Cref{sec:theory}.
In the language of SVGPs, the optimality results show that amongst all possible posteriors that condition on the chosen inducing features, our procedure targets the one that is closest to the exact posterior in Kullback-Leibler divergence up to a vanishingly small error.

\subsection{Contributions}
\label{sec:road_map}

The paper is organised in three  sections, each of which relates to three interrelated contributions: \Cref{sec:WGF-for-FSI} studies the WGF, which is used in \Cref{sec:projected-langevin} to derive a new tractable algorithm for function space inference that we study in \Cref{sec:theory}.
We provide a high-level overview of each of these below.
\begin{itemize}
    \item
    First, we derive the \textbf{WGF} for the optimisation-centric reformulation of Bayesian \textbf{function space inference}, which is discussed in \Cref{sec:WGF-for-FSI}.
    To implement the WGF numerically and allow for a practicable implementation, we  require certain Fr\'echet derivatives to exist, which---perhaps surprisingly---implies that the function space $H$ in which draws from the posterior are realised is a reproducing kernel Hilbert space (RKHS) $H_k$ associated to some kernel $k$.
    At first glance, this may seem to suggest that it would be natural to specify priors and posteriors to be GPs associated to the kernel $k$, but this is a misconception:  Driscoll's Theorem \citep{Driscoll1973} shows that samples from GPs associated to $k$ will realise \textit{outside} of $H_k$ with probability one, whenever $H_k$ is infinite-dimensional \citep[see also][]{steinwart2024does}.
    Therefore, we consider GREs in $H_k$ instead---a class of priors in function spaces that have traditionally been underexplored for Bayesian inference; albeit with notable exceptions \citep{Stuart2010,wild2022variational,wild2023rigorous}.
    With this in place, we finally obtain a diffusion on $H_k$ whose stationary distribution is the targeted Bayes posterior $\Pi_{\text{B}}$.
    \item
    Second, we perform \textbf{Bayesian computation by projection} based on the infinite-dimensional Langevin stochastic differential equation (SDE) that implements the dynamics arising from the WGF.
    In particular, as elements of $H_k$ can generally not be represented numerically without approximation, \Cref{sec:projected-langevin} provides a projected $M$-dimensional representation of this infinite-dimensional diffusion and uses it to construct an inference algorithm that scales as $\mathcal{O}(M^3)$.
    Specifically, we orthogonally project onto the first $M$ coefficients of the Kosambi-Karhunen-Loève expansion of our function, which corresponds to representing the infinite-dimensional object through an $M$-dimensional subspace.
    While various other subspaces might be considered, 
    orthogonal projection is not only a natural choice, but also turns out to result in near-optimal posterior approximations (see \Cref{sec:theory}).
    %
    %
    We call the resulting inference algorithm Projected Langevin Sampling (PLS). PLS consists of first simulating from this projected finite-dimensional diffusion on $\mathbb{R}^M$, and then probabilistically mapping back 
    into the infinite-dimensional $H_k$ via the law of total probability.
    \item
    Third, we derive \textbf{theoretical guarantees} of PLS, \text{and} shed light on its \textbf{relationship with SVGPs} in  \Cref{sec:theory}.
    To start with, we can show that PLS provides near-optimal variational approximations.
    Next---and despite being a sampling algorithm---PLS is also a variational inference scheme.
    Notably, PLS approximately samples from a nonparametric variational posterior that is strictly more expressive than the corresponding parametrically constrained Gaussian variational posterior of SVGPs.
    This allows PLS to target a strictly better posterior approximation than SVGPs, except for the special case of Gaussian likelihoods, for which PLS and SVGPs target the same (Gaussian) approximation.
    Unlike SVGPs, PLS can also be used for arbitrary likelihoods without requiring that their integrals with respect to variationally parameterised Gaussian measures are tractable or easy to approximate.
    %
\end{itemize}

\section{Wasserstein Gradient Flow (WGF) for Functional Inference}\label{sec:WGF-for-FSI}

The first step in building our new inference algorithm is the application of the Wasserstein Gradient Flow (WGF) to an optimisation problem that fully characterises a functional Bayes posterior.
To this end, we first define prerequisite notations and concepts in \Cref{sec:WGF:prelim}, and the optimisation problem to which the WGF is applied in \Cref{sec:opt-cen-perspec}.
Next, we draw instructive parallels between the WGF and gradient descent on measure spaces in \Cref{sec:WGF:GD-in-P2}, and state the SDE obtained as a limiting case of this gradient descent scheme in \Cref{sec:WGF:SDE}.
While all developments to this point apply for function space inference on arbitrary Hilbert spaces $H$, \Cref{sec:BKR-BKC} explains why computational feasibility demands that one should choose an RKHS $H=H_k$ based on a kernel $k$, and explains how to specify GRE priors on such spaces.

\subsection{Preliminaries \& Notations}
\label{sec:WGF:prelim}

On a technical level, a GRE $F$ is simply a random variable with realisations on a function space $H$.
What makes GREs attractive is  their Gaussianity:  they are fully specified
through a mean function $m \in H$ and a covariance operator $C: H \to H$ (a positive, self-adjoint, trace-class operator). 
Similarly to commonly employed notation for GPs, we thus write $F \sim \mathcal{N}(m,C)$ to indicate that $F$ is a GRE with mean function $m$ and covariance operator $C: H \to H$ (cf. Appendix \ref{sec:ap:technical-background} for a brief introduction to GREs in Hilbert spaces).

By definition, it is clear that GREs have a finite second moment. Further, and as we prove in Appendix \ref{ap:sec:moments-posterior}, Bayes posteriors depending on GRE priors are also guaranteed to have finite second moment.
To rigorously define the larger set of all measures with finite second moment on $H$, we let $\mathcal{B}(H)$ be the Borel $\sigma$-algebra on $H$, and then define the set of Borel probability measures on $H$ with finite second moment as
$$
\mathcal{P}_2(H) := \left\{ \mu : \mathcal{B}(H) \to [0, 1] \, \big| \; \mu(H)=1,\,\int_H || u ||^2 \, d\mu(u) < \infty  \right\}.
$$ 
For the application of the WGF, this set of measures is crucial: we can metrize it with the 2-Wasserstein metric distance, which for $\Pi,Q \in \mathcal{P}_2(H)$ is defined as
\begin{IEEEeqnarray}{rCl}
    W_2(\Pi,Q)^2 := \inf_{\gamma \in \mathcal{C}(\Pi,Q)} \int_{H
    \times H} ||h - h'||^2 \, d \gamma(h, h'),
    \nonumber
\end{IEEEeqnarray}
where $\mathcal{C}(\Pi,Q) \subset \mathcal{P}_2(H \times H)$ is the set of all probability measures on $\mathcal{B}(H \times H)$ whose marginals are $\Pi$ and $Q$, so that $\gamma(A \times H) = \Pi(A)$ and $\gamma(H \times B) = Q(B)$ for all $A, B \in \mathcal{B}(H)$ \citep[cf.][ Chapter 6]{ambrosio2005gradient}.

\subsection{Optimisation-centric perspectives on Bayesian inference}
\label{sec:opt-cen-perspec}

It is well-known that the Bayes posterior is the solution to a particular optimisation problem, both for the parametric  and the non-parametric case \citep[e.g., Theorem 1 in ][Appendix A.1]{knoblauch2019generalized, wild2022generalized}. 
Our WGF-based algorithmic perspective builds on this.
In particular, the WGF implements a form of gradient descent on the objective underlying this optimisation problem.
To state the objective under study, let $\Pi:=\mathcal{N}(0,C)$ be a GRE prior on the function space $H$, and $\ell(f) := -\log p(y_{1:N}|f)$ an arbitrary negative log likelihood function.\footnote{Here, $p(y_{1:N}|f)$ is called a likelihood function if $(y_{1:N},f) \in \bbR^N \times H \to p(y_{1:N}|f) \in [0, \infty)$ is product-measurable, and if  $ \int_{\bbR^N} p(y_{1:N}|f) \, d\mu(y_{1:N}) =1 $ for all fixed $f \in H$, and for $\mu$ a base measure on $\bbR^N$---such as the counting measure or Lebesgue measure.}
%
%
With this, one obtains the Bayes posterior $\Pi_{\text{B}}$ as 
\begin{IEEEeqnarray}{rCl}
    \Pi_{\text{B}} = \argmin_{Q \in \mathcal{P}_2(H)} \underbrace{\int \ell(f) \, dQ(f) + \KL(Q,\Pi)}_{=:L(Q)},
    \label{eq:vi-loss}
\end{IEEEeqnarray}
where 
$\KL(Q,\Pi)$ denotes the Kullback-Leibler divergence between the measures $Q$ and $\Pi$.\footnote{
 This optimisation problem is more commonly stated over the space of \textit{all} Borel probability measures on $H$, often denoted $\mathcal{P}(H)$. 
 As we show in \Cref{ap:sec:moments-posterior} however, the GRE prior guarantees that even over this larger space, the  solution always lies in $\mathcal{P}_2(H)$.
} 

This optimisation-centric perspective has many methodological and theoretical uses.
For example, it shows that conventional variational inference relates to full Bayesian inference like constrained to unconstrained optimisation \citep[see Theorem 2 in ][]{knoblauch2019generalized}.
To see this, simply note that conventional variational inference consists in choosing a family of measures $\calQ \subset \mathcal{P}(H)$ that has a finite-dimensional Euclidean parameterisation, and then computing the variational Bayes posterior  as $\widehat{Q} = \argmin_{Q \in \mathcal{Q}} L(Q)$ through stochastic gradient descent \citep[e.g.][]{titsias2014doubly} or other optimisation techniques.
In the specific context of Gaussian processes (GPs), such variational Bayesian approaches were pioneered by \citet{titsias2009variational}, further developed in \citet{matthews2016sparse}, and  adapted to Gaussian Random Elements (GREs) in Banach spaces by \citet{wild2022variational}.

The current paper also develops variationally Bayesian methodology, but refrains from the conventional approach: instead of choosing a parameterised family of distributions $\calQ$ and performing gradient descent on its finite-dimensional parameters,  gradient descent is performed \textit{directly} on the infinite-dimensional space $\mathcal{P}_2(H)$.
Unlike parameterised families $\calQ$, this allows for nonparametric variational posteriors, and results in much greater flexibility.
%
%
On a technical level, our methodology amounts to a computationally efficient implementation of the WGF for $L(Q)$, and transfers the insights of \citet{wild2023rigorous} to inference on measures over functions $f \in H$.
While we do not explore the case where the Kullback-Leibler divergence in \eqref{eq:vi-loss} is replaced with other regularisers, this setting was explored for finite-dimensional parameter spaces in \citet{wild2023rigorous}, and the results therein could be extended to function spaces.

\subsection[Gradient descent in the 2-Wasserstein space]{Gradient descent in $\mathcal{P}_2(H)$} 
\label{sec:WGF:GD-in-P2}

%
The intuition behind our application of the WGF is to  use the functional $Q \mapsto L(Q)$  in \eqref{eq:vi-loss} to find the minimiser $\Pi_{\text{B}}$ by performing a type of gradient descent over  $\mathcal{P}_2(H)$.
In particular and as outlined in Chapter 2 of \citet{ambrosio2005gradient}, we can pick a starting point $Q_0 \in \mathcal{P}_2(H)$, and then iteratively update according to 
\begin{IEEEeqnarray}{rCl}
    Q_{k+1} := \argmin_{Q \in \mathcal{P}_2(H)} \big\{ L(Q) + \frac{1}{2 \eta} W_2(Q, Q_k)^2 \big\},
    \nonumber
\end{IEEEeqnarray}
for $k \in \bbN$ and $\eta >0$ a sufficiently small step-size.\footnote{This is similar to gradient descent in Euclidean spaces: for a fixed and sufficiently small $\eta > 0$ and a function $\ell$ to be minimised, we can rewrite the $(k+1)$-th iterate of ordinary gradient descent as $\theta_{k+1} = \theta_{k} - \eta \nabla \ell(\theta_{k}) = \argmin_{\theta \in \Theta}\left\{ \ell(\theta) + \frac{1}{2\eta} \|\theta - \theta_{k}\|_2^2 \right\}$.  } 
If $L$ is sufficiently regular and we let $\eta \to 0$, the continuously indexed family $\big( Q_{\lfloor t/\eta \rfloor} \big)_{t \ge 0}$ converges to the limit $\big( Q(t) \big)_{t \ge0 } \subset \mathcal{P}_2(H)$ with $Q(0)=Q_0$ \citep[Chapter 11.1.3]{ambrosio2005gradient}.
This limiting evolution of measures  $\big( Q(t) \big)_{t \ge 0}$ is called the Wasserstein gradient flow (WGF) for $L$ starting at $Q_0$. 
It has obvious parallels with gradient descent, and is directly interpretable as its infinite-dimensional analogue on $\mathcal{P}_2(H)$: at every infinitesimally small step forward in time $t$,  $Q(t)$ moves in the direction amounting to the biggest decrease in the value for $Q \mapsto L(Q)$, and we would hope that the optimum is recovered in the limit so that $\Pi_{\text{B}} = \lim_{t\to\infty}Q(t)$.

\subsection{Following the WGF via the Langevin SDE in Hilbert space}
\label{sec:WGF:SDE}
Following Theorem 8.3.1 in \citet[]{ambrosio2005gradient},  the WGF $\big( Q(t) \big)_{t > 0}$ for functional $Q \mapsto L(Q)$ in \eqref{eq:vi-loss} satisfies for all test functions\footnote{More precisely, the relevant class of test functions are the smooth cylindrical functions, denoted $\text{Cyl}\big((H \times [0,T]\big)$ \citep[Definition 5.1.11]{ambrosio2005gradient}. By definition, $\varphi \in \text{Cyl}\big((H \times [0,T]\big) $ if and only if there exists $N \in \bbN$ and orthonormal vectors $v_1, \hdots, v_N \in H$ such that $\varphi(f) = \psi\big(  \langle f, v_1 \rangle, \hdots, \langle f, v_N \rangle, t\big) $ where $\psi: \bbR^N \times [0,T] \to \bbR$ is a smooth function with compact support.
} $\varphi: [0,T]\times H \to \bbR$ the equation
\begin{IEEEeqnarray}{rCl}
    \int_0^T \int_H \partial_t \varphi(t,f) -  \big\langle \nabla_W L[Q(t)](f), D \varphi(t,f) \big\rangle_{H} \, d Q(t)(f) \, dt = 0, \label{eq:wgf-characterisation}
\end{IEEEeqnarray}
where $D \varphi(t,f)$ denotes the Fréchet derivative of $\varphi$ with respect to  $f$, and $\nabla_W L[Q]: H \to H$ is the so-called \textit{Wasserstein gradient} of $L$ evaluated at $Q \in \mathcal{P}_2(H)$. 
Obtaining an explicit form for this Wasserstein gradient is essential in deriving the dynamics governing $(Q(t))_{t\geq 0}$, and we derive it in \Cref{thm:WGF} (cf. \Cref{ap:sec:WGF}) via the set of equations given by the characterisation of $(Q(t))_{t\geq 0}$ in \eqref{eq:wgf-characterisation}. 
In so doing, we find that for \eqref{eq:vi-loss}, we get
%
\begin{IEEEeqnarray}{rCl}
    \nabla_W L[Q](f) = D\ell(f) + D (\log q) (f) = D \ell(f) + \frac{D q(f)}{q(f)},
    \label{wg:kl-divergence}
\end{IEEEeqnarray}
where $ q := dQ/d\Pi : H \to \bbR$ is the Radon-Nikodym derivative of $Q$ with respect to the prior measure $\Pi$. 
Plugging \eqref{wg:kl-divergence} into \eqref{eq:wgf-characterisation}, we can identify the law of a stochastic process 
$\big(F(t)\big)_{t>0}$ with $F(t) \in H$ for which $F(t) \sim Q(t)$ for all $t \geq 0$ (cf. Theorem \ref{thm:WG-KLD} in Appendix \ref{ap:sec:WGF}).  
In particular, we find that $\big(F(t)\big)_{t >0}$ is the solution to the infinite-dimensional version of the Langevin Stochastic Differential Equation (SDE) given by
\begin{IEEEeqnarray}{rCl}
    d F(t) = -\left(  D \ell \big( F(t) \big)  + C^{-1} F(t) \right) dt + \sqrt{2} d W(t). \label{eq:langevin-infinite-dim}
\end{IEEEeqnarray}
Here, $\big( W(t) \big)_{t \ge 0}$ is a cylindrical Brownian motion in $H$ and $C^{-1}: \text{Im}(C) \subset H \to H$ is the inverse of $C$.\footnote{ The operator $C^{-1}$ is unbounded and given by $C^{-1} f = \sum_{n=1}^\infty \lambda_n^{-1} \langle f, e_n \rangle e_n $ where $\{ \lambda_n, e_n\}_{n=1}^\infty$ are the eigenvalue-eigenvector pairs obtained from the spectral decomposition of the self-adjoint trace class covariance operator $C$. }
This SDE was first introduced in \citet{hairer2005analysis}, and the existence, uniqueness and ergodic properties of its solution are discussed in \citet{hairer2007analysis}.
\citet{hairer2011signal} provides a more accessible treatment on the subject, and \citet{da2014stochastic} provides a more general discussion of SDEs in Hilbert spaces.

For our purposes, we can thankfully gloss over many of the technical aspects of dealing with this infinite-dimensional SDE.
Instead, our main reason for introducing \eqref{eq:langevin-infinite-dim} will be its use as the driving engine for our methodology. 
To this end, it is instructive to note the  parallels with the finite-dimensional case: 
if we consider \eqref{eq:vi-loss} for $H=\bbR^J$, we recover a Bayes posterior defined on a finite-dimensional parameter $\theta \in \bbR^J$. 
Writing the associated likelihood function as $p(y_{1:N}|\theta)$, and the corresponding Gaussian prior density as $p(\theta) = \mathcal{N}(\theta; 0, \Sigma)$, the resulting WGF---now defined on $\mathcal{P}_2(\bbR^J)$---yields a measure evolution that corresponds to the well-known finite-dimensional Langevin SDE 
\begin{IEEEeqnarray}{rCl}
    d \theta(t) = - \left(\nabla\ell_N\big(\theta(t)\big) +\Sigma^{-1} \theta(t) \right) dt + \sqrt{2} d\beta(t), \label{eq:langevin-sde-fidi}
\end{IEEEeqnarray}
where $\ell_N(\theta):= - \log p(y_{1:N}|\theta)$ and $\big( \beta(t) \big)_{t\ge0}$ is the standard Brownian motion in $\bbR^J$ \citep{jordan1998variational,otto2001geometry}. 
This finite-dimensional Langevin SDE is the basis of many classical algorithms for Bayesian computation  since its stationary distribution for $\theta(t)$ as $t\to\infty$ recovers the Bayes posterior  \citep[see e.g.][]{roberts1996exponential,welling2011bayesian}.
In other words, and exactly as for the infinite-dimensional case, one way of motivating \textit{why} this diffusion recovers the Bayes posterior is by interpreting it as the natural analogue of gradient descent on \eqref{eq:vi-loss} for the special case of $H = \bbR^J$.
Just as \eqref{eq:langevin-sde-fidi} is the inspiration for many seminal algorithms for computing  Bayes posteriors over finite-dimensional parameters, \eqref{eq:langevin-infinite-dim} can be used to generate samples from the posterior measure $\Pi_{\text{B}}$ that solves the optimisation problem \eqref{eq:vi-loss} with infinite-dimensional function spaces $H$.
However, the infinite dimensionality of $H$ introduces an additional complication: the Langevin SDE in \eqref{eq:langevin-infinite-dim} generally cannot be represented on a computer. 
Fortunately, we can overcome this with parsimonious numerical approximations (cf. Section \ref{sec:projected-langevin}).

\subsection[Choosing the Hilbert Space and the Prior]{Choosing Function Space $H$ and Prior $\Pi$}
\label{sec:BKR-BKC}

So far, our developments are valid for likelihoods and GREs on general Hilbert spaces.
In particular, the Langevin SDE in \eqref{eq:langevin-infinite-dim} in principle  may produce posterior inferences over general function spaces $H$ as long as they are Hilbert spaces. 
As we will explain next however, the seemingly innocuous requirement that $\ell$ be Fréchet differentiable will necessitate the assumption that $H$ is a reproducing kernel Hilbert space (RKHS) for virtually all likelihood functions of practical interest (cf. Appendix \ref{ap:sec:RKHS} for  basic properties of an RKHS).

\subsubsection{The Inevitability of the RKHS }\label{sec:loss}

Implementing the Langevin SDE in \eqref{eq:langevin-infinite-dim} requires calculation of the Wasserstein gradient in \eqref{wg:kl-divergence}, and therefore of the Fréchet derivative $D 
\ell(f)$ of $\ell :H \to \bbR$.
For simplicity, consider negative log-likelihood functions  which for a cost function $c: \calY \times \calY \to \bbR$ can be written as
\begin{IEEEeqnarray}{rCl}
    \ell(f) & =&  \sum_{n=1}^N c\big(y_n, f(x_n) \big) \label{eq:loss-pointwise} + \text{ constant},
\end{IEEEeqnarray}
where $c\big(y_n, f(x_n) \big)$ measures the discrepancy between an observation $y_n$ and the prediction $f(x_n)$ for $y_n$, and where the constant  does not depend on $f$. 
For example, the Gaussian likelihood
corresponds to  $c(y_n,f(x_n)) = \frac{1}{2 \sigma^2} \big( y_n - f(x_n) \big)^2$.
Another example is the Bernoulli likelihood for binary classification, which is
\begin{IEEEeqnarray}{rCl}
    p(y_n|f) & = &  \phi \big( f(x_n)\big)^{y_n} \cdot \Big(1- \phi\big( f(x_n) \big)\Big)^{1-y_n}.
    \label{eq:bernoulli-model}
\end{IEEEeqnarray}
In this case, $y_n \in \{0,1\}$, so that the associated negative log likelihood $\ell(f)$  is of the form \eqref{eq:loss-pointwise} with $c(y_n, f(x_n)) = -y_n \log \phi(f(x_n))  - (1-y_n)\log (1-\phi(f(x_n)))$.
Here,
$\phi: \bbR \to [0,1]$ is a  mapping that transforms latent functional outputs into probabilities such as the the logistic function $\phi_{\operatorname{logistic}}(f(x_n)) = \big(1+\exp(-f(x_n)) \big)^{-1}$. 

For any negative log likelihood of form  \eqref{eq:loss-pointwise}, we can apply the chain rule and obtain the corresponding Fréchet derivative  as 
\begin{IEEEeqnarray}{rCl}
D \ell(f) = \sum_{n=1}^N (\partial_2 c)\big( y_n, f(x_n)\big) D(s_n)(f), \label{eq:frechet-derivative-for-RKHS}
\end{IEEEeqnarray}
where we have written $s_n(f) := f(x_n)$ as the point-wise evaluation functional for $f \in H$ at $x_n$,  $\partial_2 c$ as the derivative of $c$ with respect to its second component, and $D(
s_n)$ as the Fréchet derivative of $
s_n$.
Inspecting the expression, it is now clear that we need to assume that all $s_n$ are Fréchet differentiable and therefore continuous.
However, demanding continuity of all pointwise evaluation functionals $s_n(f)$ for all $f \in H$ is in fact equivalent to $H$ being an RKHS \citep[Theorem 1]{Berlinet2004}.

In summary, even when attention is restricted to the simple sub-class of likelihood functions based   on pointwise evaluations of $f$ as in \eqref{eq:loss-pointwise},  the Wasserstein gradient in \eqref{wg:kl-divergence} needed to evolve the infinite-dimensional Langevin SDE in \eqref{eq:langevin-infinite-dim} will \textit{only} have the Fr\'echet derivatives necessary for implementing it numerically if $H$ is an RKHS.
Throughout the remainder of the paper, we therefore take $H$ to be an RKHS $H_k$ associated with the reproducing kernel function $k: \mathcal{X} \times \mathcal{X} \to \mathbb{R}$.
Thanks to the widely celebrated reproducing property of RKHS functions, this implies that $s_n(f) = \langle f, k_n \rangle$ for $k_n = k(x_n, \cdot) \in H_k$, so that the required Fréchet derivative in \eqref{eq:frechet-derivative-for-RKHS} is $D (s_n)=k_n$, and
$
    D \ell(f) = \sum_{n=1}^N (\partial_2 c)\big( y_n, f(x_n)\big) k_n .
$

\subsubsection{Covariance Operators and Gaussian processes}\label{sec:cov-op-gp-interpretation}

While choosing GRE priors may seem like an intractably difficult proposition, one can draw clear analogies and parallels to the design of Gaussian process (GP) priors.
We will leverage these connections to design our GRE priors in the current paper.
To this end, note that 
the GRE prior $\Pi :=\mathcal{N}(0,C)$ is specified by a covariance operator $C: H_k \to H_k$ that encodes prior knowledge about the unknown function. 
In this paper, we specify $C$ through a kernel function $k: \calX \times \calX \to \bbR$ and a probability measure $\nu \in \mathcal{P}(\calX)$ as 
\begin{IEEEeqnarray}{rCl}
    Cf & = & \int k(\cdot, x') \, f(x') \, d\nu(x')
    \label{eq:def-cov-op}
\end{IEEEeqnarray}
for all $f \in H_k$. 
In Appendix \ref{ap:sec:GRE_in_RKHS}, we show that constructing $C$ in this way indeed leads to a valid covariance operator on $H_k$ under mild assumptions on $\nu$ and $k$.\footnote{A GRE $F \sim \mathcal{N}(0,C)$ exists if and only if $C$ is a positive, self adjoint, trace-class operator \citep[Section 2.3.1]{da2014stochastic}. }
Every GRE $F \sim \mathcal{N}(0,C)$ on $H_k$ naturally induces a GP through the definition ${F}_{\operatorname{GP}}(x) := \langle F , k(x,\cdot) \rangle $ for all $x \in \calX$. 
${F}_{\operatorname{GP}}$ resulting from this construction is endowed with kernel $r$ (cf. Lemma \ref{lemma:GRE-GP-analogy} in Appendix \ref{ap:sec:GRE_in_RKHS}) given by
\begin{IEEEeqnarray}{rCl}
    r(x,x'):= \int k(x,\xi) k(\xi,x') \, d\nu(\xi)
    \label{eq:def-r-kernel}
\end{IEEEeqnarray}
for all $x,x' \in \calX$. 
This connection is practically useful, since it is well-understood how to specify interpretable GP function priors   \citep[see e.g.][]{rasmussen2003gaussian,duvenaud2014kernel}.
Therefore, one can specify interpretable GRE priors via the GP induced by $r$.\footnote{
Note that the GP prior $F \sim \text{GP}(0,k)$ does \textit{not} define a prior on $H_k$: by Dricoll's Theorem \citep{Driscoll1973}, the sample paths of such a GP would be almost surely not be contained in $H_k$.
}
The correspondence between $C$ and $r$ also elucidates the role of $\nu$:
the similarity between two points $x$ and $x'$ is measured by the product of their pairwise similarities with a third point $\xi$ as $k(x,\xi) k(\xi, x')$, and this product is  averaged by $\nu$ over $\xi$. 
This construction has appeared in the literature before \citep[see e.g.][]{dance2024spectral}, and can lead to poor function priors if $\nu$ is very different from the data generating process $\text{Law}[X_1]$ of $x_1,\hdots,x_N$ (see Appendix \ref{ap:sec:the-role-of-nu} for a more extensive discussion). 
To this end, we take $\nu = \text{Law}[X_1]$,\footnote{Using the data-generating distribution of $x_1, \dots x_N$ as our choice for $\nu$ has the interesting effect of making the kernel choice $r$ data-dependent. This is akin to the logic of empirical Bayes, and fine-tunes the GRE prior $F$ to the data at hand. }, which is generally unavailable to us, so that we need to estimate it.
The easiest way to do this---and the approach we use in our experiments of \Cref{section:experiments}---is through the empirical measure $\frac{1}{N}\sum_{i=1}^N \delta_{x_i}$ depending on the $N$ available data points.
In practice, this has the immediate effect of making us operate on an $N$-dimensional (rather than an infinite-dimensional) space, but neither compromises the general developments in the remainder---all of which are spelt out in largest generality and for the infinite-dimensional case---nor affects the practicality of the $M$-dimensional posterior approximations for $M \ll N$ that we develop.

\section{Projected Diffusions for Tractable Posterior Inference}\label{sec:projected-langevin}

Having laid out how the GRE prior should be designed, we are ready to deploy the infinite-dimensional Langevin SDE in \eqref{eq:langevin-infinite-dim} to produce samples from Bayes posteriors over $H_k$ that are specified via \eqref{eq:vi-loss} and rely on GRE priors.
%
But without any further numerical approximation, this cannot be done: 
$\big( F(t) \big)_{t \ge 0}$ is defined on an RKHS $H_k$, so there is no numerically exact way of  representing its evolution, even after time-discretisation.
To deal with this, we introduce a sequence of approximations to $\Pi_{\text{B}}$ that are summarised  in \Cref{tab:overview-of-measures}, and which culminate in a  tractable sampling procedure.

\begin{table*}[ht!]
\caption{Meaning of and relationships between different approximating measures. }
\centering
\begin{tabular}{p{8mm}|p{72mm}|p{15mm}|p{23mm}|p{10mm}}
 & explanation & diffusion & approximation  & space  \\
\toprule
$\Pi$ &   
    GRE prior on $H_k$ &  && $H_k$ \\
$\Pi_\text{B}$ & 
    Bayes posterior  $F \mid y_{1:N}$ & 
    $F(t)$
    &&  $H_k$ \\
${\Pi}_{\text{B}}^{1:M}$& 
    Bayes posterior ${F}^{1:M} \mid y_{1:N}$ so that  \newline 
    $\Pi_{\text{B}} =  \int \bbP( F \in \cdot \, | \,  {{F}}^{1:M}=u)d{\Pi}^{1:M}_{\text{B}}(u)$ & $F^{1:M}(t)$ &  & $\mathbb{R}^M$
    \\
$\widehat{\Pi}_{\text{B}}^{1:M}$& 
    Bayes posterior $\widehat{F}^{1:M} \mid y_{1:N}$ so that  \newline 
    $\Pi_{\text{B}} =  \int \bbP( F \in \cdot \, | \,  \widehat{{F}}^{1:M}=u)d\widehat{\Pi}^{1:M}_{\text{B}}(u)$  & && $\mathbb{R}^M$ \\
${\Pi}_{\infty}^{1:M}$ & 
    Limiting measure of projected diffusion    & 
    $\widetilde{F}^{1:M}(t)$
    & $\Pi_{\infty}^{1:M} \approx {\Pi}_{\text{B}}^{1:M}$ & $\mathbb{R}^M$ \\
$\widehat{\Pi}_{\infty}^{1:M}$ & 
    Limiting measure of projected and estimated diffusion   & 
    $\widehat{F}^{1:M}(t)$
    & $\widehat{\Pi}_{\infty}^{1:M} \approx \widehat{\Pi}_{\text{B}}^{1:M}$ & $\mathbb{R}^M$ \\
${\Pi}_{\infty}$ & 
     $\Pi_{\infty}:=  \int \bbP( F \in \cdot \, | \,  {{F}}^{1:M}=u)d{\Pi}^{1:M}_{\infty}(u)$  
     & & $\Pi_{\infty} \approx {\Pi}_{\text{B}}$ & $H_k$ \\
$\widehat{\Pi}_{\infty}$ & 
     $\widehat{\Pi}_{\infty}:=  \int \bbP( F \in \cdot \, | \,  {\widehat{F}}^{1:M}=u)d\widehat{\Pi}^{1:M}_{\infty}(u)$   &  & $\widehat{\Pi}_{\infty} \approx {\Pi}_{\text{B}}$ & $H_k$ \\
 \midrule
\end{tabular}
\label{tab:overview-of-measures}
\end{table*}

\subsection{Projected GREs}

A natural first step for approximating the intractable diffusion $\big( F(t) \big)_{t \ge 0}$ relates to the representation of $F \sim \mathcal{N}(0,C)$  via its Kosambi–Karhunen–Loève representation. 
Following Proposition 2.18 in \citet[][]{da2014stochastic}, if $C$ is defined as in \eqref{eq:def-cov-op}, this is equivalent to the well-studied Mercer decomposition of $k$, and given as 
\begin{IEEEeqnarray}{rCl}
    F(x) = \sum_{m=1}^\infty \langle F, e_m \rangle e_m(x), 
    \nonumber
\end{IEEEeqnarray}
where $ F^{m} := \langle F , e_m \rangle \sim \mathcal{N}(0, \lambda_m)$ independently with $\{\lambda_m, e_m\}_{m=1}^\infty \subset [0, \infty) \times H_k$  the eigenvalue-eigenfunction pairs with $\lambda_1 \ge \lambda_2 \ge \hdots$ obtained from the spectral decomposition of the covariance operator $C$ in \eqref{eq:def-cov-op}. 
Based on this representation, $F$ can be approximated via its orthogonal projection onto the span of the first $M$ eigenfunctions given by
\begin{IEEEeqnarray}{rCl}
    \Pr [F] := \sum_{m=1}^M 
    F^m
    e_m.
    \label{eq:projection}
\end{IEEEeqnarray}
Denoting $\text{Proj}[H_k]$ as the $M$-dimensional subspace of $H_k$ induced by this projection, and $\mathcal{P}(\text{Proj}[H_k])$ as the set of measures on this subspace, \Cref{fig:projection} illustrates our construction.
%


%
Since $C$ is a trace-class operator, this approximation through orthogonal projection is well-motivated: universally, the eigenvalues decay \textit{at least} at speed $\lambda_m = o(m^{-1})$.\footnote{
  This holds since by definition of trace-class operators with spectral decomposition, we have $\sum_{m=1}^{\infty}\lambda_m < \infty$; so since $\lambda_m \geq 0$, it must hold that $\lambda_m = o(m^{-1})$.
}
%
This bound is extremely conservative, and $\lambda_m$ typically decays at much faster rates if the functions in $H_k$ are sufficiently smooth.
As Table \ref{tab:spectral-decay} illustrates, the decay is even often exponential, which means that roughly speaking, the projection in \eqref{eq:projection} converges to $F(t)$ exponentially fast as $M\to\infty$. 
As we shall see later, this decay translates directly into the approximation quality of our posteriors (see \Cref{lemma:decay-eigenvalues}).
This further motivates choosing the $M$-dimensional subspace obtained through projection to approximately represent  randomness on $H_k$: unlike more commonly chosen subspaces---like those obtained through pointwise evaluations in the inducing point framework of \citet{titsias2009variational}---our projection leads to subspaces that are in some sense optimal.
On a technical level, this happens because they represent the subspace of $H_k$ where the prior $F$ concentrates the most.

\begin{table*}[h!]
\caption{Spectral decay for different kernels $k$ and input distributions $\nu \in \mathcal{P}(\bbR^D)$; taken from \citet{burt2019rates} and containing results from \citet{ritter1995multivariate}. }
\centering
\begin{tabular}{l|l|l}
kernel & input distribution $\nu$  & decay of $\lambda_m$  \\
\toprule
 Squared Exponential & compact support & $\mathcal{O}\big(\exp(-\alpha \frac{m}{D} \log \frac{m}{D}) \big) $  \\
 Squared Exponential & Gaussian &  $\mathcal{O}\big(\exp(-\alpha \frac{m}{D} ) \big)$ \\
 Matérn $l+1/2$ & Uniform on interval & $\mathcal{O}\big(m^{-2l-2} \log(m)^{2(D-1)(l+1)} \big)$ \\
 \midrule
\end{tabular}\label{tab:spectral-decay}
\end{table*}

\begin{figure}[b!]
\small
\centering
\includegraphics[width=\linewidth]{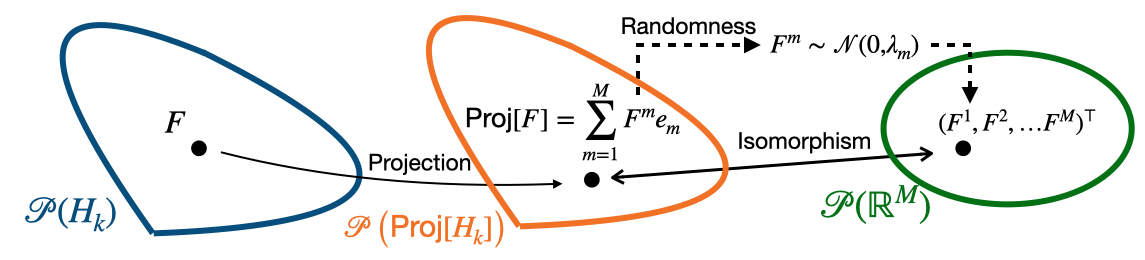} 
  \caption{\small{
  To represent a GRE $F$ on an infinite-dimensional space $H_k$, we orthogonally project it onto the first $M$ terms of its Kosambi-Karhunen-Loève representation via \eqref{eq:projection}.
  Importantly, the projected random variable $\text{Proj}[F]$ can be exactly represented using its first $M$ coefficients $F^1, F^2, \dots F^M$, which are jointly distributed as a fully factorised $M$-dimensional zero-mean  Gaussian with variances $\lambda_1^2, \lambda_2^2, \dots, \lambda_M^2$.
  }}
  \label{fig:projection}
\end{figure}

\subsection{Finite-Dimensional Projection of the WGF in Hilbert space}
\label{sec:fin-dim-projection-of-WGF}

From a practical standpoint, the projection in \eqref{eq:projection} allows us to study and simulate the $M$-dimensional SDE  $(F^{1:M}(t))_{t\geq 0}$ instead of $(F(t))_{t\geq 0}$.
As we will show in the remainder, the limiting distribution of this projected SDE allows us to exactly sample from the Bayes posterior $\Pi_{\text{B}}$ when $M=N$, or to tractably approximate it with provable guarantees when $M\ll N$ (cf. \Cref{thm:PLS-optimal-close}).
Intriguingly, the resulting posterior approximation coincides with SVGPs for Gaussian likelihoods (cf. \Cref{lemma:Gaussian-case-optimal}), but is \textit{not} constrained to a Gaussian variational family---and thus is generally non-Gaussian for non-Gaussian likelihoods.

To power this approximation algorithm, we first study the time evolution for $F^{1:M}(t) = (F^1(t), F^2(t), \dots F^M(t))^{\top}$ with $F^m(t):= \langle F(t), e_m \rangle \in \mathbb{R}$.
Using It\'{o}'s Rule, we find
\begin{IEEEeqnarray}{rCl}
    d F^m(t) = - \Big(\sum_{n=1}^N (\partial_2 c)\big(y_n, F(t)(x_{n})\big) e_{m}(x) + \frac{ F^m(t)}{\lambda_m} \Big) dt + \sqrt{2} d B^m(t),
    \label{eq:components-Langevin}
\end{IEEEeqnarray}
%
%
where $B^m(t) := \langle W(t), e_m \rangle$ are stochastically independent Brownian motions in $\bbR$. 
While this diffusion is simpler than \eqref{eq:langevin-infinite-dim}, access to the $M$ largest eigenvalue-eigenfunction pairs $\{{\lambda}_m, {e}_m \}_{m=1}^M$ would still not suffice to make it tractable due to its dependence on $F(t)$, for which we do not have an exact representation.
To rectify this, we therefore replace $F(t)$ with its \textit{exact} projection $\text{Proj}[F(t)]$. This yields the new SDE
\begin{IEEEeqnarray}{rCl}
    d \widetilde{F}^m(t) = - \Big(\sum_{n=1}^N (\partial_2 c)\big(y_n, {\text{Proj}}[F(t)](x_{n})\big) e_{m}(x) + \frac{ \widetilde{F}^m(t)}{\lambda_m} \Big) dt + \sqrt{2} d B^m(t),
    \label{eq:components-Langevin-2}
\end{IEEEeqnarray}
where ${\text{Proj}}[F(t)] = \sum_{m=1}^M \widetilde{F}^m(t) e_m$ and $\widetilde{F}_m = \langle F(t), e_m\rangle$.
Further, the limiting measure of this diffusion is defined as $\Pi_{\infty}^{1:M}$ so that $\Pi_{\infty}^{1:M} \sim \lim_{t\to\infty} \widetilde{F}^{1:M}(t)$
Even this approximation is not computationally feasible however: since the exact values of $\{\lambda_m, e_m\}_{m=1}^{M}$ are unknown in practice, $\text{Proj}[F(t)]$ cannot be computed exactly and needs to be {estimated}.
To this end, we use the Nyström method to estimate the $M$ largest eigenvalue-eigenfunction pairs $\{\widehat{\lambda}_m, \widehat{e}_m \}_{m=1}^M$  (cf. Appendix \ref{sec:ap:nystroem-method}). 
Given $\{\widehat{\lambda}_m, \widehat{e}_m \}_{m=1}^M$, a natural estimate for $F(t)(x_n)$ then follows from the \textit{estimated} projection 
\begin{IEEEeqnarray}{rCl}
    \widehat{\text{Proj}}[F(t)] &:= &  \sum_{m=1}^M  \underbrace{\langle F(t) , \widehat{e}_m \rangle}_{=: \widehat{F}^m(t)} \widehat{e}_m   
    \label{eq:approximate-projection}
\end{IEEEeqnarray} 
and its evaluations $\widehat{\text{Proj}}[F(t)](x_n)= \sum_{m=1}^M  \widehat{F}^m(t) \widehat{e}_m(x_n)$.
Substituting this into \eqref{eq:components-Langevin-2} leads to the tractable SDE $\widehat{F}^{1:M}(t)$ whose components evolve as
\begin{IEEEeqnarray}{rCl}
    d \widehat{F}^m(t) = - \left(\sum_{n=1}^N (\partial_2 c)\Big( y_n, \widehat{\text{Proj}} [F_t](x_n) \Big) \widehat{e}_{m}(x) + \frac{\widehat{F}^m(t)}{\widehat\lambda_m} \right) dt + \sqrt{2} d B^m(t). \label{eq:approx-langevin-components}
\end{IEEEeqnarray}   
This $M$-dimensional diffusion $\widehat{F}^{1:M}(t) = (\widehat{F}^{1}(t), \dots \widehat{F}^{M}(t))^{\top}$ can be forward-simulated via Euler-Maruyama discretisation.
We visually summarise the steps that take us from the infeasible diffusion $(F(t))_{t\geq 0}$ to the computationally tractable $(\widehat{F}^{1:M}(t))_{t\geq 0}$ in \Cref{fig:diffusion-illustration}.

Before moving on, it is useful to point out that while all computations and algorithmic proposals will be based on the limiting measure $\widehat{\Pi}_{\infty}^{1:M}$ of \eqref{eq:approx-langevin-components}, the theory we develop in \Cref{sec:theory} will instead rely on the limiting measure ${\Pi}_{\infty}^{1:M}$ of \eqref{eq:components-Langevin-2} for its analysis.
This means that while the computational algorithm we propose  relies on the Nystr\"om approximation, the theory we derive ignores the associated estimation error.
This decision is chiefly guided by the difficulty associated to the mathematical analysis of \eqref{eq:approx-langevin-components}, but our empirical results in \Cref{section:experiments}  additionally legitimise this, and suggest that the conclusions we draw from \eqref{eq:components-Langevin-2} are broadly applicable to \eqref{eq:approx-langevin-components}.
This is perhaps unsurprising: for the special case where $M = N$ and $\nu = \frac{1}{N}\sum_{i=1}^N \delta_{x_i}$, it can be shown that the Nystr\"om approximation is exact so that \eqref{eq:approx-langevin-components} and $\widehat{\Pi}_{\infty}^{1:M}$ are identical to \eqref{eq:components-Langevin-2} and ${\Pi}_{\infty}^{1:M}$.

\begin{figure}[t!]
\small
\centering
\includegraphics[width=\linewidth]{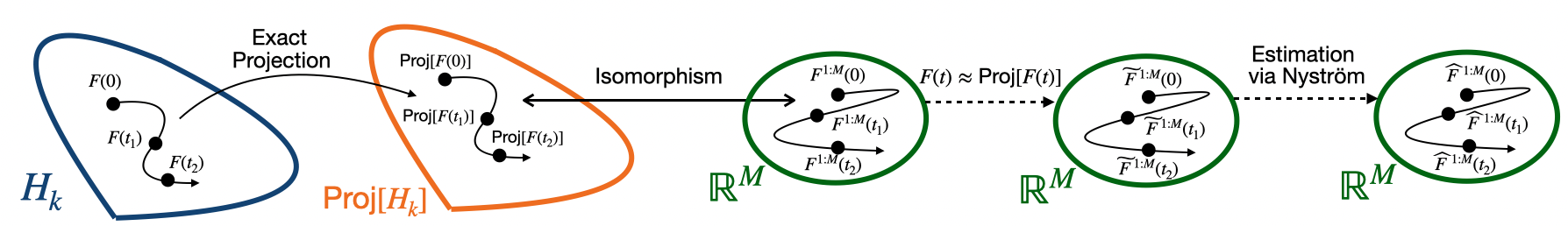} 
  \caption{\small{
 In an ideal world, we would directly evolve the SDE $(F(t)))_{t\geq 0}$ on $\mathcal{H}_k$. 
 Since this is numerically infeasible however, we instead use the projection in \eqref{eq:projection}.
 To evolve the resulting SDE, we rely on its representation via $M$ components $\widetilde{F}^{1:M}$, which we can evolve according to \eqref{eq:components-Langevin-2}. 
 In its exact form, this SDE depends on unknown quantities. We therefore have to estimate the unknown eigenvalue-eigenfunction pairs $\{{\lambda}_m, {e}_m \}_{m=1}^M$, resulting in our final target SDE $(\widehat{F}^{1:M}(t))_{t\geq 0}$ that we can approximate via \eqref{eq:approx-langevin-components} by applying Euler-Maruyama discretisation and forward simulation.
 %
  }}
  \label{fig:diffusion-illustration}
\end{figure}

\subsection{Posteriors on Infinite-Dimensional Spaces via Projection}


Having developed the tractable projected diffusion  $(\widehat{F}^{1:M}(t))_{t\geq 0}$, we now seek to use this projection for approximation of the Bayes posterior $\Pi_{\text{B}}$.
To this end, first note that we can characterise the Bayes posterior via $\Pi_{\text{B}}\in \mathcal{P}(H_k)$ given by $\Pi_{\text B} \sim \lim_{t\to\infty} F(t) $.
It might thus appear  reasonable to expect that the limiting distribution $\widehat{\Pi}^{1:M}_{\infty}\in \mathcal{P}(\mathbb{R}^M)$ given by $\widehat{\Pi}^{1:M}_{\infty} \sim \lim_{t\to\infty} \widehat{F}^{1:M}(t)$ evolved via \eqref{eq:approx-langevin-components} can approximate an $M$-dimensional summary of $\Pi_{\text{B}}$.
 Specifically, we might expect the limiting distribution of $\widehat{\text{Proj}}[F(t)]$  in \eqref{eq:projection} constructed using components $\widehat{F}^m$ sampled from $\widehat{\Pi}^{1:M}_{\infty}$ to follow  a distribution that is roughly equal to $\Pi_{\text{B}}$.
But a construction like this would restrict posterior uncertainty to an $M$-dimensional subspace of $H_k$, and lead to miscalibrated representations of uncertainty. 
In particular, if we wanted to quantify the posterior uncertainty of functions in $H_k$ that are far outside of the $M$-dimensional subspace given by $\text{span}\{\widehat{e}_m:1\leq m \leq M\}$, their uncertainty would degenerate to zero.
Clearly, this is undesirable: these are \textit{precisely} the types of functions in $H_k$ for which we have the highest degree of uncertainty.

Fortunately, we can construct a much smarter approximation that rectifies this issue and which we illustrate in \Cref{fig:PLS-illustration}.
The key insight is that to provide uncertainty on the remaining dimensions of $H_k$, 
we have to push the finite-dimensional randomness from $\widehat{\Pi}^{1:M}_{\infty}$ into the infinite-dimensional space $H_k$.
The algorithm we propose implements this pushforward approximately, using the same logic that underlies SVGPs: In particular, we consider for a particular choice of $\tau \in \mathcal{P}_2(\bbR^M)$ the measure
\begin{IEEEeqnarray}{rCl}
 \widehat{\Pi}_{\tau} := \int \bbP\big( F \in \cdot \, | \, \widehat{F}^{1:M}= u \big) \, d \tau(u), \label{eq:prob-measures-tau-widehat}
\end{IEEEeqnarray}
which can be understood as a member of the non-parametric variational family given by 
\begin{IEEEeqnarray}{rCl}
    \widehat{\mathcal{Q}}_M := \left\{\widehat{\Pi}_{\tau}: \tau \in \mathcal{P}_2(\bbR^M) \right\} \subset \mathcal{P}_2(H_k).
    \label{eq:nonparametric-variational-family-widehat}
\end{IEEEeqnarray}
While a classic variational method would seek to find an optimal value of $\tau$ through some optimisation routine, our approach is different.
Rather than attempting to solve a difficult optimisation problem like this over a non-parametric variational family, we instead \textit{hand-select} a $\tau$ which we know how to sample from relatively cheaply, and which we hypothesise provides a nearly-optimal variational approximation---a hypothesis that we prove to be true in \Cref{sec:theory}.
In particular, we advocate for choosing $$\tau = \widehat{\Pi}_{\infty}^{1:M},$$ which is the limiting measure of \eqref{eq:approx-langevin-components}, and motivate why this is sensible.
The next section provides this motivation from an intuitive and practical standpoint, while \Cref{sec:theory} does so from a  theoretical vantage point.

\subsection{Approximate Inference via Sufficiency}

Before explaining our  approximation, we have to  introduce a number of relevant objects.
First, we define
$\widehat{\Pi}^{1:M}_{\text{B}} = \widehat{\phi}^{1:M} \# \Pi_{\text{B}}$, where $\#$ denotes the pushforward operator, and $\widehat{\phi}^{1:M}(f):= \langle f, \widehat{e}^{1:M} \rangle = ( \langle f, \widehat{e}_1 \rangle, \hdots, \langle f, \widehat{e}_M \rangle)^{\top}$ denotes the function $\widehat{\phi}^{1:M}:H_k \to \bbR^M$ that maps $f \in H_k$ into the $M$ largest \textit{estimated} components $\widehat{F}^1, \widehat{F}^2, \dots \widehat{F}^M$ used  in \eqref{eq:projection}.
In other words, $\widehat{\Pi}^{1:M}_{\text{B}}$ is really the Bayes posterior  corresponding to the conditional distribution given by $\widehat{F}^{1:M} \mid y_{1:N}$, so that $\widehat{\Pi}^{1:M}_{\text{B}}(A') = \bbP( \widehat{F}^{1:M} \in A'\mid y_{1:N} )$ for all measurable subsets $A' \subset \bbR^{M}$.
Since our ultimate goal is to approximate the predictive distribution of $F(x_{1:N_*}^*)\mid y_{1:N}$ for $F\mid y_{1:N} \sim \Pi_{\text{B}}$ and new test points $x_{1:N_*}^* \in \calX^{N_*}$, we additionally define the point-wise evaluation functional  $\operatorname{eval}[x^*_{1:N_*}](f):= f(x^*_{1:{N_*}})$ for $f \in H_k$.
As the law of total probability implies that
$\Pi_{\text{B}}(B) = \int \bbP(F \in B \;|\: y_{1:N}, \widehat{F}^{1:M} = u) d\widehat{\Pi}^{1:M}_{\text{B}}(u)$ for all $B \in \mathcal{B}(H_k)$, we thus obtain for all measurable $A \subset \mathbb{R}^{N_*}$ the exact target predictive distribution as
\begin{align}
    \big(\operatorname{eval}[x^*_{1:N_*}]\# \Pi_{\text{B}}\big)(A) = \int_{\bbR^M} \underbrace{\bbP( F(x^*_{1:N_*}) \in A \mid  y_{1:N}, {\widehat{F}}^{1:M}=u)}_{\text{(I)}} \;\; d 
    \underbrace{\widehat{\Pi}^{1:M}_{\text{B}}(u)}_{\text{(II)}}, \label{eq:law-of-prob}
\end{align}
which depends on two generally intractable components (I) and (II) that we approximate.

We provide a tractable approximation via the relation $\Pi_{\text{B}} \approx \widehat{\Pi}_{\tau}$ as defined in \eqref{eq:prob-measures-tau-widehat}, so that (I) and (II) are approximated with $\bbP( F(x^*_{1:N_*}) \in A \mid  {\widehat{F}}^{1:M}=u)$ and  $\tau = \widehat{\Pi}_{\infty}^{1:M}$, respectively.
Here, the approximation of (II) relies on the reasonable assumption that
\begin{IEEEeqnarray}{rCl}
    \widehat{\Pi}^{1:M}_{\infty} \approx \widehat{\Pi}^{1:M}_{\text{B}}, \label{eq:approximation-idea}
\end{IEEEeqnarray}
which is motivated by the developments of \Cref{sec:fin-dim-projection-of-WGF}.
The approximation of (I) in turn is motivated by the idea that conditioning on $\widehat{F}^{1:M}$ is approximately \textit{sufficient}, so that the additional conditioning on $y_{1:N}$ is not required, motivating 
\begin{IEEEeqnarray}{rCl}
\bbP( F(x^*_{1:N_*}) \in A \, | \,  y_{1:N}, {\widehat{F}}^{1:M}=u) \approx \bbP( F(x^*_{1:N_*}) \in A \, | \,  {\widehat{F}}^{1:M}=u) \; \text{ for all } u \in \bbR^M.\label{eq:sufficiency-asumption}
\end{IEEEeqnarray}
This approximation is intimately related to one of the key ingredients of SVGPs, whose variational approximation similarly suppress dependence on $y_{1:N}$, and only condition on $M\ll N$ so-called inducing points \citep[see][]{titsias2009variational}.

Plugging \eqref{eq:sufficiency-asumption} into \eqref{eq:law-of-prob} now gives rise to our final approximation $\widehat{\Pi}_{\infty}\in \mathcal{P}_2(H_k)$ of $\Pi_{\text{B}}$,
 which for any $B \in \mathcal{B}(H_k)$ is defined as 
 \begin{IEEEeqnarray}{rCl}
     \widehat{\Pi}_{\infty}(B):=  \int \bbP( F  \in B \, | \,  {\widehat{F}}^{1:M}=u)d\widehat{\Pi}^{1:M}_{\infty}(u).
 \end{IEEEeqnarray}
Based on $\widehat{\Pi}_{\infty}$, we can  approximate the posterior at test points $x^*_{1:N_*}$ in \eqref{eq:law-of-prob} via
\begin{align}
    \big(\operatorname{eval}[x^*_{1:N_*}]\#\widehat{\Pi}_{\infty})(A) = \int \bbP( F(x^*_{1:N_*})  \in A \, | \,  {\widehat{F}}^{1:M}=u) \, d\widehat{\Pi}^{1:M}_{\infty}(u), \label{eq:def-pi-hat}
\end{align}
which can be sampled from by following the algorithm described in Section \ref{sec:PLS-algo}.

In summary, to move from ${\Pi_{\text{B}}}$ to $\widehat{\Pi}_{\infty}$, we used a sequence of steps illustrated in \Cref{fig:PLS-illustration}. This relied on two key ingredients: the approximation of the pushforward by estimated projection in \eqref{eq:approximation-idea}, and the sufficiency condition in  \eqref{eq:sufficiency-asumption}.
Generally, these two approximations do \textit{not} hold with equality.
The one notable exception to this is the setting where both $M=N$, and $\nu = \frac{1}{N}\sum_{n=1}^N \delta_{x_n}$.
For this special case, $\widehat{\Pi}^{1:M}_{\infty} = \widehat{\Pi}^{1:N}_{\text{B}}$ and therefore $\widehat{\Pi}_{\infty}= \Pi_{\text{B}}$, so that the method proposed in Section \ref{sec:PLS-algo} is an exact sampling algorithm for $\Pi_{\text{B}}$.\footnote{
The fact that $\widehat{\Pi}^{1:M}_{\infty} \neq \widehat{\Pi}^{1:N}_{\text{B}}$ can be understood by decomposing $\widehat{\Pi}^{1:M}_{\infty}$ and $\widehat{\Pi}^{1:N}_{\text{B}}$ into their likelihood and prior components. 
When $M < N$, the likelihoods  differ.
For $M=N$, the likelihoods match, and the only difference comes from the prior, and specifically the estimated eigenvalues.
If $M=N$ and we choose $\nu= \frac{1}{N} \sum_{n=1}^N \delta_{x_n}$ for the covariance operator $C$, this difference also vanishes (cf. Appendix \ref{sec:ap:nystroem-sufficient}).}

\begin{figure}[h]
\small
\centering
\includegraphics[width=\linewidth, angle = -90]{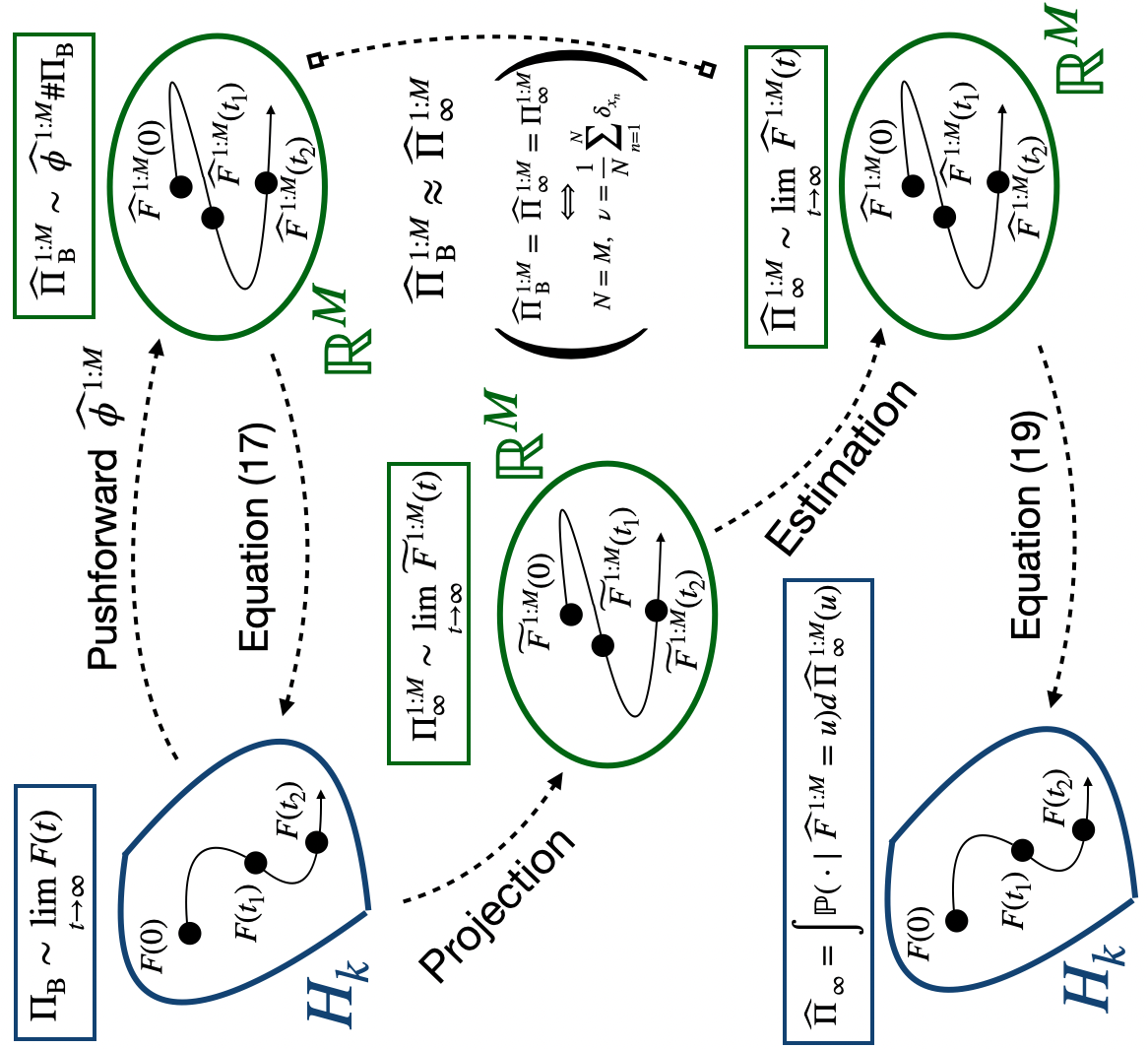} 
  \caption{\small{
  Projected Langevin Sampling (PLS) implements an approximate sampling procedure for targeting the Bayes posterior $\Pi_{\text{B}}$. 
  In particular, an $M$-dimensional Langevin SDE $(\widehat{F}^{1:M}(t))_{t\geq 0}$ is obtained by projection and estimation and approximated using Euler-Maruyama discretisation. 
  Based on the approximation $\widehat{F}^{1:M}(T) \approx \widehat{\Pi}^{1:M}_{\infty}$ for large enough $T$, this allows us to approximate $\Pi_{\text{B}}$ as $\widehat{\Pi}_{\infty}$ via Matheron's Rule, which implements the integration in \eqref{eq:def-pi-hat}.
  %
  %
  %
  }}
  \label{fig:PLS-illustration}
\end{figure}

\subsection{Projected Langevin Sampling (PLS)}\label{sec:PLS-algo}

Whether we choose to conduct inference with $M=N$ or $M\ll N$, the  algorithm we propose is the same.
Drawing 
samples 
that are distributed according to $\operatorname{eval}[x^*_{1:N_*}]\#\Pi_{\infty}$ is straightforward and proceeds in two steps: first, we evolve $J$ independent SDEs as in \eqref{eq:approx-langevin-components} to obtain draws $\widehat{F}_1^{1:M}, \widehat{F}_2^{1:M}, \dots \widehat{F}_J^{1:M}$ from the distribution $\widehat{\Pi}^{1:M}_{\infty}$. Second, we then use Matheron's Rule \citep[see e.g.][]{journel1976mining,wilson2020efficiently}
 to convert them into  samples from $\operatorname{eval}[x^*_{1:N_*}]\#\widehat{\Pi}_{\infty}$ via the conditional distributions on the righthand side of \Cref{eq:sufficiency-asumption}. 
    As $F$ is a GRE and $\widehat{F}^{1:M}$ is linear in $F$, these distributions are themselves conditionally Gaussian, so that computation scales as $\mathcal{O}(M^3+JM^2)$
 (cf. Appendix \ref{sec:ap:matherons-rule}).
In other words, we obtain samples by
\begin{itemize}
    \item[1.] 
    Sampling the initial conditions of our $J$ prequisite SDEs as $\widehat{F}^{1:M}_j(0) \sim \rho$ where $\rho \in \mathcal{P}(\bbR^M)$ is a user-chosen initial distribution and $j=1,2,\dots J$;
    \item[2.]
    Using the Euler-Maruyama discretisation to simulate SDEs with initial condition $\widehat{F}^{1:M}_j(0)$ forward until time $T$ according to \eqref{eq:approx-langevin-components}, thereby obtaining $J$ approximate\footnote{This approximation is due to the finite amount of time $T$ and the discretisation error of the SDE, and vanishes the finer the discretisation gets and the larger $T$ becomes.} samples $\widehat{F}^{1:M}_j(T) \sim \widehat{\Pi}^{1:M}_{\infty}$, for $j=1,2,\dots J$;
    \item[3.]
    Applying Matheron's rule \citep{journel1976mining,wilson2020efficiently} to implement the integration on the right of \eqref{eq:vi-loss} for any arbitrary set of inputs $x^*_{1:N_*} \in \calX^{N_*}$.
    In particular, we obtain the $J$ approximate posterior samples  $F_j(x^*_{1:N_*}) \sim \operatorname{eval}[x^*_{1:N_*}]\#\Pi_{\infty}$ for  $j=1,2, \dots J$ as
    \begin{IEEEeqnarray}{rCl}
        F_j(x^*_{1:N_*}) = G_j(x^*_{1:N_*}) + \widehat{e}^{1:M}(x^*_{1:N_*})^{\top} \Big( \widehat{F}^{1:M}_j(T) - \langle G_j, \widehat{e}^{1:M} \; \rangle \Big), 
        \nonumber
    \end{IEEEeqnarray}
    where  $\widehat{e}^{1:M}(x^*_{1:N_*}) \in \bbR^{M \times N_*}$  is the matrix whose entry at  $(m,n)$ is  $\widehat{e}_m(x^*_n)$,  $\widehat{\Lambda}_M := \text{diag}( \widehat{\lambda}_1, \hdots, \widehat{\lambda}_M) \in \bbR^{M \times M}$ is the diagonal matrix with entries $\widehat{\lambda}_1, \hdots, \widehat{\lambda}_M$,  and where
    $G_j(x^*_{1:N_*}) = (G_j(x^*_1), \dots G_j(x^*_{N_*}))^{\top}$, and $\langle G_j, \widehat{e}^{1:M} \; \rangle = \left( \langle G_j, \widehat{e}_1 \rangle, \dots \langle G_j, \widehat{e}_M \rangle \right)^{\top}$ are random variables sampled as
    $\left(G_j(x^*_{1:N_*}), \langle G_j, \widehat{e}^{1:M}\; \rangle \right)^{\top} \sim \mathcal{N}(0, R_{{N_*},M} )$
   with covariance matrix $R_{N_*, M}$  defined as
    \begin{align*}
    R_{N_*, M} :=
        \begin{bmatrix}
            r(x_{1:N_*}, x_{1:N_*}) &  \widehat{e}^{1:M}(x_{1:N_*})^T \widehat{\Lambda}_M  \\
            \widehat{\Lambda}_M \widehat{e}^{1:M}(x_{1:N_*}) &  \widehat{\Lambda}_M
        \end{bmatrix}   \in \bbR^{(N_*+M) \times (N_*+M)}.
    \end{align*}
\end{itemize}

Note that since each SDE in the first two steps above can be evolved without interaction,  the entire sampling algorithm is embarrassingly parallel over the number of samples $J$.
Thus, while a naive implementation would scale as $\mathcal{O}(M^3+ JM^2)$, parallelisation speeds things up significantly.  
For example, we found that the differences in computation time between $J=1$ and $J=100$ in our parallelised implementation were negligible---something we  demonstrate numerically in \Cref{section:experiments} and particularly in \Cref{fig:computation-PLS}. 

A more detailed version of the algorithm can be found in Appendix \ref{ap:sec:implementation-details}, which also includes further implementation details. 
Since the underlying inferential engine is the projection of the infinite-dimensional Langevin SDE in \eqref{eq:langevin-infinite-dim} into a finite-dimensional presentation as in  \eqref{eq:approx-langevin-components}, we will refer to this algorithm as \textit{projected Langevin sampling} (PLS) throughout the remainder of the paper.

\section{Theoretical Analysis 
and Connections to Previous Methods
}
\label{sec:theory}

So far, the motivation for our algorithm was guided heuristically through two approximations (I) and (II) given by  \eqref{eq:approximation-idea} and  \eqref{eq:sufficiency-asumption}.
%
%
In the remainder, we demonstrate that these approximations are sound, and that the resulting method can be rigorously justified.
%

To do so, we first study the theoretically optimal approximation $\Pi^*_{M}$ for $\Pi_{\text{B}}$ for the class of approximations mapping the first $M$ components of the orthogonal projection in \eqref{eq:projection} back to $\Pi_{\text{B}}$ (cf. Theorem \ref{thm:characterisation-optimal-approximation}).
Perhaps surprisingly, we obtain an exact expression for this theoretically optimal approximation $\Pi^*_{M}$, and use it to assess the quality of $\Pi_{\infty}$ (cf. Theorem \ref{thm:PLS-optimal-close}).
The result is an explicit bound on this difference in terms of the eigenvalues $\{\lambda_m\}_{m>M}$ corresponding to the terms that were left out in the projection of \eqref{eq:projection}, illustrating that the error depends on spectral decay of the eigenvalues $\{\lambda_m\}_{m=1}^\infty$ associated to the covariance operator $C$.
In Lemma \ref{lemma:decay-eigenvalues}, we quantify this error precisely and show that measured by the Kullback-Leibler divergence, the discrepancy between $\Pi_{\infty}$ and $\Pi^*_{M}$ decays rapidly as $M$ increases.  
Further, we show that for the special case of Gaussian likelihoods, the approximation $\Pi_{\infty}$, the measure $\Pi^*_{M}$, and SVGP posteriors all coincide (Lemma \ref{lemma:Gaussian-case-optimal}).

\subsection{Assumptions and Notations}

For the remainder, we will assume that we can rewrite $-\log p(y_{1:N}|f) =: \ell_N( f(x_{1:N}))$ for an appropriately defined  function $\ell_N: \bbR^N \to \bbR$ that is allowed to depend on $y_{1:N}$. This is strictly more general than the form assumed for the derivation of the WGF in \eqref{eq:loss-pointwise}\footnote{If $-\log p(y_{1:N}|f) = \sum_{n=1}^N c(y_n, f(x_n))$, we can always take $\ell_N( f(x_{1:N})) = \sum_{n=1}^N c(y_n, f(x_n))$.}, and more notationally convenient.

Throughout our theoretical developments, we further assume that we have access to the $M$ largest eigenvalue-eigenfunction pairs $\{\lambda_m, e_m\}_{m=1}^M$ of the covariance operator $C$.
In other words, our theory relates to 
the limiting measure of the SDE in \eqref{eq:components-Langevin-2} given by $\Pi^{1:M}_{\infty}$.
More specifically, for $F^m = \langle F, e_m \rangle$ for $m=1,\hdots,M$ and $F^{1:M} = (F^1, \dots F^M)^{\top}$, our theoretical developments study  the distribution given by
\begin{IEEEeqnarray}{rCl}
    \Pi_{\infty}(B) = \int \bbP( F \in B \, | \, F^{1:M}= u ) \, d \Pi_{\infty}^{1:M}(u),
    \label{eq:Pi-hat-for-theory}
\end{IEEEeqnarray}
for all $B \in \mathcal{B}(H_k)$.
This will allow us to ignore the estimation error of $\{\widehat{\lambda}_m, \widehat{e}_m\}_{m=1}^M$ and simplify the already challenging mathematical arguments. For certain choices of $k$ and $\nu$, the decomposition $\{\lambda_m, e_m\}_{m=1}^M$ is known exactly  \citep{zhu1997gaussian}. However, in practice, using Nyström's method to estimate it induces some estimation error.  Nevertheless, this error is generally considered negligible -- especially for the $M$ largest eigenvalue-eigenfunction pairs (see for instance Section 4.3.2 in \citet{rasmussen2003gaussian} and
\citet{koltchinskii2000random}). While the differences between $\{\widehat{\lambda}_m, \widehat{e}_m\}_{m=1}^M$  and $\{{\lambda}_m, {e}_m\}_{m=1}^M$ mean that our theory leaves the approximation $\Pi_{\infty}^{1:M} \approx \widehat{\Pi}^{1:M}_{\infty}$ unaccounted for, the resulting analysis is still meaningful for the PLS algorithm proposed in Section \ref{sec:projected-langevin}.
Notably, when $N=M$ and $\nu = \frac{1}{N}\sum_{n=1}^N\delta_{x_n}$, this estimation error vanishes completely, so that $\widehat{\Pi}_{\text{B}}^{1:M} = \Pi_{\infty}^{1:M} = \widehat{\Pi}^{1:M}_{\infty}$.


\subsection{Characterising Optimal Approximations}

In Theorem \ref{thm:PLS-optimal-close}, we will show that $\Pi_{\infty}$ is nearly optimal amongst a class of nonparametric variational posteriors.
To define this class, we re-interpret $\Pi_{\infty}^{1:M}$ in \eqref{eq:Pi-hat-for-theory} as a parameter that takes values on $\mathcal{P}_2(\bbR^M)$ and which indexes our variational approximation.
To make this logic explicit, we define for each $\tau \in \mathcal{P}_2(\bbR^M)$ the measure
\begin{IEEEeqnarray}{rCl}
 {\Pi}_{\tau} := \int \bbP\big( F \in \cdot \, | \, F^{1:M}= u \big) \, d \tau(u). \label{eq:prob-measures-tau}
\end{IEEEeqnarray}
For each choice of $M$, the collection of all measures like this defines the non-parametric variational family conditioning on the \textit{unapproximated} components $F^{1:M}$ given by
\begin{IEEEeqnarray}{rCl}
    \mathcal{Q}_M := \left\{{\Pi}_{\tau}: \tau \in \mathcal{P}_2(\bbR^M) \right\} \subset \mathcal{P}_2(H_k).
    \label{eq:nonparametric-variational-family}
\end{IEEEeqnarray}
Notice also that as previously remarked upon in the discussion of \eqref{eq:def-pi-hat}, we know that $\Pi_{\text{B}} \in \mathcal{Q}_N$  for the special case where $\nu$ is chosen to be the empirical distribution on the data, since the Nystr\"om approximation is exact in this case so that $\widehat{F}^{1:M} = F^{1:M}$  (cf. Appendix \ref{sec:ap:nystroem-sufficient}). 
More generally and perhaps  more surprisingly, we can even give a closed form for the optimal variational approximation to $\Pi_{\text{B}}$ in $\mathcal{Q}_M$---even when $M<N$ and $\Pi_{\text{B}} \notin \mathcal{Q}_M$. 
%
The next result derives this optimal variational approximation.

\begin{theorem}
\label{thm:characterisation-optimal-approximation}
For the optimal variational approximation $\Pi^{*}_{M}$ of  $\Pi_{\text{B}}$ over $\mathcal{Q}_M$ given by
\begin{IEEEeqnarray}{rCl}
\Pi^{*}_{M} := \argmin_{Q \in \calQ_M} \KL( Q ,\Pi_{\text{B}}),
\nonumber
\end{IEEEeqnarray}
we have that $\Pi^{*}_{M} = {\Pi}_{\tau^*}$ with probability measure $\tau^*(u) \propto \exp\left( -V^*(u) \right)$, where
\begin{IEEEeqnarray}{rCl}
    V^*(u) := - \log \frac{d \tau^*}{du}(u) =  \bbE_{\xi} \Big[ \ell_N \big( \mu_u(x_{1:N}) + \sqrt{ \Sigma(x_{1:N})} \xi \big) \Big]  + u^T \frac{1}{2} \Lambda_M^{-1} u \label{eq:optimal-V-SVGP}
\end{IEEEeqnarray}
for all $u \in \bbR^M$. Here, $\xi \sim \mathcal{N}(0,I_M)$, $\Lambda_M := \operatorname{diag}(\lambda_1,\hdots,\lambda_M)$, and for $u \in \bbR^M$,
\begin{IEEEeqnarray}{rCl}
\mu_u(x_{1:N}) &:=& u^T e^{1:M}(x_{1:N})  \in \bbR^N, \label{eq:def-mu-u}\\
\Sigma(x_{1:N})&:=& r(x_{1:N},x_{1:N}) - e^{1:M}(x_{1:N})^T \Lambda_M e^{1:M}(x_{1:N}) \in \bbR^{N \times N},
\nonumber
\end{IEEEeqnarray}
where $e^{1:M}(x_{1:N}) := \big( e^m(x_n) \big)_{m,n=1}^N \in \bbR^{M \times N}$.
\end{theorem}
Intriguingly, for the special case of Gaussian likelihood functions, $V^*$ in the above result is quadratic, so that $\tau^*$ is a Gaussian measure.
This allows us to relate our findings to SVGPs, which variationally approximate GP posteriors using a version of \eqref{eq:prob-measures-tau} that forces $\tau$ to be a Gaussian measure on $\bbR^M$.
%
In light of this, an implication of Theorem \ref{thm:characterisation-optimal-approximation}  is that an SVGP using $F^{1:M}$ as its inducing features is optimal \textit{only} under Gaussian likelihoods.
In contrast, outside of the Gaussian likelihood setting, the Gaussianity enforced upon $\tau$ implies that SVGPs are increasingly poor approximations the more non-Gaussian the measure $\tau^*$.



\subsection{Optimal Approximations and PLS}

Having established the optimal approximation $\Pi^*_M$ over \eqref{eq:nonparametric-variational-family} in Theorem \ref{thm:characterisation-optimal-approximation}, our next result shows that the approximation $\Pi_{\infty}$ targeted by PLS is provably close to $\Pi^*_M$  under standard regularity assumptions on the negative log likelihood $\ell_N$.  
To do this, we first derive the explicit form for the limiting measure $\Pi^{1:M}_{\infty}$ featuring in $\Pi_{\infty}$ as per \eqref{eq:Pi-hat-for-theory}.
In particular, we can show (cf. Appendix \ref{ap:sec:asymptotics-pls}) that $\Pi^{1:M}_{\infty}(u) \propto \exp\left( - V_{\infty}(u) \right)$ for all $u \in \bbR^M$ and 
\begin{IEEEeqnarray}{rCl}
   V_\infty(u) &:= \ell_N\big( \mu_u(x_{1:N})\big) +  \frac{1}{2} u^T \Lambda_M^{-1} u, 
    \nonumber
\end{IEEEeqnarray}
where we use  $\mu_u(x_{1:N}) \in \bbR^N$ as defined in \eqref{eq:def-mu-u}.
To arrive at this result, simply recognise that the SDE in \eqref{eq:approx-langevin-components} is a finite-dimensional Langevin diffusion, and that $V_\infty$ is the  potential associated to this diffusion.

In a similar vein, $V^*$ defined in Theorem \ref{thm:characterisation-optimal-approximation} can  be interpreted as the potential of a Langevin diffusion with stationary distribution $\tau^*$.
This raises an immediate question: given that $\tau^*$ parameterises the optimal variational measure $\Pi^*_M$, what stops us from constructing a conventional Langevin sampler based directly on $V^*$?
The answer is simple: $V^*$ depends on an intractable  $M$-dimensional integral.
In contrast, $V_{\infty}$ is tractable, so that constructing PLS as the Langevin diffusion that draws approximate samples from $\tau_{\infty}$ is computationally feasible.

While  $V_{\infty}$ and $V^*$ are very different in terms of the computational effort required to evaluate them, there is reason to hope that their numerical differences are relatively small.
Specifically, a close inspection  reveals that $V_\infty$ is a simple approximation to $V^*$: instead of evaluating the intractable integral by averaging over $\xi \sim \mathcal{N}(0,I_M)$ as in $V^*$, the potential $V_{\infty}$ uses a somewhat coarse approximation, and evaluates the integrand only at the mode for $\xi=0$. 
While crude, this is not a bad approximation strategy for sufficiently large $M$.
In fact, the next result quantifies the resulting error between $\Pi_{\infty}$ and $\Pi^*_M$ in KL divergence, and shows it to be negligible whenever the eigenvalues $\{\lambda_m\}_{m>M}$ decay sufficiently quickly and the negative log likelihood is both Lipschitz continuous and convex.

\begin{theorem}\label{thm:PLS-optimal-close}
Assume that for some $\kappa >0$, $\ell_N:\bbR^N \to \bbR$ is a  $\kappa$-Lipschitz continuous and convex function.
Then for any fixed $x_1,\hdots,x_N \in \calX$, we have
\begin{align}
    \KL\big( {\Pi}_{\infty}, \Pi^{*}_{M} \big) \le \frac{  \kappa^2}{2} \mathrm{tr}\big[\Sigma(x_{1:N}) \big] = \frac{\kappa^2}{2} \sum_{m=M+1}^\infty \lambda_m \sum_{n=1}^N ( e_m(x_n))^2. \label{eq:upper-bound1}
\end{align}
Further, if $x_1,\hdots,x_N \in \calX$ are independently and identically distributed according to $\nu$, then
\begin{align}
    \bbE_{x_{1:N}} \Big[  \KL\big( \Pi_{\infty}, \Pi^{*}_{M} \big) \Big] \le \frac{N \kappa^2}{2} \cdot R_M\label{eq:upper-bound2}
\end{align}
for the remainder term $R_M = \sum_{m=M+1}^\infty \lambda_m^2$.
\end{theorem}

If the remainder term $R_M$ in \eqref{eq:upper-bound2} decays sufficiently fast, the above result provides a strong performance guarantee. 
Fortunately, the next result shows that $R_M$ does indeed decay at sufficient speeds even under no further assumption---and the decay is substantially faster at the expense of a variety of mild assumptions. Its proof is in Appendix
\ref{sec:proof:decay-eigenvalues}.

\begin{lemma}
    \label{lemma:decay-eigenvalues}
    For the remainder term $R_M = \sum_{m=M+1}^{\infty}\lambda_m^2$ in \eqref{eq:upper-bound2}, it holds that
    \begin{itemize}
        \item $R_M = \mathcal{O}(M^{-1})$ regardless of $k$ and $\nu$;
        \item $R_M = \mathcal{O}(\exp(-c M))$ for Gaussian kernels and $\nu$ being a Gaussian measure or a measure with compact support;
        \item $R_M = \mathcal{O}(M^{-3(l+1)})$ for Mat\'ern kernels of order $l+1/2$ and $\nu$ being a uniform measure with compact support.
    \end{itemize}
    %
\end{lemma}

According to the asymptotic decay of the remainder \( R_M \) derived in Lemma \ref{lemma:decay-eigenvalues}, a constant, non-vanishing error persists in \eqref{eq:upper-bound2} for the most general case. 
However, for commonly chosen smoother kernels \( k \) and  \(\nu\) endowed with mild regularity conditions, the message is far more hopeful: for these situations, \( M = \mathcal{O}(\log(N)^{1+\epsilon}) \) and \( M = \mathcal{O}(N^{1/(3l+3) + \epsilon}) \) are sufficient to ensure that this error vanishes for any arbitrarily small \(\epsilon > 0\) as $N\to\infty$. 
Since Gaussian and Mat\'ern kernels are often the default choice for various kernel-based methods in machine learning, we can therefore expect substantial computational savings without deterioration in approximation quality for a large class of practically relevant settings.

%
The above two results show that under Lipschitz continuity and convexity for $\ell_N$, the measure $\Pi_{\infty}$ targeted by PLS is therefore not only computationally efficient, but also guaranteed to be very close to the optimum $\Pi^*_M$. 
Note that these conditions are satisfied in a variety of settings of interest, for instance for the binary classification losses with logistic link functions as in \eqref{eq:bernoulli-model}.
For a comprehensive overview of other losses satisfying these conditions, see \citet{SteChr2008}.

It is also important to note that assuming Lipschitz continuity and convexity for $\ell_N$ is sufficient, but definitely not necessary. 
In line with this, we find that empirically, $\Pi_{\infty}$ continues to be an excellent posterior approximation even when these conditions are violated.
Lipschitz continuity and convexity should therefore be understood as  technical requirements that are likely far too strict.
In fact, the negative log Gaussian likelihood $\ell_N(f(x_{1:N})) = \frac{1}{2 \sigma^2} || y_{1:N} - f(x_{1:N}) ||^2$ is a simple example demonstrating that targeting $\Pi_{\infty}$ can be justified without the prerequisite regularity conditions on $\ell_N$.
Specifically, even though the Gaussian likelihood  is not Lipschitz continuous,  $\Pi_{\infty}$ is in fact  \textit{equal} to $\Pi^*_{M}$. 
%
\begin{lemma}
    \label{lemma:Gaussian-case-optimal}
    If $\ell_N\big(f(x_{1:N} ) \big) = \frac{1}{2 \sigma^2} || y_{1:N} - f(x_{1:N}) ||^2$, then $\Pi^{1:M}_{\infty} = \tau^*$ and ${\Pi}_{\infty}= \Pi^*_{M}$, so that both $\Pi^{1:M}_{\infty}$ and $\Pi_{\infty}$ are Gaussian measures.
\end{lemma}

Both Theorem \ref{thm:PLS-optimal-close} and Lemma \ref{lemma:Gaussian-case-optimal} only consider the discrepancy between $\Pi_{\infty}$ and ${\Pi}^*_M$, and not that between $\Pi_{\infty}$ and the full Bayes posterior $\Pi_{\text{B}}$.
For the Gaussian likelihoods to which this result applies however, this latter discrepancy has been studied thoroughly in \citet{burt2019rates}, which showed that for Gaussian likelihoods, the KL divergence between $\Pi_{\text{B}}$ and ${\Pi}^*_M$ can be upper bounded similarly to \eqref{eq:upper-bound1} and \eqref{eq:upper-bound2}.
Since   $\Pi_{\infty}= {\Pi}^*_M$ for this setting by Lemma \ref{lemma:Gaussian-case-optimal}, the results in \citet{burt2019rates} thus transfer to the target measure $\Pi_{\infty}$ of PLS.
While similar bounds for non-Gaussian likelihoods are \textit{not} known, the fact that $M=N$ and $\nu = \frac{1}{N}\sum_{i=1}^N\delta_{x_i}$ implies $\Pi_{\text{B}} \in \calQ_N$ combined with the typically rapid decay of $\{\lambda_m\}_{m>M}$ should make us hopeful that for large enough $M$ we will have $\Pi^*_M \approx \Pi_{\text{B}}$.

\subsection{Related Approaches}

Continuing the comparison with SVGPs \citep{titsias2009variational}, Lemma \ref{lemma:Gaussian-case-optimal} tells us that the measure $\Pi_{\infty}$ targeted by PLS is not only equal to the optimal variational measure $\Pi^*_M$, but also to the SVGP approximation in the case of a Gaussian likelihood.\footnote{If the GP is specified via kernel $r$, and the inducing features are chosen as $U_m=\langle F, e_m \rangle$} 
Unlike the SVGP, the measure $\Pi_{\infty}$ targeted by PLS does not force $\tau$ in \eqref{eq:nonparametric-variational-family} to be Gaussian, and  is therefore close to the optimal variational measure $\Pi^*_M$ for non-Gaussian likelihoods, too (cf. Theorem \ref{thm:PLS-optimal-close}).
Beyond SVGPs, \citet{hensman2015mcmc} propose a Hamiltonian Monte Carlo (HMC) sampler reliant on the sufficiency condition of SVGP to perform joint posterior inference over kernel hyperparameters and inducing features.
Their method can be recast as performing HMC on potential $V^*$ in
\eqref{eq:optimal-V-SVGP}, where the integral is approximated using Gauss-Hermite quadrature, and where Bayesian inference is extended to the kernel hyperparameters.
Importantly, this prior work has to propose several heuristics to handle the additional intractabilities due to hyperparameter inference, and does \textit{not} show that the proposed HMC sampler can target an approximation of suitable quality.
In contrast,  Theorem \ref{thm:PLS-optimal-close} proves that PLS is close to the KL-optimal $M$-dimensional approximation
$\Pi^*_M$.

\begin{figure}[t!]
\small
\centering
\includegraphics[width=0.32\linewidth]
{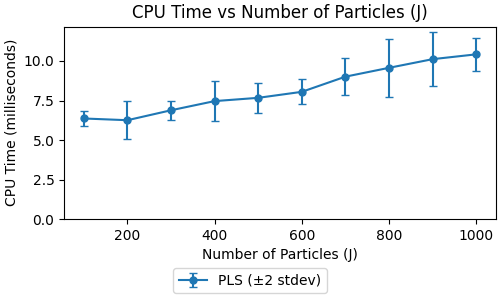} 
\includegraphics[width=0.32\linewidth]
{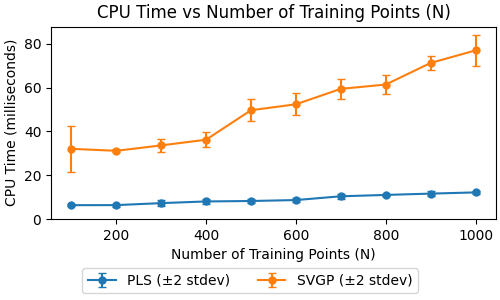} 
\includegraphics[width=0.32\linewidth]
{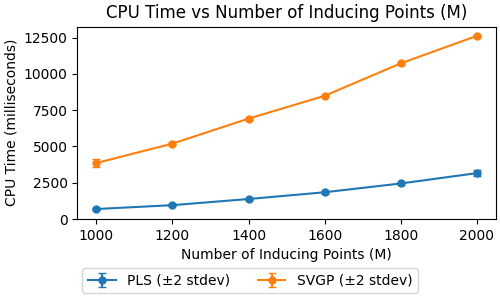} 
  \caption{\small{
    \textbf{Left:}
    Our results show that while the increase in computation time for PLS with $J$ is roughlly linear, it is negligible relative to the dominant $\mathcal{O}(M^3)$ term.
    \textbf{Middle:}
    The computational complexity of PLS in $N$ is linear, but thanks to parallelisation this effect is barely noticable in practice.
    \textbf{Right:}
    The dominant term for computation is $\mathcal{O}(M^3)$ for both PLS and SVGP, but parallelisation and a much smaller constant factor obscured by the $\mathcal{O}$-notation make PLS substantively faster, even for large $M$.
  }}
  \label{fig:computation-PLS}
\end{figure}
 
\section{Numerical illustrations}
\label{section:experiments}

In the remainder of the paper, we illustrate the performance and behaviour of PLS by comparing it to SVGP with inducing features $U_m := F(z_m)$, $m=1,\hdots,M$. 
Notice that PLS relies implicitly on inducing points $z_{1:M}$, since the estimation of the $M$ largest eigenvalue-eigenfunction pairs via the Nyström method is based on the kernel matrix $\frac{1}{M} k(z_{1:M}, z_{1:M})$ (see also Appendix \ref{sec:ap:nystroem-method}).
Thus, to make comparisons fair, we use $M= \sqrt{N}$ inducing points for both methods, and rely on the same techniques to choose kernel hyper-parameters and inducing points $z_{1:M}$ (see Appendix \ref{ap:sec:implementation-details} for details). 
For all experiments conducted, implementations are available at \href{https://github.com/jswu18/projected-langevin-sampling}{https://github.com/jswu18/projected-langevin-sampling}.

\paragraph{Computational complexity.}
In \Cref{sec:time-and-space-complexity}, we show that the computational complexity of PLS is of order $\mathcal{O}\big(M^3 + J NM\big)$, where $J$ denotes the number of samples drawn.
The $\mathcal{O}(M^3)$ term is due to the spectral decomposition of $\frac{1}{M} k(z_{1:M},z_{1:M})$, which only has to be computed \textit{once} at the start of the algorithm.
Since the algorithm is embarassingly parallel in $J$, this means that the size of $M$ is the main limiting factor in its scalability. 
Our implementation exploits this, and results are obtained after parallelising over the 8 CPU cores of a MacBook Pro with an M1 Pro chip and 16GB of memory, and comparing against the standard SVGP implementation in  \texttt{GPyTorch} \citep{gardner2018gpytorch}.
Hyperparameters are fixed across both methods, and we did not batch SVGP to ensure fair comparisons. 
\Cref{fig:computation-PLS} compares the clock times for PLS and SVGP across 10 replications of fitting a simple 1-dimensional regression with $f(x) = 2 \sin(0.35 \pi x^2)$ uniformly sampled on $x \in [-3,3]$ and injected with independent Gaussian noise distributed as $\mathcal{N}(0, 0.2)$.
Here, the default settings we consider are $N=100$, $M=10$, and $J=100$. To investigate their effects on computation time, we vary these three different parameters one after the other in our plots.
%


\paragraph{Regression.} 
Table \ref{tab:regression-class} compares SVGP to PLS for regression tasks on various benchmark data sets. 
Since the regression likelihood is Gaussian, we know both SVGP and PLS target the optimal $M$-dimensional approximation $\Pi^*_M$ (cf. Lemma \ref{lemma:Gaussian-case-optimal}). 
Unsurprisingly, we obtain similar behaviour, although PLS seems to have an edge overall.
This is perhaps not surprising: since we construct the covariance operators  in Section \ref{sec:cov-op-gp-interpretation} by taking $\nu$ to be the empirical measure over the covariates, the kernel $r$ that is implicitly used within PLS is effectively data-adaptive.

\begin{table}[b!]
    \caption{ \small{
    Comparison between PLS and SVGP on various classification benchmark data sets.
    For each data set, results are obtained for five different train/test splits, and we report both mean as well as standard deviation (in brackets). 
    \textbf{Fat text} denotes better performance, and $^*$ / $^{**}$ denotes statistical significance at a 95\% / 99\% level using a 2-sample $t$-test.
    }}
        \centering
{
\resizebox{0.8\textwidth}{!}{
        \begin{tabular}{lll|ll}

         & \multicolumn{4}{c}{{{\textbf{Bernoulli Likelihood}}}} \\
         \toprule
                  & \multicolumn{1}{c}{{PLS}}       & \multicolumn{1}{c}{SVGP} & \multicolumn{1}{c}{{PLS}}       & \multicolumn{1}{c}{SVGP}     \\ \midrule
        &    \multicolumn{2}{c}{{{\small \textit{Area Under the Curve}}}}  &    \multicolumn{2}{c}{{{\small \textit{Accuracy}}}}\vspace{0.5mm}                          \\
Breast & 98.37 (0.90) & \textbf{98.44 (0.81)} & 94.74 (1.48) & \textbf{94.74 (0.83)}\\
Diabetes & \textbf{83.38 (2.93)} & 82.95 (2.59) & \textbf{76.56 (4.33)} & 76.46 (4.44)\\
Heart & \textbf{91.34 (0.75)} & 91.21 (0.98) & \textbf{83.35 (1.83)} & 83.11 (1.41) \\
Ionosphere & 90.94 (2.85) & \textbf{92.94 (1.71)} & 85.91 (4.47) & \textbf{87.27 (5.35)}\\
Mushrooms & \textbf{81.77 (2.09)} & 81.76 (1.95) & \textbf{75.64 (1.64)} & 74.83 (2.85)\\
Rice & \textbf{94.65 (1.54)} & 94.04 (1.40) & \textbf{88.96 (1.14)} & 88.46 (0.56)\\
Wine (colour) & 94.95 (1.67) & \textbf{94.99 (1.34)} & \textbf{93.98 (0.63)}** & 85.98 (3.34)\\
Yeast & \textbf{69.16 (2.39)} & 67.23 (1.03)& \textbf{64.93 (2.20)} & 64.13 (1.30) \\
        \bottomrule
        \end{tabular}
        }\label{tab:binary-classification}
        }
\end{table}

\paragraph{Binary Classification.} Complementing the regression results,  Table \ref{tab:binary-classification} compares PLS and SVGP on a binary classification task. 
Similarly to the regression loss, the classification loss satisfies the necessary requirements for the optimality result in Theorem \ref{thm:optimal-measure-calQ} to apply.\footnote{In particular, the classification loss is convex and Lipschitz continuous.}
In contrast, SVGP is parametrically constrained to a Gaussian variational distribution, and will therefore introduce an approximation error whose magnitude cannot be controlled. 
Fortunately, the Bernoulli likelihood is still convex, so that the optimal $M$-dimensional posterior characterised in Theorem \ref{thm:optimal-measure-calQ} will be \textit{unimodal}---and therefore usually close enough to a Gaussian distribution for SVGP to still perform well. 
Unsurprisingly then, both methods perform relatively similar. 
While the Table seems to suggest that PLS outperform SVGP, this is an intellectually lazy reading of the numbers: the differences between both methods are not statistically significant.

\begin{table}[hb!]
    \caption{ \small{Comparison between PLS and SVGP on various regression benchmark data sets.
    For each data set, results are obtained for five different train/test splits, and we report both mean as well as standard deviation (in brackets). 
    \textbf{Boldface} denotes better performance, and $^*$ / $^{**}$ denotes statistical significance at a 95\% / 99\% level using a 2-sample $t$-test.
    As expected, SVGP and PLS perform similarly, though PLS performs better overall. 
    }}
    \resizebox{\textwidth}{!}{
\begin{tabular}{lll|ll}
 & \multicolumn{2}{c}{{{\textbf{Gaussian Likelihood}}}}  & \multicolumn{2}{c}{{{\textbf{Student-T Likelihood}}}} \\
 \toprule
                   & \multicolumn{1}{c}{{PLS}}       & \multicolumn{1}{c}{SVGP}  & \multicolumn{1}{c}{PLS}       & \multicolumn{1}{c}{SVGP}     \\ 

\midrule
  &    \multicolumn{4}{c}{{{\small \textit{Negative Log Likelihood}}}} 
  \vspace{0.5mm} \\ 
Boston & 1.075 (0.427) & \textbf{1.026 (0.267)} & 1.359 (0.819) & \textbf{1.174 (0.384)} \\
Concrete & \textbf{1.182 (0.112)}** & 1.296 (0.096)& \textbf{1.277 (0.099)}** & 1.289 (0.081) \\
Energy (cooling)& \textbf{0.769 (0.073)}** & 1.289 (0.032) & \textbf{0.948 (0.091)}** & 1.299 (0.045) \\
Energy (heating) & 0.407 (0.145) & \textbf{0.253 (0.113)}** & 0.425 (0.159) & \textbf{0.287 (0.128)}** \\
Kin8Nm & \textbf{0.839 (0.027)}** & 0.914 (0.022)& \textbf{0.908 (0.032)}** & 0.915 (0.025) \\
Wine (quality) & \textbf{1.345 (0.079)}** & 1.361 (0.073) & \textbf{1.353 (0.081)}** & 1.391 (0.088) \\
        \midrule
                  &    \multicolumn{4}{c}{{{\small \textit{Mean Absolute Error}}}}
                  \vspace{0.5mm} \\ 
Boston & \textbf{0.360 (0.067)}** & 0.403 (0.084) & \textbf{0.368 (0.058)}**& 0.385 (0.081) \\
Concrete & \textbf{0.620 (0.072)}** & 0.707 (0.072) & \textbf{0.669 (0.059)}** & 0.686 (0.057) \\
Energy (cooling) & \textbf{0.413 (0.019)}** & 0.777 (0.031) & \textbf{0.484 (0.029)}** & 0.765 (0.033) \\
Energy (heating) & 0.237 (0.020) & \textbf{0.235 (0.023)}**& 0.255 (0.031) & \textbf{0.238 (0.026)}** \\
Kin8Nm & \textbf{0.423 (0.012)}** & 0.473 (0.009) & \textbf{0.441 (0.012)}** & 0.468 (0.011) \\
Wine (quality) & \textbf{0.766 (0.037)}** & 0.782 (0.034)  & \textbf{0.774 (0.040)}** & 0.787 (0.036) \\
\midrule
                  &    \multicolumn{4}{c}{{{\small \textit{ Average Interval Width}}} }
                  \vspace{0.5mm} \\ 
Boston & \textbf{2.30 (0.35)}* & 2.48 (0.57) & 2.584 (0.548) & \textbf{2.539 (0.758)}* \\
Concrete & \textbf{3.54 (0.70)} & 3.58 (0.68) & \textbf{3.794 (0.840)}** & 3.805 (0.809) \\
Energy (cooling) & \textbf{2.11 (0.30)}** & 3.17 (0.26) & \textbf{2.471 (0.300)} & 3.451 (0.393) \\
Energy (heating) & 1.62 (0.24) & \textbf{1.39 (0.17)}**  & 1.487 (0.187) & \textbf{1.398 (0.248)}** \\
Kin8Nm & \textbf{2.21 (0.05)}** & 2.39 (0.11)  & \textbf{2.369 (0.109)}** & 2.386 (0.125) \\
Wine (quality) & 3.62 (0.32) & \textbf{3.62 (0.36)}  & 3.653 (0.329) & \textbf{3.644 (0.289)}** \\
\bottomrule
\label{tab:regression-class}
\end{tabular}
 }
\end{table}

\paragraph{Poisson Regression.} 
In Figure \ref{fig:multimodal-poisson}, we consider a Poisson regression with an unknown rate function modeled as $\lambda(x) = f(x)^2$ using simulated data. 
Doing so corresponds to the likelihood function $ p(y_n|f) = (y_n!)^{-1} (f(x_n))^{2y_n} \exp( - (f(x_n))^2)$ for $y_n \in \bbN_0$.
Importantly, this likelihood function is symmetric with $p(y_n\mid f) = p(y_n\mid -f)$, and therefore non-convex in $f$.
As a consequence, Theorem \ref{thm:PLS-optimal-close} does not apply and we do not have guarantees for PLS. 
Despite this, PLS still performs remarkably well: it perfectly models the bimodal nature of the posterior, which manifests in  the symmetry of the posterior uncertainty around the x-axis in Figure \ref{fig:multimodal-poisson}. 
This is in stark contrast with SVGP, whose unimodal Gaussian variational posterior would prevent it from taking any multimodal shape.
%



\begin{figure}[h!]
\small
\centering
  \includegraphics[width=\linewidth]{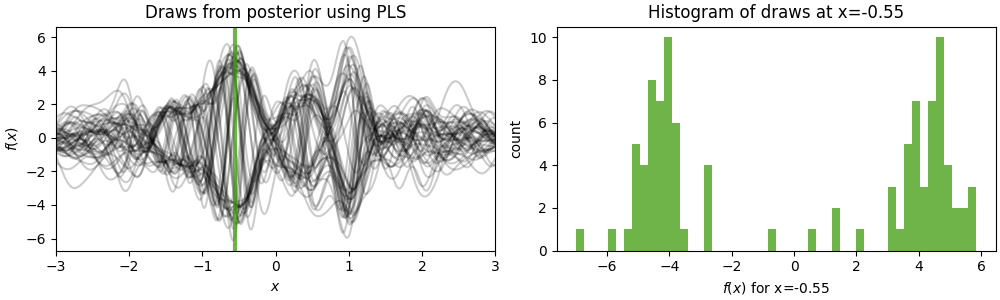} 
  \caption{
  \small{Posterior inference for a  Poisson regression with rate function $\lambda(x) = f(x)^2$. 
 \textbf{Left:} Displayed are $J=100$ posterior sample draws from PLS. Due to the symmetry of the associated likelihood function around $f=0$, we observe symmetry around the x-axis. 
  \textbf{Right:} Depicted is a histogram displaying the location of the posterior draws from PLS on the left at $x=-0.55$. The shape of the histogram reflects the observed bimodality, and is decidedly non-Gaussian. 
  }}
  \label{fig:multimodal-poisson}
\end{figure}

\paragraph{Multimodal Regression with presence of an unknown shift.}
While the previous examples served to illustrate quantitative and qualitative differences in the inferences between PLS and SVGP, we now give an illustration that demonstrates the capability of our method to efficiently sample from complex and exotic posterior distributions with relative ease.
In particular, we will show that as long as \( c \) and \( \partial_2 c \) as defined in Section \ref{sec:loss} are tractable, then PLS can be applied without any further model-specific adjustments.
To do this, we posit the following relationship between the observations \( y_n \) and the input features \( x_n \), given for $n=1,2,\dots N$ by
\begin{IEEEeqnarray}{rCl}
    y_n = f(x_n) + s \cdot b + \epsilon_n,
    \nonumber
\end{IEEEeqnarray}
where \( s \in \mathbb{R} \) is a shift parameter of known magnitude, \( b \sim \operatorname{Ber}(\alpha) \) is a Bernoulli random variable with success probability \( \alpha \in (0,1) \), and \( \epsilon_n \sim \mathcal{N}(0, \sigma^2) \) represents Gaussian noise. 
Including a Bernoulli random variable that 'switches on' the shift results in a multimodal likelihood model, as the unknown function \( f \) can only be identified up to the constant shift \( s \) (for details, see Appendix \ref{ap:sec:implementation-details}). 
This model is related to the classical fixed effects approach, a class of models popular in econometrics \citep{wooldridge2010econometric}. 
Intuitively, it describes a setting where we are uncertain whether a measurements $y_n$ has been shifted by a constant $s$ prior to observation.
As shown in Figure \ref{fig:multimodal-shift}, PLS accurately captures the uncertainty we have about whether or not our data points were shifted prior to observation by generating posterior samples that correspond to both scenarios.
%
As a result, it produces a bimodal distribution whose modes are apart by a distance of almost exactly \( s \).

\begin{figure}[ht]
\small
\centering
  \includegraphics[width=\linewidth]{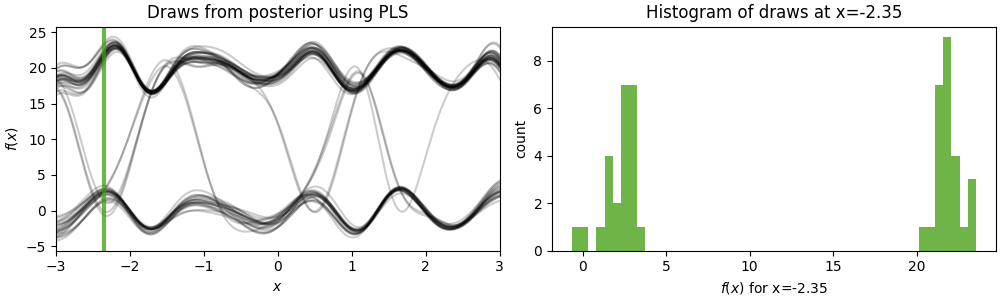} 
  \caption{
  \small{
  Posterior inference for a model with possible constant shift \( s \in \mathbb{R}^N \). 
  The underlying data was simulated using \( f(x) = 2\sin(1.5\pi x) \), $\alpha = 0.5$, \( s = 20 \), and $\sigma^2 = 1$.
 \textbf{Left:} Displayed are $J=50$ posterior sample draws from PLS over $x \in [-3,3]$ with a vertical line indicating $x=-2.35$. Due to the nature of the model, we observe two different groups of draws: one at the top and one of the bottom. Unsurprisingly, the two groups of draws are separated by a distance of about  $s$. 
  \textbf{Right:} Depicted is a histogram displaying the location of posterior draws from PLS on the left at $x=-2.35$. The histogram reflects the observed bimodality and non-Gaussianity. 
  }}
  \label{fig:multimodal-shift}
\end{figure}

\section{Conclusion}
To summarise, we derived the Wasserstein gradient flow for Bayesian inference with a Gaussian random element prior distribution, and for arbitrary likelihood functions.
We demonstrated that the WGF can be efficiently implemented by a projection of the Langevin equation from a reproducing kernel Hilbert space (RKHS) into Euclidean spaces, which led to a new method we call projected Langevin sampling (PLS). 
Our results show that the posterior targeted by this projected sampler coincides with the optimal $M$-dimensional posterior approximation for Gaussian likelihoods.
Beyond that, we showed that it is the first method to be provably close to the optimal $M$-dimensional posterior approximation for non-Gaussian likelihoods under mild convexity and Lipschitz assumptions. 
We concluded by exploring the performance of
PLS on several regression and classification tasks, compared it to Sparse variational GPs (SVGPs), and provided demonstrations for some of its capabilities for inference in multi-modal posteriors that are impossible to achieve with SVGPs. 

Throughout this development, we made a key observation that motivated much of the paper's focus: in Section \ref{sec:BKR-BKC}, we discovered that in order to implement a WGF-powered functional inference scheme, the functions of interest \textit{needed} to lie in an RKHS. 
It is worth reflecting upon this: we did not choose RKHS spaces \textit{directly}.
Rather, they emerged as a consequence of 
assuming that the relevant function space is a Hilbert space, and that the negative log likelihood is Fr\'echet differentiable. 
Although it is not obvious how these assumptions could be relaxed while ensuring that the WGF derived in Section \ref{sec:WGF:SDE} can be tractably evolved on a computer, it raises the fundamental question if weaker assumptions could be found to generalise our algorithm to functional inference such as more general Hilbert spaces or  Banach spaces.
While we believe that these questions could establish an exciting new frontier for functional inference, the underlying mathematics would have to fundamentally depart from the results we derived here.

Turning attention to the methodological gains made in our paper, it is noteworthy that PLS can compete with SVGPs along several dimensions.
On top of its ability to make substantive computational savings by being run in parallel, one advantage of our approach seemed to be the ability to accommodate multi-modal posterior distributions.
While our illustrations provided evidence that PLS can in principle perform multi-modal inference in fortuitous circumstances, it is worth noting that PLS is essentially a Langevin sampler---a class of algorithms well-known to struggle with multi-modal inference. 
An important future direction of research on using the WGF to derive functional inference algorithms would be the derivation of a different class of algorithms that is more naturally suited to multi-modal posteriors.
Doing so would require a fundamentally new approach:
the theoretical guarantees we derived for PLS depended on convexity of the negative log likelihood function---but multi-modality  emerges mostly when they are non-identifiable and non-convex.
Though the multi-modal setting would require a completely different mathematical toolkit from the one developed in the current paper and is beyond the scope of what we studied here, it poses intriguing questions for further research into functional inference based on WGFs.

\subsection*{Acknowledgements}

VW was supported by the Scatchered scholarship and  EPSRC grant EP/W005859/1.
JK was supported by  EPSRC grants EP/W005859/1 and EP/Y011805/1.


\renewcommand{\theHsection}{A\arabic{section}}
\newpage
\appendix

\section{Technical background}\label{sec:ap:technical-background}

\subsection{Reproducing Kernel Hilbert Spaces}\label{ap:sec:RKHS}

Let $\calX$ be non-empty set and $\mathcal{F}(\calX, \bbR)$ be the set containing all functions $f: \calX \to \bbR$. A Hilbert space $\big( H, \langle \cdot, \cdot \rangle\big)$ with $H \subset \mathcal{F}(\calX, \bbR)$ is called Reproducing Kernel Hilbert Space (RKHS) if and only if the pointwise evaluation functionals $\pi_x : H \to \bbR$, $f \mapsto f(x)$ are continuous for all $x \in \calX$.
A function $k:\calX \times \calX \to \bbR$ is called kernel, if the kernel matrix $k(x_{1:N},x_{1:N}) := \big( k(x_n, x_{n'}) \big)_{n,n'=1}^N \in \bbR^{N \times N}$ is positive semi-definite for all $x_{1:N} \in \calX^N$ and $N \in \bbN$.

 Moore-Aronszajn theorem \citep{aronszajn1950theory} states that, for every RKHS $H$, there exists a kernel $k:\calX \times \calX \to \bbR$ with the property that $k(x, \cdot) \in H$ and  $\pi_x(f) = \big\langle f, k(x,\cdot) \big\rangle$ for all $x \in \calX$. The functions $k(x, \cdot) : \calX \to \bbR$ are called canonical feature maps and the property $f(x) = \langle f , k(x, \cdot) \rangle$ is called reproducing property.

Conversely, for any kernel $k$, we can construct a Hilbert space $H \subset \mathcal{F}(\calX, \bbR)$ as closure of $H_0$ defined as
\begin{align}
    H_0:=\left\{ \sum_{n=1}^N \alpha_n k(x_n, \cdot) \, | \, \,  \alpha_n \in \bbR, \,  x_n \in \calX ,  \, N \in \bbN \right\}
\end{align}
with respect to norm induced by the inner product $\langle \sum_{n=1}^N \alpha_n k(x_n, \cdot), \sum_{n=1}^N \beta_n k(\widehat{x}_n, \cdot) \rangle := \sum_{n,n'=1}^N \alpha_n \beta_n k(x_n,\widehat{x}_{n'}) $ for all $\alpha_n, \beta_n \in \bbR$, $x_n, \widehat{x}_n \in \calX$ and $N \in \bbN$. The Hilbert space constructed in this way is an RKHS with kernel $k$ \citep[Theorem 10.10]{wendland2004scattered}.  

Additional details about RKHS can be found in \citet{Berlinet2004}, \citet{wendland2004scattered} and \citet{SteChr2008}. Furthermore \cite{scholkopf2002learning} and \citet{hofmann2008kernel} describe how the theory of RKHS can be used in the context of machine learning.

\subsection{Gaussian Random Elements and Gaussian Measures}\label{ap:sec:GREs}
Let $\big( \Omega, \mathcal{A}, \bbP)$ be the underlying physical probability space and $\big( H, \langle \cdot, \cdot \rangle \big)$ a (separable) Hilbert space.

The measurable mapping $F: \Omega \to H$ is called Gaussian Random Element (GRE) in $H$ if $\langle F , h \rangle : \Omega \to \bbR$ is a Gaussian random variable for all $h \in H$ \footnote{In this context a Gaussian random variable $ X \sim \mathcal{N}(\mu,0)$ is defined as Dirac measure at $\mu$} \citep[Definition 4.1]{van2008stochastic}. If $F$ is a GRE then $ \bbE \big[\| F \|^2 \big] < \infty$ \citep[Theorem 4.3]{van2008stochastic} and we can define the mean element $\bbE[F] \in H$ and the covariance operator $\mathbb{C}[F]: H \to H$ via
\begin{align}
    \bbE[F] &:= \int F(\omega) \, d \bbP(\omega)  \\
    \mathbb{C}[F] h &:= \int F(\omega)  \langle F ,h \rangle  \, d\bbP(\omega) - \langle m ,h \rangle m  
\end{align}
for $g \in H$, where the integrals are understood as Bochner integrals. By standard rules for Bochner integrals, one obtains $ \langle \bbE[F] , h \rangle =\bbE \big[ \langle F, h \rangle \big] $, $\langle \mathbb{C}[F] h , g \rangle = \mathbb{C}\big[ \langle F, h \rangle, \langle F,g \rangle \big] $ and consequently $ \langle F , h \rangle \sim \mathcal{N}\big( \langle \bbE[F] , h \rangle, \langle \mathbb{C}[F] h , h \rangle \big)$. 
The covariance operator $\mathbb{C}[F] : H \to H$ is a positive, symmetric, trace class operator \citep[Proposition 2.16]{da2014stochastic}. 
Conversely, for a given mean element $m \in H$ and given  positive, symmetric, trace class operator $C: H \to H$ there exits a GRE $F: \Omega \to H$ such that $\bbE[F] = m$ and $\bbC[F] = C$ \citep[Proposition 2.18]{da2014stochastic}. We write $F \sim \mathcal{N}(m,C)$ for a GRE with mean element $m$ and covariance operator $C$.

A probability measure $P$ defined on the Borel $\sigma$-algebra $\mathcal{B}(H)$ is called Gaussian measure if and only if the push-forward measure $T\#P: \mathcal{B}(\bbR) \to [0,1], B \mapsto T\#P(B) := P\big( T^{-1}(B) \big)$ is a Gaussian measure on $\mathcal{B}(\bbR)$ for all bounded, linear functionals $T:H \to \bbR$. Notice that by Riesz representation theorem every bounded, linear functional $T$ is of the form $T = \langle \cdot, h \rangle$ for a $h \in H$. Hence, by definition, if $F : \Omega \to H$ is a GRE then the push-forward measure $F \# \bbP$ is a GM. Consequently, the statements about GREs outlined above carry over to GMs mutatis mutandis. 

\citet{Bogachev1998} is the authoritative source on Gaussian measures and discusses their properties on general Fréchet spaces. Gaussian measures on Banach and Hilbert spaces are described in \citet{da2014stochastic}. \citet{van2008stochastic} discusses Banach space valued Gaussian random elements.

\section[Wasserstein gradient flow for probability measures on Hilbert spaces]{Wasserstein gradient flow in $\mathcal{P}_2(H)$}\label{ap:sec:WGF}

Let $\big( H, \langle \cdot , \cdot \rangle \big)$ be a separable Hilbert space and 
\begin{align}
    \mathcal{P}_2(H) := \left\{ \mu : \mathcal{B}(H) \to [0, 1] \, \big| \; \mu(H)=1,\, \int_H || u ||^2 \, d\mu(u)  < \infty \right\}
\end{align}
be the space of Borel probability measures on $\mathcal{B}(H)$ with finite second moment.  Here $\mathcal{B}(H)$ denotes the Borel $\sigma$-algebra on $H$. Define further for $Q \in \mathcal{P}_2(H)$ the Bochner spaces
\begin{align}
     L^2(Q) &:= \left\{ f: H \to \bbR \, \big| \, \int \big(f(u)\big)^2 \, dQ(u) < \infty \right\} \\
     L^2(Q;H) &:= \left\{ f: H \to H \, \big| \, \int \langle f(u), f(u) \rangle \, dQ(u)  < \infty \right\} 
\end{align}
where functions are identified $Q$-almost everywhere.

We are interested in calculating the Wasserstein gradient flow for 
\begin{align}
    L(Q) = \int_H \ell(u) \, dQ(u) + \KL (Q, \Pi), \label{ap:eq:vi-loss}
\end{align}
where $Q \in \mathcal{P}_2(H)$ and $\ell: H \to \bbR $ is a Fréchet differentiable function, $\Pi = \mathcal{N}(0,C)$ is a Gaussian measure on $H$ with covariance operator $C: H \to H$ (cf. Appendix \ref{ap:sec:GREs})  and $\KL$ the Kullback-Leibler divergence defined as 
\begin{align}
    \KL(Q, \Pi) := \int \log \left( \frac{dQ}{d\Pi} \right) (u)  \, dQ(u)
\end{align}
for $Q$ dominated by $\Pi$ where $dQ/d\Pi$ denotes the corresponding Radon-Nikodym derivative and $\KL(Q,\Pi) = \infty$ otherwise. 

The first step in deriving the Wasserstein gradient flow is to calculate the Wasserstein gradient. We present a definition for the Wasserstein gradient below which is taken from \citet[Definition 4.2.2]{figalli2021invitation} adjusted to the Hilbert space case. A more general definition can be found in Chapter 11.1 of \citet{ambrosio2005gradient}.

\begin{definition}
 Let $L : \mathcal{P}_2(H) \to [0, \infty]$ be a functional. The Wasserstein gradient at $Q$ is the unique element $\phi \in L^2(Q;H)$ (if it exists) such that 
 \begin{align}
     \frac{d}{dt} \big|_{t=0} L\big( Q(t) \big) = \int \langle \phi(u),  v(u) \rangle \, dQ(u) \label{eq:WG-def2}
 \end{align}
for all smooth curves $\big( Q(t) \big)_{t \in (-\epsilon, \epsilon)} \subset \mathcal{P}_2(H)$ with $Q(0) = Q$ and \textit{tangent vector} (explanation below) $v \in L^2(Q;H)$ at $t=0$. We write $\nabla_W L[Q]: H \to H$ for the Wasserstein gradient $\phi$ at $Q$ if it exists.
\end{definition}

We will now discuss how to construct a curve  $\big( Q(t) \big)_{t \in (-\epsilon, \epsilon)} \subset \mathcal{P}_2(H)$ at $Q$ with tangent vector $v$. Let $u(t;u_0) \in H$ be the solution to the ODE
\begin{align}
    u(0) &= u_0  \label{eq:ODE-1}\\ 
    u'(t) &= v \big( u(t) \big) \label{eq:ODE-2}
\end{align}
for $t \in (- \epsilon, \epsilon)$ with given initial value $u_0 \in H$. We assume that $v \in L^2(Q;H)$ is sufficiently regular so that a solution exits for small enough $\epsilon >0$. Then $\big( Q(t) \big)_{t \in (-\epsilon, \epsilon)} \in \mathcal{P}_2(H)$ with $Q(t):=u(t; \cdot) \# Q$ is a smooth curve in $L^2(H)$ with tangent vector $v$ at $Q \in \mathcal{P}_2(H)$.

Indeed, it suffices to show that \eqref{eq:WG-def2} holds for all curves constructed via \eqref{eq:ODE-1} and \eqref{eq:ODE-2} \citep[Theorem 8.3.1]{ambrosio2005gradient}.

\begin{theorem}\label{thm:WG-KLD}
    The Wasserstein gradient for $L$ in \eqref{ap:eq:vi-loss} is given as
    \begin{align}
        \nabla_W L[Q](u) = D \ell(u) + D \log (dQ/d\Pi)(u)
    \end{align}
    for all $u \in H$ and $Q\in \mathcal{P}_2(H)$ dominated by $\Pi = \mathcal{N}(0,C)$. Notice that $L(Q) = \infty$ whenever $Q$ is not dominated by $\Pi$ and in this case the Wasserstein gradient is undefined.
\end{theorem}

\begin{proof}
Let $\big( Q(t) \big)$ be the curve constructed in \eqref{eq:ODE-1} and \eqref{eq:ODE-2}.We denote by $q(t) := dQ(t)/d\Pi$ the Radon-Nikodym derivative of $Q(t)$ with respect to $\Pi$ which exists for regular enough $v$ in a small enough neighbourhood around $t=0$. 

We start the proof by deriving an equation for the time evolution of $q(t)$. 
The Hilbert space version of the Fokker-Planck equation (FPE) \citep[Section 14.2.2]{da2014stochastic} states that
\begin{align}
    \frac{d}{dt} \int \varphi(u) \, dQ_t(u) = \int \big\langle v(u), \, D \varphi(u) \big\rangle \, dQ_t(u)
\end{align}
for all $\varphi \in C^2_b(H)$. For the LHS of the FPE we have
\begin{align}
    \frac{d}{dt} \int \varphi(u) \, dQ_t(u) =  \int \varphi(u) \partial_t q(t,u) \, d\Pi(u)
\end{align}
and for the RHS of the FPE we obtain 
\begin{align}
    \int \big\langle v(u), \, D \varphi(u) \big\rangle \, dQ_t(u) 
    &= \int \big\langle v(u), \, D \varphi(u) \big\rangle q(t,u) \, d\Pi(u) \\
    & = \int \big\langle v(u) q(t,u), \, D \varphi(u) \big\rangle  \, d\Pi(u) \\
&= \int - \text{Tr} \left(D \big[ v(u) q(t,u) \big] \right) \varphi(u) \, d\Pi(u) \\
&+ \int \big\langle C^{-1} u, v(u) q(t,u)  \big\rangle \varphi(u) \, d\Pi(u)
\end{align}
where the last equality holds due to the integration by parts (IBP) formula for Gaussian measures on Hilbert spaces \citep[Lemma 10.1 and Section 10.4]{da2006introduction} for regular enough $v$. Here $\text{Tr}$ denotes the trace operator and all Fréchet derivatives are with respect to the $u$-variable. Technically speaking $D$ is the Friedrichs operator introduced in Section 10 of \citet{da2006introduction}. However, we will always assume Fréchet differentiability of all quantities involved and in this case $D$ coincides with the Fréchet derivative. Combining both calculation gives the equality 
\begin{align}
     &\int \varphi(u) \partial_t q(t,u) \, d\Pi(u) \\
     &= \int - \text{Tr} \left(D \big[ v(u) q(t,u) \big] \right) \varphi(u) \, d\Pi(u) 
+ \int \big\langle C^{-1} u, v(u) q(t,u)  \big\rangle \varphi(u) \, d\Pi(u)
\end{align}
for all regular enough test functions $\varphi:H \to \bbR$ and consequently
\begin{align}
    \partial_t q(t,u) = - \text{Tr} \left(D \big[ v(u) q(t,u) \big] \right)   
+ \big\langle C^{-1} u, v(u) q(t,u)  \big\rangle \label{eq:density-evolution}
\end{align}
holds for all $u \in H$ and all $t$ in a small enough interval around $t = 0$. Also note that $q(0,\cdot) = q := dQ/d\Pi$ by construction.

We will now calculate the Wasserstein gradient for $L$ at $Q$. By standard rules for Radon-Nikodym derivatives we obtain
\begin{align}
    L\big( Q(t) \big) = \underbrace{\int \ell(u) q(t,u) \, d\Pi(u)}_{(a)} + \underbrace{\int \log q_t(u) q_t(u) \, d\Pi(u)}_{(b)}.
\end{align}
We calculate for (a)
\begin{align}
    \frac{d}{dt} \int \ell(u) q(t,u) \, d\Pi(u) &= \int \ell(u) \partial_t q(t,u) \, d\Pi(u) \\
    &= \int \ell(u) \left(  - \text{Tr} \left(D \big[ v(u) q(t,u) \big] \right)   
+ \big\langle C^{-1} u, v(u) q(t,u)  \big\rangle \right) \, d\Pi(u) \\
&= \int \big \langle D\ell(u) , v(u) q(t,u) \big \rangle \, d\Pi(u) \\
&= \int \big \langle D \ell(u) , v(u) \big\rangle \, dQ_t(u),
\end{align}
where the second equality follows from \eqref{eq:density-evolution} and the third euality from IBP. For (b) we calculate
\begin{align}
    &\frac{d}{dt} \int \log q_t(u) q_t(u) \, d\Pi(u) \\
    &= \int \partial_t q(t,u) (1 + \log q_t(u) ) \, d\Pi(u) \\
    &= \int \left( - \text{Tr} \left(D \big[ v(u) q(t,u) \big] \right)   
+ \big\langle C^{-1} u, v(u) q(t,u)  \big\rangle  \right) \big( 1 + \log q_t(u) \big) \, d\Pi(u) \\
&= \int \big \langle D ( 1 + \log q_t(u) ), v(u) q(t,u) \big \rangle \, d\Pi(u) \\
&= \int \big \langle D (  \log q_t(u) ), v(u) \big \rangle \, dQ_t(u) 
\end{align}
where we again use \eqref{eq:density-evolution} and IBP. We put the calculation for (a) and (b) together and obtain 
\begin{align}
    \frac{d}{dt} L\big( Q(t) \big) = \int \big\langle D \ell(u) + D \log q_t(u) , v(u) \big \rangle \, d\Pi(u) \label{eq:WG-def}.
\end{align}
We evaluate the RHS of \eqref{eq:WG-def} for $t=0$ and see that the Wasserstein gradient at $Q$ is given as
\begin{align}
\nabla_W L[Q](u) = D \ell(u) + D \log q(u)
\end{align}
for $u \in H$ with $q = dQ/d\Pi$.  Note that this is precisely what we would expect from the finite-dimensional case.
\end{proof}
The next step is to identify a suitable stochastic process $\big( F(t) \big)$ with $F(t): \Omega \to H$ such that $\text{Law}[F(t)] = Q(t)$ where $Q(t)$ is the WGF at time $t$.
\begin{theorem}\label{thm:WGF}
    Let $\big( F(t) \big)_{t \in [0, T] }$ be the solution (which we assume exists, see \citet{hairer2007analysis} or \citet{da2014stochastic} for conditions) to 
    \begin{align}
        F(0) &\sim Q_0 \\
        d F(t) &= - \left(D \ell\big( F(t) \big) + C^{-1} F(t)  \right) dt + \sqrt{2} dW(t) \label{eq:SDE-ID-2}
    \end{align}
    for $t \in [0, T]$ where $Q_0 \in \mathcal{P}_2(H)$ is given, $C^{-1}$ is the inverse of the covariance operator $C$ and $\big(W(t)\big)$ a cylindrical Wiener process. Then $Q(t) := \text{Law} \big[ F(t) \big]$ follows the Wasserstein gradient flow for $L$ in \eqref{eq:WG-def2} and starts at $Q_0$. 
\end{theorem}

\begin{proof} Recall that the loss $L$ in \eqref{eq:WG-def2} for all $Q$ dominated by $\Pi$  can be written as 
\begin{align}
    L(Q) = \KL (Q ,\Pi_{\text{B}} ) + \text{ const.}
\end{align}
where $\Pi_{\text{B}}$ is the Bayesian posterior \citep{wild2022variational}. Recall further that by Bayes theorem
\begin{align}
   \frac{d \Pi_{\text{B}}}{d\Pi}(u) = \frac{p(y|u)}{p(y)} 
\end{align}
for all $u \in H$ where $p(y|u)$ is the likelihood function, $p(y)= \int p(y|u) \, d\Pi(u) $ the marginal likelihood and $\ell(u) := - \log p(y|u)$ as introduced in the main text. Note that $\Pi_{\text{B}}$ is log-concave in the sense of Definition 9.4.9 in \citet{ambrosio2005gradient} as long as $\ell(u)$ is convex \citep[Theorem 9.4.11]{ambrosio2005gradient}. We will assume $\ell$ to be convex from now on which is typically true for functional losses. 
Theorem 11.2.12 in \citet{ambrosio2005gradient} implies that the Wasserstein gradient flow $\big( Q(t) \big)$ for $L$ exists and further that the Radon-Nikodym derivative of $Q(t)$ with respect to $\Pi_{\text{B}}$ exists and that $\rho_t := dQ(t)/d \Pi_{\text{B}}$ satisfies 
\begin{align}
    \int_0^T \int_H \partial_t \varphi(t,u) - \langle D \log \rho_t(u), D \varphi(t,u) \rangle \, dQ_t(u) \, dt = 0 \label{eq:WGF-charact-eq}
\end{align}
for all test functions $\varphi: [0,T] \times H \to \bbR$. 
We want to rewrite this equation in terms of $q(t) :=  dQ(t)/d\Pi$. First note that $q(t)$ exists since $\Pi$ and $\Pi_{\text{B}}$ are equivalent in our case \citep[Section 1.3]{ghosal2017fundamentals}. We further have by the chain-rule for Radon-Nikodym densities
\begin{align}
    \log \rho_t(u) &= \log (dQ(t)/d \Pi_{\text{B}})(u) \\
    &= \log q_t(u) + \log  (d\Pi/d\Pi_{\text{B}}) (u) \\
    &= \log q_t(u) - \log ( d \Pi_{\text{B}}/ d\Pi)(u) \\
    &=\log q_t(u) - \log p(y|u) + \log p(y) \\
    &= \log q_t(u) + \ell(u) + \log p(y).
\end{align}
We plug this into \eqref{eq:WGF-charact-eq} and obtain the dynamics for $q(t)$ as
\begin{align}
    \int_0^T \int_H \partial_t \varphi(t,u) - \langle D \log q_t(u) + D \ell(u) , D \varphi(t,u) \rangle \, dQ_t(u) \, dt = 0. \label{eq:wgf-char-proof}
\end{align}
Since $D \log q_t + D \ell = \nabla_W L[Q(t)]$ (cf. Theorem \ref{thm:WG-KLD}) we conclude that $q(t)$ indeed satisfies \eqref{eq:wgf-characterisation} and therefore follows the WGF.

However note that every solution to \eqref{eq:wgf-char-proof} satisfies (differentiate w.r.t. to time)\footnote{The reverse may be easier: Multiply both sides of \eqref{eq:FPE-2} with a test function and then integrate with respect to time. This gives \eqref{eq:wgf-char-proof}.}
\begin{align}
    \frac{d}{dt} \int \psi(u) dQ_t(u) = -\int_H \langle D \log q_t(u) + D \ell(u) , D \psi(u) \rangle \, dQ_t(u) \label{eq:FPE-2}
\end{align}
for all $t$ and all test functions $\psi:H \to \bbR$. The RHS of \eqref{eq:FPE-2} is equal to
\begin{align}
    &-\int_H \langle D \log q_t(u) + D \ell(u) , D \psi(u) \rangle \, dQ_t(u) \\
    &= - \int \langle D \ell(u), D \psi(u) \rangle dQ_t(u) - \int_H \langle D \log q_t(u)  , D \psi(u) \rangle \, dQ_t(u) \\
    &= - \int \langle D \ell(u), D \psi(u) \rangle dQ_t(u) - \int_H \langle D  q_t(u)  , D \psi(u) \rangle \, d\Pi(u) \\
    &= - \int \langle D \ell(u), D \psi(u) \rangle dQ_t(u) + \int \text{Tr} \big[ D^2 \psi(u) \big] \, dQ_t(u) - \int \langle C^{-1} u, D \psi(u) \rangle d Q_t(u) \label{eq:ibp-result}
\end{align}
where the last equality follows from integration by parts for Gaussian measures on Hilbert spaces \citep[Lemma 11.1.9]{da2002second}. We combine \eqref{eq:FPE-2} and \eqref{eq:ibp-result} to obtain
\begin{align}
    \frac{d}{dt} \int \psi(u) dQ_t(u) = \int  \mathcal{A}_t \psi(u) d Q_t(u) \label{eq:FPE-3}
\end{align}
where $\mathcal{A}_t$ is the Kolmogorov operator defined as 
\begin{align}
    \mathcal{A}_t \psi(u) = \text{Tr} \big[ D^2 \psi(u) \big] + \big\langle - D \ell(u) - C^{-1}u, D\psi(u) \big\rangle.
\end{align}
We recognise \eqref{eq:FPE-3} as the Fokker-Planck equation (FPE) associated with the solution to \eqref{eq:SDE-ID-2} (see Chapter 14.2.2 of \citet{da2014stochastic} or \citet{bogachev2008parabolic,bogachev2010existence} for earlier references). Note that any solution $\big( Q(t) \big)$ to the FPE solves \eqref{eq:wgf-char-proof} by reversing the above argument. As a consequence $\big( Q(t) \big)$ satisfies the WGF if and only if it is a solution to the FPE. We can therefore simulate the SDE \eqref{eq:SDE-ID-2} in order to follow the WGF.
\end{proof}

\section{The moments of the posterior measure }\label{ap:sec:moments-posterior}
Our Bayesian model consist of:
\begin{itemize}
    \item A prior measure $\Pi \in \mathcal{P}_2(H)$
    \item A likelihood function $p: \bbR^N \times H \to [0, \infty)$, i.e. $p$ is $\mathcal{B}(\bbR^N) \otimes \mathcal{B}(H)$- $\mathcal{B}([0,\infty))$ measurable and there exists a measure $\mu \in \mathcal{P}(\bbR^N)$ such that 
\begin{align}
    \int p(y|f) \, d\mu(y) = 1 \label{eq:mkf-property}.
\end{align}
\end{itemize}
According to Bayes' theorem, the posterior $\Pi_{\text{B}}$ exists and has a Radon-Nikodym derivative with respect to the prior $\Pi$ given as
\begin{align}
    \frac{d\Pi_{\text{B}}}{d \Pi}(f) = \frac{p(y|f)}{p(y)},
\end{align}
where $p(y):= \int p(y|f) \, d\Pi(f)$. By definition of the likelihood function, we know that  $\gamma(A):= \int_A p(y) \, d\mu(y)$, $A \in \mathcal{B}(\bbR^N)$, is a probability measure. We can now state the theorem.

\begin{theorem}\label{thm:posterior-moment}
The Bayesian posterior satisfies $\Pi_{\text{B}} \in \mathcal{P}_2(H)$ for $\gamma$-almost every $y \in \bbR^N$.
\end{theorem}

\begin{proof}
Let $\big( \Omega, \mathcal{A}, \bbP)$ be an (sufficiently large) underlying probability space and $F: \Omega \to H$, $Y: \Omega \to \bbR^N$ two measurable mappings such that $F \#\bbP = \Pi, \, Y\#\bbP \sim \gamma$ and 
\begin{align}
    \bbP( Y \in A | F= f) = \int_A p(y|f) \, d\mu(y)
\end{align}
for all $A \in \mathcal{B}(\bbR^N)$, $f \in H$. By the above construction, the claim of Theorem \ref{thm:posterior-moment}  holds if and only if 
\begin{align}
    \bbE\big[\|F \|^2 | Y = y \big] < \infty
\end{align}
for $\gamma$-almost every $y \in \bbR^N$. However, we know by the tower-property of expected values, that 
\begin{align}
    \bbE\big[ \| F \|^2 \big] = \int \bbE\big[\|F \|^2 | Y = y \big] \, d \gamma(y).
\end{align}
Since $\bbE[||F||^2] < \infty$ by Fernique's theorem  \citep[Theorem 2.7]{da2014stochastic}, we can immediately conclude that $\bbE\big[\|F \|^2 | Y = y \big] < \infty$ for $\gamma$-almost every $y \in \bbR^N$.    
\end{proof}

\begin{remark}
The proof above is identical for finite-dimensional parameters. However, we felt that including an infinite-dimensional version may still provide some additional benefit.

Furthermore, the result can easily be generalised to arbitrary moments, in the sense that, for a fixed $l \in \bbN$, it holds that $\bbE \big[ \|F \|^l \big] < \infty$ 
implies $\bbE \big[ \| F \|^l \, | \,  Y = y\big] < \infty$ for $\gamma$-almost every $y \in \bbR^N$.

\end{remark}

\section{Gaussian random elements with values in the RKHS}\label{ap:sec:GRE_in_RKHS}

In this section we prove the existence of a GRE $F: \Omega \to H_k$ with covariance operator $C:H_k \to H_k$ defined in \eqref{eq:def-cov-op}. This allows us to effectively parameterise RKHS valued GREs via the kernel $k$ from the RKHS $H=H_k$.

\begin{lemma}\label{lemma:GRE-in-RKHS}
Let $k: \calX \times \calX \to \bbR$ be a positive, symmetric and continuous kernel on a compact metric space $\calX$ and $H=H_k$ the associated RKHS. Let further $\nu \in \mathcal{P}(\calX)$ be a Borel probability measure with full support and assume that $\int k(x,x) \, d\nu(x) < \infty$. Then there exists a Gaussian random element $F \sim \mathcal{N}(0,C)$ in $H$ with covariance operator $C: H \to H$ given as
\begin{align}
    Cf = \int k(\cdot, x') f(x') \, d\nu(x'). \label{eq:def-cov-op-ap}
\end{align}
\end{lemma}

\begin{proof}
Let
\begin{align}
    L^2(\nu, \bbR):= \left\{ f: \calX \to \bbR \, : \, \int \big( f(x) \big)^2 \, d\nu(x) < \infty \right\}
\end{align}
be the space of equivalence classes of $\nu$-almost everywhere identical square-integrable functions with inner product denoted by $\langle \cdot ,\cdot \rangle_2$. Define $T_k: L^2(\nu, \bbR) \to L^2(\nu, \bbR) $´as $T_k f = \int k(\cdot, x')  f(x') \, d\nu(x') $. Note that $T_k$ and $C$ have the same functional form but are defined on different spaces. 

Under the assumptions on $k$ we know that $T_k$ is self-adjoint, compact operator and therefore the spectral theorem guarantees the existence of an orthonormal basis $\{ b_n \}_{n=1}^\infty \subset L_2(\nu,\bbR)$ and eigenvalues $\{ \lambda_n \}_{n=1}^\infty \subset \bbR_+$ such that 
\begin{align}
    T_k f = \sum_{n=1}^\infty \lambda_n \langle f, b_n \rangle_2 \,  b_n
\end{align}
where the sum converges in $L^2( \nu, \bbR)$. We can now define $S: \mathcal{L}^2(\nu, \bbR) \to  H_k$ via
\begin{align}
    S f = \sum_{n=1}^\infty \sqrt{\lambda_n} \langle f, b_n \rangle_2 \,  b_n. \label{eq:S-iso-iso}
\end{align}
It is well-known that $S$ an isometric isomorphism between $L^2(\nu, \bbR)$ and $H_k$ \citep[Theorem 4.5.1]{SteChr2008}. In particular this means
\begin{align}
    \langle S f , S g \rangle = \langle f, g \rangle_2
\end{align}
for all $f,g \in L^2(\nu, \bbR)$. 

According to Theorem 1 in \citet{wild2022generalized}, there exists a GRE $G$ in $L^2(\nu, \bbR)$ with $G \sim \mathcal{N}(0,T_k)$. Our goal is now to show that $F:= S_k \circ G$ is a Gaussian random element in $H_k$ with covariance operator $C$. 

Lemma 5.6 in \citet{kukush2020gaussian} implies that $F$ is a GRE in $H_k$ with covariance operator $S T_k S^*$ where $S^*$ is the adjoint operator of $S$. It remains to show that $C=S T_k S^*$. 

First, note that $S^*: H_k \to \mathcal{L}^2(\nu,\bbR)$ is characterised by satisfying 
\begin{align}
    \langle S b_i , b_j \rangle = \langle b_i, S^* b_j \rangle_2 \label{eq:S_adjoint1}
\end{align}
for all $(i,j) \in \bbN^2$ since $\{ b_n \}_{n=1}^\infty$ is an orthonormal system in $H_k$. By construction, we know that $S b_j = \sqrt{\lambda_j} b_j$ and therefore
\begin{align}
    \langle S b_i , b_j \rangle = \frac{1}{\sqrt{\lambda_j}} \langle S b_i , S b_j \rangle = \frac{1}{\sqrt{\lambda_j}} \langle b_i, b_j \rangle_2 \label{eq:S_adjoint2}
\end{align}
where the last equality follows from the isometry property of $S$. Combining \eqref{eq:S_adjoint1} and \eqref{eq:S_adjoint2}  leads to 
\begin{align}
    \langle b_i, S^* b_j \rangle_2 =  \frac{1}{\sqrt{\lambda_j}} \langle b_i, b_j \rangle_2
\end{align}
for all $(i,j) \in \bbN^2$ and therefore $S^*b_j = \frac{1}{\sqrt{\lambda_j}} b_j$ for all $j \in \bbN$. It now follows immediately that 
\begin{align}
    (S T_k S^*) b_j &=  \frac{1}{\sqrt{\lambda_j}} S T_k b_j =  \lambda_j b_j.
\end{align}
On the other hand, we have
\begin{equation}
    C b_j = T_k b_j = \lambda_j b_j
\end{equation}
for all $j \in \bbN$ by the spectral theorem and therefore $C = S T_k S^*$ by density of $\{b_n\}_{n=1}^{\infty}$ in $H_k$ as claimed.
\end{proof}

We show that $F$ naturally induces a Gaussian process $G$ albeit with a new kernel $r$ which is a smoothed version of $k$.

\begin{lemma}\label{lemma:GRE-GP-analogy}
Let $F \sim \mathcal{N}(0,C)$ be the GRE in the RKHS $H=H_k$ with covariance operator $C$ defined in \eqref{eq:def-cov-op-ap} (cf. Lemma \ref{lemma:GRE-in-RKHS}). Define $G(x) := \langle F , k(x, \cdot) \rangle$ and $ G := \big\{ G(x) \, : \, x \in \calX \big\}$. Then $G$ is a Gaussian process with kernel $r$ given by
\begin{align}
    r(x,x') = \int k(x,z) k(z, x') \, d\nu(z)
\end{align}
for all $x,x' \in \calX$.
\end{lemma}

\begin{proof}
The process $G$ is a GP with kernel $r$ if and only if $G(X) \sim \mathcal{N}\big(0, r(X,X)\big)$ for all $X=(x_1, \hdots, x_N) \in \calX^N$. The later is by definition equivalent to 
\begin{align}
    \alpha^T G(X) \sim \mathcal{N}(0, \alpha^T r(X,X) \alpha ).
\end{align}
for all $\alpha \in \bbR^N$ which we will prove hereafter. Let now $\alpha \in \bbR^N$ and $X=(x_1, \hdots, x_N) \in \calX^N$ arbitrary. We then have
\begin{align}
     \alpha^T G(X) &= \sum_{n=1}^N \alpha_n \langle G, k(x_n,\cdot) \rangle \\
     &= \big\langle G, \underbrace{\sum_{n=1}^N \alpha_n k(x_n, \cdot)}_{=:h} \big\rangle \label{eq:GRE-in-RKHS}
\end{align}
where the first equality follows from the reproducing property and the second by bilinearity of the scalar product. By definition of GREs (cf. Appendix \ref{ap:sec:GREs}), we infer from \eqref{eq:GRE-in-RKHS}  that $\alpha^T G(X)$ is Gaussian random variable and it's variance is given as 
\begin{align}
    \mathbb{V}[\alpha^T G(X)] &= \langle C h , h \rangle = \sum_{n,n'=1}^N \alpha_n \alpha_{n'} \big \langle C k(x_n, \cdot) , k(x_{n'}, \cdot) \rangle.
\end{align}
By standard rules for Bochner integrals, we further obtain
\begin{align}
    \big\langle C k(x_n, \cdot) , k(x_{n'}, \cdot) \big\rangle &= \left\langle \int k(\cdot, z) k(z, x_n) \, d\nu (z), k(x_{n'}, \cdot) \right\rangle \\
    &= \int \big\langle k(\cdot, z) , k(x_{n'}, \cdot) \big\rangle k(z, x_n) \, d \nu(z) \\
    &= \int  k(z, x_{n'}) k(z, x_n) \, d\nu(z) \\
    &= r(x_n, x_{n'})
\end{align}
for all $n,n'=1,\hdots,N$ which proves the claim
\end{proof}

\begin{remark}
This kernel $r$ was already proposed in in Section 3.1 of \citet{filippi2016bayesian} in an effort to construct a Gaussian processes with sample paths that lie in the RKHS $H=H_k$. Our derivation gives a natural interpretation of the corresponding GP as being derived from a GP with sample paths in $L^2(\nu)$ which is equipped with additional smoothness by virtue of the isometric isomorphism $S$ in \eqref{eq:S-iso-iso}.
\end{remark}

\section[The influence of the measure on the input data space]{The influence of the measure $\nu$}\label{ap:sec:the-role-of-nu}

The Gaussian process $G$ constructed in Lemma \ref{lemma:GRE-GP-analogy} has kernel $r$ given as 
\begin{align}
    r(x,x') = \int k(x,s) k(s,x') \, d\nu(s) \label{eq:r-kernel2}
\end{align}
for all $x,x' \in \calX$. In principle any Borel measure $\nu \in \mathcal{P}(\calX)$\footnote{In fact $\nu$ is only required to be finite and does not need to be a probability measure.} which satisfies the conditions in Lemma \ref{lemma:GRE-in-RKHS} leads to a valid GRE in $H_k$.

One idea is therefore to choose a combination of $k$ and $\nu$ such that $r$ can be calculated analytically. This strategy was proposed in \citet{filippi2016bayesian} and can be very effective. However, it severely restricts the class of available kernels $k$ since few kernel exists that lead to a tractable expression in \eqref{eq:r-kernel2}.

Alternatively, we can choose a probability measure $\nu \in \mathcal{P}(\calX)$ from which which it is easy to generate samples $\widehat{X}_1, \hdots \widehat{X}_{N_S} \sim \nu$ where $N_S \in \bbN$ is the number of samples. This leads to the Monte Carlo estimator 
\begin{align}
    r(x,x') \approx \frac{1}{N_S} k(x,\widehat{X}) k( \widehat{X}, x') := \frac{1}{N_S} \sum_{n=1}^{N_S} k(x,\widehat{x}_n) k(\widehat{x}_n, x')
\end{align}
for $r$ where $\widehat{X}:= ( \widehat{X}_1, \hdots \widehat{X}_{N_S}) $.

We now want to illustrate that the kernel effectively becomes useless, if the input data distribution, i.e. the law of $X_1,\hdots,X_N$ deviates strongly from $\nu$. To this end, assume that we use a squared exponential kernel $k$ with
\begin{align}
    k(x,x') = \exp\big( -|x-x'|^2\big)
\end{align}
for $x,x' \in \bbR$. Assume further that the data is given as $X_1,\hdots,X_N \sim \mathcal{N}(0,1)$. We now choose $\nu= \mathcal{N}(10,1)$ and generate samples $\widehat{X}_1, \hdots, \widehat{X}_{N_S} \sim \nu$. Take now two points $x,x'$ which are in a high-density region of $\mathcal{N}(0,1)$, e.g. $x=0$ and $x'=0.5$, then
\begin{align}
    r(x,x') &= \bbE_{\xi} \Big[ \exp\big( - |\xi|^2 - |0.5- \xi|^2 \big) \Big] \approx 1.3 \cdot 10^{-43} 
\end{align}
where $\xi \sim \nu = \mathcal{N}(10,1)$. Put differently, a kernel matrix based on $r$ would be extremely uninformative. This is why we choose $\nu = \frac{1}{N}\sum_{n=1}^N \delta_{x_n}$ or an approximation based on sub-samples of $x_1,\hdots,x_N$ in all our experiments.

\section{Langevin SDE in ONB representation}
The time evolution of the SDE in $H$ is given as (cf. Theorem \ref{thm:WGF})
\begin{align}
    d F(t) = - \left( D \ell \big( F(t) \big) + C^{-1} F(t) \right) dt + \sqrt{2} d W(t)
\end{align}
for $t \in (0, T]$ and $F(0) = F_0$ with $F_0 \in H$ given. Define now $F^{m}(t):= \langle F(t) , e_m \rangle$ where $\{\lambda_m, e_m \}_{m=1}^\infty$ is the spectral decomposition of the covariance operator $C$ in \eqref{eq:def-cov-op}. Define 
$\phi: H_k \to \bbR$ with $\phi(f):= \langle f, e_m \rangle$ and by linearity it's Fréchet derivatives are given as $D\phi(f) = e_m$ and $D^2 \phi =0$ and consequently by Ito's rule \citep[Chapter 4.4]{da2014stochastic} we obtain
\begin{align}
    d F^{m}(t) &= d \phi\big( F(t)\big) \\
    &=- \langle D \ell \big( F(t) \big) + C^{-1} F(t)  , e_m \rangle dt + \sqrt{2} d \langle W(t), e_m \rangle \\
    &= - \langle D \ell\big( F(t) \big), e_m \rangle dt - \langle C^{-1} F(t) , e_m \rangle dt + \sqrt{2} d \langle W(t), e_m \rangle.
\end{align}
Recall, that our loss $\ell$ is of the form $\eqref{eq:loss-pointwise}$ and we therefore can simplify further as 
\begin{align}
d F^{m}(t) &= -\sum_{n=1}^N (\partial_2 c)\big( y_n, F_t(x_n)\big) \langle k(x_n,\cdot) , e_m \rangle dt - \langle F(t), C^{-1} e_m \rangle dt + \sqrt{2} d \langle W(t), e_m \rangle \\
&=     - \Big(\sum_{n=1}^N (\partial_2 c)(y_n, F_t(x_{n})) e_{m}(x_n) + \frac{ F^m(t)}{\lambda_m} \Big) dt + \sqrt{2} d B^m(t),
\end{align}
where we used the reproducing property and that $C^{-1}$ is self adjoint. Note that $B^m(t) := \langle W(t), e_m \rangle$ $(m=1,\hdots,M)$ are stochastically independent Brownian motions by standard properties of the cylindrical Wiener process \citep[Chapter 4.1.2]{da2014stochastic}.

\section{The Nyström method}\label{sec:ap:nystroem-method}

The Nyström method \citep{williams2001using} is designed to approximate eigenvalue-eigenfunction pairs $\{\lambda_n, b_n \}_{n=1}^\infty$ of the kernel integral operator $T_k : \mathcal{L}_2(\nu;\bbR) \to \mathcal{L}_2(\nu;\bbR)$ defined as 
\begin{align}
    T_k f := \int k(\cdot, x') f(x') \, d\nu(x')
\end{align}
which we introduced in Appendix \ref{ap:sec:GRE_in_RKHS}. For $X_1,\hdots, X_N $ independent samples from $\nu$, the Nyström method calculates the spectral decomposition of $\frac{1}{N}k(x_{1:N}, x_{1:N}):= \frac{1}{N} \big( k(x_n, x_{n'})\big)_{n,n'=1}^N \in \bbR^{N \times N}$ as 
\begin{align}
    \frac{1}{N}k(x_{1:N}, x_{1:N}) = V \widehat{\Lambda} V^T \label{eq:ap:spectra-decomp}
\end{align}
where $V=(v_1|\hdots|v_N)$ and $\widehat{\Lambda}:=\text{diag}(\widehat{\lambda}_1, \hdots, \widehat{\lambda}_N)$. Then 
$$
\widehat{b}_n(x) := \frac{1}{ \sqrt{N} \widehat{\lambda}_n} v_n^T k(X,x) := \frac{1}{ \sqrt{N} \widehat{\lambda}_n} \sum_{j=1}^N v_{nj} k(x_j, x)
$$
approximates $b_n(x)$ and $\widehat{\lambda}_m$ approximates $\lambda_m$. 

However, $C$ has the same eigenvalues as $T_k$ and the eigenfunctions are related \citep[Theorem 4.51]{SteChr2008} with
$
    e_n = \sqrt{\lambda}_n b_n.
$
We can therefore use 
\begin{align}
    \widehat{e}_n(x):= \sqrt{\widehat{\lambda}_n} \widehat{b}_n(x) = \frac{1}{ \sqrt{N\widehat{\lambda}_n} } v_n^T k(X,x)
\end{align}
as approximation for $e_n$. 

In many cases, using all available data $X_1, \hdots,X_N$ would be computationally prohibitive since the spectral decomposition in \eqref{eq:ap:spectra-decomp} costs $\mathcal{O}(N^3)$. However, for typical kernels, such as the squared exponential, the eigenvalues decay very rapidly and only the first few eigenvalues contribute meaningfully to explain the data. We can therefore chose a subsample $z_1,\hdots, z_M$ of $x_1,\hdots,x_N$ and apply the Nyström method for $z_{1:M}:=(z_1,\hdots,z_M)$ where $M \ll N$
which leads to $\mathcal{O}(M^3)$ costs. We describe a heuristic in Appendix \ref{ap:sec:implementation-details} for choosing $z_{1:M}$.

\section{Sufficiency of Nyström projections}\label{sec:ap:nystroem-sufficient}

\begin{lemma}
Let  $F\sim \mathcal{N}(0,C)$ be a GRE and $Y_{1:N} \sim \gamma$ be two random mappings such that (cf. Appendix \ref{ap:sec:moments-posterior} for a construction)
\begin{align}
    \bbP( Y_{1:N} \in A | F=f) = \int_A p(y_{1:N}|f) \, d\mu(y_{1:N}).
\end{align}
We assume that the negative log-likelihood is of the form
\begin{align}
    - \log p(y_{1:N}|f) = \sum_{n=1}^N c(y_n, f(x_n)) + \text{ const}.
\end{align}
We define $\psi: \bbR^N \to \bbR$ as $\psi\big( f(x_{1:N}) \big):= \exp\left( - \sum_{n=1}^N c(y_n, f(x_n)) -\text{const}\right) $.
Then it holds that
\begin{align}
    \bbP\big( F(x^*) \in B \, | \, y_{1:N}, \widehat{F}^{1:N}\big) = \bbP\big( F(x^*) \in B \, | \, \widehat{F}^{1:N} \big) \label{eq:sufficiency-statematnet}
\end{align}
for all $x^* \in \calX^{N_*}$ and Borel sets $B \subset \bbR^{N_*}$. Furthermore, if $\nu = \frac{1}{N} \sum_{n=1}^N \delta_{x_n}$, then we have 
\begin{align}
    \widehat{\Pi}_{\infty}^{1:N}
    =
    {\Pi}_{\infty}^{1:N}
     =
    \widehat{\Pi}_{\operatorname{B}}^{1:N}
     =
    {\Pi}_{\operatorname{B}}^{1:N}
    = \frac{1}{Z} %
    \int \exp(- \widehat{V}_\infty(u) ) \, du \label{eq:measure-equality}
\end{align}
where $Z:= \int \exp(-\widehat{V}_\infty(u)) \, du$ and $\widehat{V}_\infty:\bbR^N \to \bbR$ given as
\begin{IEEEeqnarray}{rCl}
    \widehat{V}_\infty(u) & := & \sum_{n=1}^N c\big(y_n, \sum_{m=1}^N\widehat{e}_m(x_n) u_m \big) + \frac{1}{2} u^T \widehat{\Lambda}^{-1} u. 
\end{IEEEeqnarray}
with $\widehat\Lambda$ as  defined in Appendix \ref{sec:ap:nystroem-method}.
\end{lemma}


\begin{proof}
Recall that $\widehat{F}^{1:N} := \big( \langle F, \widehat{e}_1 \rangle, \hdots, \langle F, \widehat{e}_N \rangle \big)^T$ with 
\begin{align}
    \widehat{e}_n(x) = \widehat{v}_n^T k_X(x) = \sum_{i=1}^N \widehat{v}_{ni} k(x , x_i), 
\end{align}
where $\widehat{v}_n  = \frac{1}{\sqrt{N \widehat{\lambda}_n} } v_n \in \bbR^N$ (cf. Appendix \ref{sec:ap:nystroem-method}). We now want to show that the orthogonal projection $\widehat{\Pr} [f] = \sum_{n=1}^N \langle f, \widehat{e}_n \rangle \widehat{e}_n$ is exact for all data points $x_1,\hdots,x_N$ which were used in the Nyström method, i.e $\widehat{\Pr}[f](x_n) = f(x_n)$ for all $n=1,\hdots,N$. To this end, we define $k_X(\cdot) = \big( k(x_1, \cdot), \hdots, k(x_N,\cdot) \big)^T$ and calculate
\begin{align}
    \widehat{\Pr}[f](x_n) &= \sum_{i=1}^N \langle f, \widehat{e}_i \rangle \widehat{e}_i(x_n) \\
    &= \sum_{i=1}^N  \widehat{v}_i^T f(x_{1:N}) \widehat{v}_i^T k_X(x_n) \\
    &= f(x_{1:N})^T \Big( \sum_{i=1}^N \widehat{v}_i \widehat{v}_i^T \Big) k_X(x_n) \\
    &= f(x_{1:N})^T \Big( \frac{1}{N} \sum_{i=1}^N \frac{1}{\lambda_i} v_i v_i^T \Big) k_X(x_n) \\
    &= f(x_{1:N})^T k(x_{1:N},x_{1:N})^{-1} k_X(x_n) \\
    &= f(x_n),
\end{align}
for all $n=1,\hdots,N$. This can equivalently be expressed as
\begin{align}
    f(x_{1:N}) = \widehat{\Pr}[f](x_{1:N}) = \sum_{i=1}^N \underbrace{\langle f, \widehat{e}_i\rangle }_{=:\widehat{f}^i}  \widehat{e}_i(x_{1:N})
\end{align}
for all $f \in H_k$. Consequently, we have 
\begin{align}
    p(y_{1:N}|f) = \psi\Big( \sum_{i=1}^N \widehat{f}^i \widehat{e}_i(x_{1:N}) \Big) \label{eq:measurable-fct-of-hats}
\end{align}
and conclude that the conditional density of $Y_{1:N}$ given $F=f$ is a measurable function of $\widehat{f}^{1:N}$. We now write $p(f(x^*)| \widehat{f}^{1:N}, y_{1:N})$ for the pdf of $F(x^*)$ given ($\widehat{F}^{1:N}, Y_{1:N}) =( \widehat{f}^{1:N}, y_{1:N})$ and similarly for other marginal and conditional distributions involved and obtain
\begin{align}
    p\big(f(x^*)| \widehat{f}^{1:N}, y\big) &= \frac{p\big(f(x^*), \widehat{f}^{1:N}, y_{1:N}\big)}{ p\big(\widehat{f}^{1:N}, y\big)} \\
    &= \frac{p\big(y_{1:N} | f(x^*), \widehat{f}^{1:N}\big)  p\big(f(x^*)| \widehat{f}^{1:N}\big) p\big(\widehat{f}^{1:N} \big)
    }{p\big( y_{1:N} | \widehat{f}^{1:N} \big) p\big( \widehat{f}^{1:N} \big)} .
\end{align}
Here, we notice that $p\big(y_{1:N} | f(x^*), \widehat{f}^{1:N}\big) = p(y_{1:N} |  \widehat{f}^{1:N}) $ because $p(y_{1:N}| f)$ is a measurable function of $\widehat{f}^{1:N}$ and consequently we obtain
\begin{align}
    p\big(f(x^*)| \widehat{f}^{1:N}, y_{1:N} \big) = p( f(x^*)| \widehat{f}^{1:N})
\end{align}
which shows \eqref{eq:sufficiency-statematnet}. It remains to show \eqref{eq:measure-equality}. However,
analogous to what we do in in Appendix \ref{ap:sec:asymptotics-pls}, it is easy to show that \eqref{eq:approx-langevin-components} is a Langevin diffusion with potential 
\begin{IEEEeqnarray}{rCl}
    \widehat{V}_\infty(u) & := & \sum_{n=1}^N c\big(y_n, \sum_{m=1}^N\widehat{e}_m(x_n) u_m \big) + \frac{1}{2} u^T \widehat{\Lambda}^{-1} u.
\end{IEEEeqnarray}
Consequently, we immediately infer that 
%
%
$\widehat{\Pi}_{\infty}^{1:N} \propto \exp(- \widehat{V}_\infty)$ which proves the second equality in \eqref{eq:measure-equality}. Furthermore, it follow from \eqref{eq:measurable-fct-of-hats} that
\begin{align}
   - \log p(y| \widehat{f}^{1:N}) &= - \log \psi\big( \sum_{m=1}^N \widehat{f}^m \widehat{e}_m(x_{1:N}) \big) \\
   &=\sum_{n=1}^N c\big(y_n, \sum_{m=1}^N\widehat{e}_m(x_{n}) \widehat{f}^m  \big)
\end{align}
and for the prior we know that $\widehat{F}^{1:N} \sim \mathcal{N} \Big(0 ,  \big(\langle C \widehat{e}_m , \widehat{e}_{m'} \rangle\big)_{m,m'=1}^N \Big)$ by definition of a GRE. By definition, $v_m$ is the eigenvector of $\frac{1}{N} k(x_{1:N}, x_{1:N})$ with eigenvalue $\widehat{\lambda}_m$ and further $r(x_{1:N}, x_{1:N}) = \frac{1}{N} k(x_{1:N},x_{1:N})^2$ due to $\nu= \sum_{n=1}^N \delta_{x_n}$ 
which leads to
\begin{align}
    \langle C \widehat{e}_m , \widehat{e}_{m'} \rangle &= \widehat{v}_m^T r(x_{1:N}, x_{1:N}) \widehat{v}_{m'} \\
    &=  \frac{1}{N}  \widehat{v}_{m}^T \big( k(x_{1:N}, x_{1:N}) \big)^2 \widehat{v}_{m'} \\
    &= \frac{1}{\sqrt{\widehat{\lambda}_m \widehat{\lambda}_{m'} }} v^T_m \big( \frac{1}{N} k(x_{1:N}, x_{1:N}) \big)^2 v_{m'} \\
    &= \widehat\lambda_m \delta_{m,m'},
\end{align}
where $\delta_{m,m'}$ denotes the Kronecker delta. As a result, we have by Bayes theorem
%
\begin{align}
    p_{\widehat{F}^{1:N}|Y_{1:N}}(u|y_{1:N}) &\propto \exp \Big( \log p(y_{1:N}| \widehat{F}^{1:N}=u) + \log p_{\widehat{F}^{1:N}}(u) \Big) \\
    &= \exp \Big( - \sum_{n=1}^N c\big(y_n, \sum_{m=1}^N\widehat{e}_m(x_n) u_m \big) - \frac{1}{2} u^T \widehat\Lambda^{-1} u \Big)
\end{align}
which proves the claim.
\end{proof}

\section{Matheron's Rule for Gaussian Random Elements}\label{sec:ap:matherons-rule}

Matheron's rule \citep{journel1976mining,wilson2020efficiently} is a simple trick to sample a conditional Gaussian. Let $\big( U, V\big) \sim \mathcal{N}(\mu, \Sigma)$ with 
\begin{align}
&\mu =
\begin{pmatrix}
 \mu_U \\
 \mu_V
\end{pmatrix} 
&&\Sigma=
\begin{pmatrix}
\Sigma_{UU} & \Sigma_{UV} \\
\Sigma_{VU} & \Sigma_{VV}
\end{pmatrix}
\end{align}
Then $U| V=v \sim \widetilde{U} + \Sigma_{UV} \Sigma_{VV}^{-1} \big( v-  \widetilde{V} \big) $ where $\big( \widetilde{V}, \widetilde{U} \big) \sim \mathcal{N}(\mu, \Sigma)$ is independent from $(U,V)$. In other words: We can transform a sample from the joint distribution of $(U,V)$ into a sample of the conditional distribution $U|V=v$. 

We want to use Matheron's rule to generate samples from $ F(X^*) | {F}^{1:M}=u$ where $X^* \in \calX^{N_*}$ is a set of input locations. Since $F \sim \mathcal{N}(0,C)$ is a GRE, we know that $\big(F(X^*), {F}^{1:M} \big)$ is jointly Gaussian with mean zero and covariance matrix 
\begin{align}
    R := \begin{pmatrix}
  \mathbb{C}\big[ F(X^*), F(X^*) \big] &  \mathbb{C}\big[ F(X^*), F^{1:M}\big] \\
\mathbb{C}\big[  F^{1:M }, F(X^*)\big]  &  \mathbb{C}\big[ F^{1:M}, F^{1:M}\big] 
\end{pmatrix}.
\end{align}
We calculate the relevant covariance matrices as 
\begin{align}
    &\mathbb{C}\big[ F(X^*), F(X^*) \big]  = r(X^*,X^*) 
    &&\mathbb{C}\big[ F^{1:M}, F^{1:M}\big] = \Lambda
\end{align}
where $\Lambda = \text{diag}(\lambda_1, \hdots, \lambda_M)$ and $r$ is the kernel defined in \eqref{eq:def-r-kernel}. Further, we have 
\begin{align}
   \mathbb{C}\big[  F^{m }, F(x_n^*)\big]  = \langle C e_{m}, k_{x_n^*}(\cdot) \rangle 
   = \lambda_m \langle e_{m}, k_{x_n^*}(\cdot) \rangle 
   = \lambda_m e_{m}(x_n^*)
\end{align}
for all $m=1,\hdots,M$, $n=1,\hdots,N^*$ and consequently 
\begin{align}
    \mathbb{C}\big[  F^{1:M }, F(X^*)\big]  = \Lambda e^{1:M}(X^*) \in \bbR^{M\times N_*},
\end{align}
where $(e^{1:M}(X^*))_{m,n} := e_{m}(x_n^*) $. Matheron's rule therefore takes the form 
\begin{align}
    F(X^*)|F^{1:M}=v \sim U + e^{1:M}(X^*)^T (v - V)
\end{align}
for all $u \in \bbR^M$ with $(U,V) \sim \mathcal{N}(0,R)$. Naturally, we do not have access to $e^{1:M}(X^*)$, $\Lambda$ and $R$. We therefore use the corresponding approximations
\begin{align}
\widehat{R}:=
\begin{pmatrix}
    \widehat{r}(X^*,X^*) &  \widehat{e}^{1:M}(X^*)^T \widehat{\Lambda}\\
    \widehat{\Lambda} \widehat{e}^{1:M}(X^*) & \widehat{\Lambda}
\end{pmatrix},
\end{align}
where $\widehat{\Lambda}$ and $(\widehat{e}^{1:M}(X^*))_{m,n} := \widehat{e}_m(x_n^*)$ ($m=1,\hdots,M$, $n=1,\hdots,N_*$)
are obtained from the Nyström approximation based on the samples $Z:=z_{1:M}$ (cf. Appendix \ref{sec:ap:nystroem-method}). The kernel matrix is approximated via
\begin{align}
    \widehat{r}(X^*,X^*) := \frac{1}{N_*+M} k(X^*,\widehat{X}) k(\widehat{X}, X^*)
\end{align}
with $\widehat{X}:=(X^*,Z) \in \calX^{N_*+M}$. Notice that inclusion of $X^*$ in $\widehat{X}$ leads to $\widehat{r}(X^*,X^*)$ having full rank $N^*$.

\section{Asymptotic Analysis of Projected Langevin Sampling}\label{ap:sec:asymptotics-pls}

\begin{lemma}
Let $F \sim \mathcal{N}(0,C)$ be a GRE and assume that $-\log p(y|f) = \ell_N( f(x_{1:N}))$ for a function $\ell_N: \bbR^N \to \bbR$. Let further $\left(\widetilde{F}^{1:M}(t)\right)_{t\geq 0}$ be the solution to the SDE whose components are given as
\begin{IEEEeqnarray}{rCl}
    d \widetilde{F}^m(t) = - \left(\sum_{n=1}^N (\partial_n \ell_N )\Big( {\Pr} [F(t)](x_{1:N}) \Big) {e}_{m}(x_n) + \frac{\widetilde{F}^m(t)}{\lambda_m} \right) dt + \sqrt{2} d B^m(t), \label{eq:approx-langevin-components:ap}
\end{IEEEeqnarray}
where $\Pr[F(t)] = \sum_{m=1}^M \widetilde{F}^m(t) e_m$. Notice, that for $\ell_N\big( f(x_{1:N} ) \big) = \sum_{n=1}^N c\big(y_n , f(x_n) \big)$ we recover \eqref{eq:approx-langevin-components}. We want to show that 
\begin{align}
    \widetilde{F}^{1:M}(t) \overset{\mathcal{D}}{\longrightarrow} \Pi^{1:M}_\infty
\end{align}
for $t \to \infty$ where $\Pi^{1:M}_\infty$ has the potential
\begin{IEEEeqnarray}{rCl}
    V_\infty(u) & =  \ell_N\big( \mu_u(x_{1:N}) \big) + \frac{1}{2} u^T {\Lambda_M}^{-1} u + \text{ const.}
    \label{eq:potential-asy-distr:ap}
\end{IEEEeqnarray}
for $u \in \bbR^M$ where $\mu_u(x_{1:N})=u^T e^{1:M}(x_{1:N}) $ and $\Lambda_M = \text{diag}(\lambda_1,\hdots,\lambda_M)$.
\end{lemma}

\begin{proof}
Let $V_\infty$ be the potential above and note that, by the chain-rule, we have
\begin{align}
    \partial_m V_\infty(u) = \sum_{n=1}^N \partial_n \ell_N\big( \mu_u(x_{1:N}) \big) e_m(x_n) + \frac{u_m}{{\lambda}_m}
\end{align}
where $\partial_m$ denotes the derivative with respect to $m$-th coordinate. Let further $\nabla V_\infty= ( \partial_1 V_\infty, \hdots, \partial_M V_\infty)^T$ be the gradient of $V_\infty$. We now immediately recognise \eqref{eq:approx-langevin-components:ap} as Langevin diffusion for the potential $V$, since it holds that
\begin{align}
    d \widetilde{F}^m(t) = - \partial_m V\big( \widetilde{F}^{1:M}(t)\big) + \sqrt{2} dB^m(t). \label{eq:langevin-diff:ap}
\end{align}
It is well established \citep{roberts1996exponential} that---under mild assumption on $V$---the limiting distribution of  \eqref{eq:langevin-diff:ap} is given as $\Pi^{1:M}_\infty$.
\end{proof}

\section{Optimal Variational Approximation}\label{sec:ap:optimal-var-approx}
We prove a slightly more general result. Indeed, we can derive the optimal variational distribution for arbitrary inducing features $U$ with $U:=  \big(\langle F , h_m \rangle \big)_{m=1}^M \in \bbR^M$ for set of functions $\{h_1,\hdots,h_M\} \subset H_k $. Theorem \ref{thm:characterisation-optimal-approximation} follows for the choice $h_m = e_m$, $m=1,\hdots,M$. Define now for  $\tau \in \mathcal{P}_2(\bbR^m)$ the measure
\begin{align}
    Q_\tau(A) := \int \bbP \big( F \in A \, | \, U=u \big) \, d\tau(u),  \label{eq:definition-calQ}
\end{align}
for $A \in \mathcal{B}(H_k)$ and the set $\calQ_M := \{ Q_\tau \, : \, \tau \in \mathcal{P}_2(\bbR^M) \} \subset \mathcal{P}_2(H_k)$. We can find a closed form expression for the best posterior approximation in $\calQ_M$.

\subsection{Proof of \Cref{thm:characterisation-optimal-approximation}}

Below, we derive \Cref{thm:optimal-measure-calQ} below without relying on the decomposition into eigenvalues and eigenvectors, since doing so will allow us to more easily draw upon the result in subsequent analysis. 
The required steps to translate the results of \Cref{thm:optimal-measure-calQ} into the more intuitively appealing form stated in \Cref{thm:characterisation-optimal-approximation} are trivial.

\begin{theorem}\label{thm:optimal-measure-calQ}
Define $\Pi^*_M$ as 
\begin{align}
    \Pi^*_M := \argmin_{Q \in \calQ_M} \KL(Q , \Pi_{\text{B}})
\end{align}
where $\Pi_{\text{B}}$ is the Bayesian posterior. Then $\Pi^*_M$ satisfies \eqref{eq:definition-calQ} with $\tau^* \propto \exp(- V^*(u))$ with
\begin{align}
    V^*(u) :=\bbE_{\xi} \Big[ \ell_N \big( \mu_u(x_{1:N}) + \sqrt{ \Sigma(x_{1:N})} \xi \big) \Big]  + \frac{1}{2} \Big( \langle C h, h \rangle \Big)^{-1}
\end{align}
for all $u \in \bbR^M$. Here, $\xi \sim \mathcal{N}(0,I_M)$, and we define 
\begin{align}
\mu_u(x_{1:N}) &:= \mathbb{C} \big[ F(x_{1:N}) , U ] \mathbb{C}[U,U]^{-1}u \\
\Sigma(x_{1:N})&:= \mathbb{C} \big[ F(x_{1:N}) , F(x_{1:N})] - \mathbb{C} \big[ F(x_{1:N}) , U ] \mathbb{C}[U,U]^{-1} \mathbb{C} \big[ U , F(x_{1:N}) ] \\
\langle C h, h \rangle &:= \big( \langle C h_m, h_{m'} \rangle \big)_{m,m'=1}^M
\end{align}
for all $u \in \bbR^M$. Notice that by standard properties of the GRE these covariance terms can be expressed in terms of the covariance operator $C$.
\end{theorem}
\begin{proof}
Let $Q \in \calQ_M$ be of the form \eqref{eq:definition-calQ} with probability measure $\tau$. By Theorem 4 in \citet{wild2022variational}, we know that 
\begin{align}
       \KL(Q, \Pi_{\text{B}}) =  \bbE_Q\Big[\ell_N(f(x_{1:N})) \Big] + \KL(\tau, \Pi_U) + \log p(y)
\end{align}
where $\Pi_U$ is the prior law of $U$ given as $\mathcal{N}(0, \langle C h ,h \rangle)$. Furthermore, we know that $F(X)|U$ is Gaussian for every $Q \in \calQ_M$ by definition $\calQ_M$. We therefore know  by standard properties of Gaussians that 
\begin{align}
    F(x_{1:N}) | U=u \sim \mu_u(x_{1:N}) + \sqrt{\Sigma(x_{1:N})} \xi
\end{align}
for fixed $u \in \bbR^M$ and a $\xi \sim \mathcal{N}(0, I_M)$. We can therefore condition on $U=u$ and use the tower property of expected values to obtain 
\begin{align}
    \KL(Q, \Pi_{\text{B}}) = \int_{\bbR^M} \bbE_{\xi}\Big[\ell_N\big( \mu_u(x_{1:N}) + \sqrt{\Sigma(x_{1:N})} \xi\big) \Big] \, d\tau(u) + \KL(\tau, \Pi_U) + \log p(y).
\end{align}
We  now define $L(\tau) := \int \phi(u) \, d\tau + \KL(\tau, \Pi_U)$ with $\phi(u) :=\bbE_{\xi}\Big[\ell_N\big( \mu_u(x_{1:N}) + \sqrt{\Sigma(x_{1:N})} \xi\big) \Big] $, $u \in \bbR^M$ and $\tau \in \mathcal{P}_2(\bbR^M)$. From the above calculations, it then immediately follows that
\begin{align}
    \min_{Q \in \calQ} \KL(Q, \Pi_{\text{B}}) = \min_{\tau} L(\tau) + \log p(y).
\end{align}
However, the global minimiser of $L$ is well-known and given as $\tau^* \propto \exp(- V^*(u))$ \citep[Theorem 1]{knoblauch2019generalized} with potential  
\begin{align}
    V^*(u) = \phi(u) - \log \frac{d \Pi_U}{du}(u) = \phi(u) + \frac{1}{2} \Big(\langle C h, h \rangle\Big)^{-1}.
\end{align}
This proves our claim.
\end{proof}

\section{Optimality of Projected Langevin Sampling}\label{sec:ap:optimality-pls}

Let  $\{ \lambda_n , e_n \}_{n=1}^\infty$  be the eigenvalue-eigenfunction pairs of the covariance operator $C$ with $\lambda_1>\lambda_2 > \hdots $ and define $F^{m}:= \langle F, e_m \rangle$, $m=1,\hdots,M$ and $F^{1:M}:= \big( F^{1}, \hdots, F^{M}\big)^T$. We simulate according to the PLS algorithm and obtain $\Pi^{1:M}_{\infty}$ with potential $V_\infty$ as derived in Appendix \ref{ap:sec:asymptotics-pls}. The PLS approximation to the posterior is defined as
\begin{align}
    \Pi_{\infty}(A) := \int \bbP( F \in A \mid F^{1:M}=u) \, d \Pi^{1:M}_{\infty} (u)
\end{align}
for $A \in \mathcal{B}(H_k)$. By construction, clearly $\Pi_{\infty}\in \calQ_M$ and so it is natural to compare $\Pi_{\infty}$ to the optimal measure $\Pi^*_M$ of Theorem \ref{thm:optimal-measure-calQ}.


\subsection{Proof of \Cref{thm:PLS-optimal-close}}

\begin{proof}
Let $\Pi$ the Gaussian prior measure. Define the map $\Phi : H_k \to \bbR^M$ as $\Phi(f):= \langle f, e^{1:M} \rangle := \big( \langle f, e_m  \rangle \big)_{m=1}^M$. The by construction we have
\begin{align}
    &\Phi \# \Pi_{\infty}= \Pi^{1:M}_{\infty} &&\text{ and }
    &&&\Phi \# \Pi^*_M = \tau^*.
\end{align}
Furthermore $\Pi_{\infty}\in \calQ$ means that $\Pi_{\infty}$ is dominated by the prior measure $\Pi$ and of the form \citep{matthews2016sparse,wild2023rigorous}
\begin{align}
    \frac{d \Pi_{\infty}}{d \Pi}(f) = \frac{d (\Phi \# \widehat{\Pi})}{d (\Phi \# \Pi)}\big( \phi(f) \big) = \frac{d \Pi^{1:M}_{\infty} }{d(\mathcal{N}(0, \Lambda_M))}\big( \phi(f) \big)
\end{align}
and similarly for $\Pi^*_M$
\begin{align}
    \frac{d \Pi^*_M }{d \Pi}(f) = \frac{d (\Phi \# \Pi^*_M)}{d (\Phi \# \Pi)}\big( \phi(f) \big) = \frac{d \tau^*}{d(\mathcal{N}(0, \Lambda_M))}\big( \phi(f) \big)
\end{align}
for all $f \in H_k$. Here $\Lambda_M := \text{diag}\big( \lambda_1, \hdots, \lambda_M)$, Consequently, we have by standard rules for Radon-Nikodym derivatives that 
\begin{align}
    \frac{d {\Pi}_{\infty}}{{ d \Pi^*_M}}(f) = \frac{d \Pi^{1:M}_{\infty}}{d \tau^*} ( \phi(f) )
\end{align}
for all $f \in H_k$. We can now calculate the KL-divergence as 
\begin{align}
    \KL(\Pi_{\infty}, \Pi^*_M ) &=  \int_H \log \Big( \frac{d \Pi_{\infty}}{d\Pi^*_M}\Big) (f) \, d\Pi_{\infty}(f) \\
    &= \int_H \log \Big( \frac{d \Pi^{1:M}_{\infty}}{ d \tau^*}\Big)( \phi(f) )  \, d\Pi_{\infty}(f) \\
    &= \int_{\bbR^M} \log \frac{d \Pi^{1:M}_{\infty}}{d \tau^*}(u) \, d \tau_{\infty}(u) \\
    &= \KL( \Pi^{1:M}_{\infty}, \tau^*)
\end{align}
where we applied the change of measure formula. It remains to find an upper bound for the KL-divergence between the two probability measures $\Pi^{1:M}_{\infty}, \tau^* \in \mathcal{P}_2(\bbR^M)$.

By Theorem \ref{thm:optimal-measure-calQ}, we know that $\tau^*$ has potential given as 
\begin{align}
   V^*(u) =  \bbE_{\xi} \Big[ \ell_N \big( \mu_u(x_{1:N}) + \sqrt{ \Sigma(x_{1:N})} \xi \big) \Big]  + \frac{1}{2} \Big( \langle C h, h \rangle \Big)^{-1}
\end{align}
where $\xi \sim \mathcal{N}(0,I_M)$. For $h_m = e_m$ the mean vector and covariance matrix can be calculated as 
\begin{align}
\mu_u(x_{1:N}) &:= \mathbb{C} \big[ F(x_{1:N}) , U ] \mathbb{C}[U,U]^{-1}u          \\&= \big\langle C k_{x_{1:N}}(\cdot) , e^{1:M}\rangle \big( \langle C e^{1:M},e^{1:M} \rangle \big)^{-1} u \\
&= e^{1:M}(x_{1:N})^T \Lambda_M (\Lambda_M)^{-1} u \\
& = e^{1:M}(x_{1:N})^T u  \label{eq:mu_u-is-ortho}
\end{align}
and further
\begin{align}
\Sigma(x_{1:N})&:= \mathbb{C} \big[ F(x_{1:N}) , F(x_{1:N})] - \mathbb{C} \big[ F(x_{1:N}) , U ] \mathbb{C}[U,U]^{-1} \mathbb{C} \big[ U , F(x_{1:N}) ] \\
&= r(x_{1:N},x_{1:N}) - e^{1:M}(x_{1:N})^T \Lambda_M e^{1:M}(x_{1:N})
\end{align}
where $e^{1:M}(x_{1:N}) := \big( e^m(x_n) \big)_{m,n=1}^N \in \bbR^{M \times N}$.

Further, the potential of $\Pi^{1:M}_{\infty}$ is given as (cf. Appendix \ref{ap:sec:asymptotics-pls}) 
\begin{align}
    V_\infty(u) &=\ell_N( \mu_u(x_{1:N}) ) + \frac{1}{2} u^T \Lambda_M^{-1} u
\end{align}
 where the second equality follows from the definition of $\ell_N$ and the calculations for $\mu_u(X)$. Due to the convexity of $\ell_N$, we obtain from Jensen's inequality 
 \begin{align}
     V^*(u) &=  \bbE_{\xi} \Big[ \ell_N \big( \mu_u(x_{1:N}) + \sqrt{ \Sigma(x_{1:N})} \xi \big) \Big]  + \frac{1}{2} u^T\Lambda_M^{-1} u \\
     &\ge \ell_N \Big( \bbE_{\xi}\big[ \mu_u(x_{1:N}) + \sqrt{ \Sigma(x_{1:N})} \xi \big] \Big)   + \frac{1}{2} u^T\Lambda_M^{-1} u \\
     &=V_\infty(u).
 \end{align}
By Lemma 3 in \citet{dalalyan2017theoretical} this implies that 
\begin{align}
    \KL( \Pi^{1:M}_{\infty} , \tau^*) \le \frac{1}{2} \bbE_{U} \Big[ | V^*(U) - V_\infty(U) |^2 \Big]
\end{align}
where $U \sim \Pi^{1:M}_{\infty}$. We now calculate 
\begin{align}
    \bbE_{U} \Big[ | V^*(U) - V_\infty(U) |^2 \Big] &= \bbE_{U} \Big[ | \bbE_{\xi} \big[ \ell_N( \mu_U(x_{1:N})) -\ell_N ( \mu_U(x_{1:N}) + \sqrt{\Sigma} \xi )  \big]  |^2 \Big] \\
    &\le \bbE_{U} \bbE_\xi \Big[ |  \ell_N( \mu_U(x_{1:N})) -\ell_N ( \mu_U(x_{1:N}) + \sqrt{\Sigma} \xi ) |^2  \Big] \\
    &\le \kappa^2 \bbE_{U} \bbE_\xi \Big[ || \sqrt{\Sigma}(x_{1:N}) \xi ) ||^2   \Big] \\
    &=\kappa^2 \bbE \big[ || \sqrt{\Sigma}(x_{1:N}) \xi ) ||^2   \big] \\
    &= \kappa^2 \text{tr}[\Sigma(x_{1:N})],
\end{align}
where the first inequality is due to Jensen's inequality and the second inequality uses the Lipschitz continuity of $\ell_N$.

We combine the results to obtain
\begin{align}
    \KL\big( \Pi_{\infty}, \Pi^*_M \big) &= \KL( \Pi^{1:M}_{\infty}, \tau^*) \le \frac{1}{2} \kappa^2 \text{tr}[ \Sigma(x_{1:N}) ].
\end{align}
Furthermore, the trace can be simplified to 
\begin{align}
    \text{tr}[ \Sigma(x_{1:N})] &= \sum_{n=1}^N r(x_n, x_n) - \sum_{n=1}^M \lambda_m \text{tr}[ e_m(x_{1:N})^T e_m(x_{1:N}) ] \\
    &= \sum_{n=1}^N r(x_n, x_n) - \sum_{m=1}^M \lambda_m \sum_{n=1}^N (e_m(x_n))^2. \label{eq:trace}
\end{align}
Furthermore note that 
\begin{align}
    r(x,x') &= \langle C k_x(\cdot), k_{x'}(\cdot) \rangle \\
    &= \sum_{m=1}^\infty \lambda_m \langle k_x , e_m \rangle \langle k_{x'}, e_m \rangle \\
    &= \sum_{m=1}^\infty \lambda_m e_m(x) e_m(x')
\end{align}
for all $x,x' \in \calX$. We plug this expression back into \eqref{eq:trace} and obtain 
\begin{align}
    \text{tr}[\Sigma(x_{1:N})] = \sum_{m=M+1}^\infty \lambda_m \sum_{n=1}^N ( e_m(x_n))^2. 
\end{align}
Furthermore we know that for $X_n \sim \nu$ it holds that 
\begin{align}
    \bbE \big[ (e_m(X_n))^2 \big] &= \int (e_m(s))^2 \, d\nu(s)  \\
    &= \lambda_m,
\end{align}
due to the isometric isomorphism $S$ introduced in Appendix \ref{ap:sec:GRE_in_RKHS}. See also \citet[Theorem 4.51]{SteChr2008}. We combine all calculations and obtain for fixed $x_1,\hdots,x_N$ the inequality
\begin{align}
    \KL\big( \Pi_{\infty}, \Pi^*_M \big) \le \frac{\kappa^2}{2} \sum_{m=M+1}^\infty \lambda_m \sum_{n=1}^N ( e_m(x_n))^2
\end{align}
and further for $X_1,\hdots,X_N \sim \nu$ 
\begin{align}
    \bbE_{X_{1:N}} \Big[  \KL\big( \Pi_{\infty}, \Pi^*_M \big) \Big] \le \frac{N \kappa^2}{2} \sum_{m=M+1}^\infty \lambda_m^2
\end{align}
which proves the claim.
\end{proof}

The assumptions that $\ell_N$ is Lipschitz continuous and convex is indeed often satisfied in practice. For example, the binary classification loss with logistic link function introduced in Section \ref{sec:loss} is both convex and Lipschitz continuous \citep{SteChr2008}. Unfortunately, the squared loss does not satisfy Lipschitz continuity (even though it is convex), but in this case the PLS measure $\Pi_{\infty}$ actually coincides with ${\Pi}^{*}_{M}$. 

\subsection{Proof of Lemma \ref{lemma:decay-eigenvalues}}
\label{sec:proof:decay-eigenvalues}

\begin{proof}
    First, consider the Gaussian kernel. 
    Following the findings of \citet{ritter1995multivariate} summarised in \Cref{tab:spectral-decay} and by virtue of the big-$\mathcal{O}$ definition, we can always find $M$ so that for all $m\geq M$, $\lambda_m^2 \leq \lambda_m \leq \exp\{-(\alpha/D) x \}$.
    Further, we can always choose $M$ large enough so that the function $x \mapsto \exp\{-(\alpha/D) x \}$ is monotonically decreasing for all $x\geq M$. 
    Thus, 
    \begin{IEEEeqnarray}{rCl}
        \sum_{m=M+1}^{\infty}\lambda_m^2
        & \leq &
        \int_{M}^{\infty} \exp\{-(\alpha/D) x \} dx
        = \frac{D}{\alpha}\exp\{-(\alpha/D)M\}.
        \nonumber
    \end{IEEEeqnarray}

    Using similar arguments, in the case of Mat\'ern kernels with uniformly supported distributions on intervals, we can choose $M$ so that for all $m\geq M$, it also holds that we have $\lambda_m^2 \leq m^{-4(l+1)}\left( \log(m)^{2(D-1)(l+1)} \right)^2 \leq m^{-4l-3}$.
    Since $x\mapsto x^{-4l-3}$ is monotonically decreasing, it therefore holds that
    \begin{IEEEeqnarray}{rCl}
        \sum_{m=M+1}^{\infty}\lambda_m^2
        & \leq &
        \int_{M}^{\infty} x^{-4l-3} dx
        = \frac{1}{3(l+3)}M^{-3(l+1)}.
        \nonumber
    \end{IEEEeqnarray}

    For the general case, we follow the same arguments after noting that $\lambda_m = o(m^{-1})$ due to the definition of trace-class operators with spectral decomposition, for which both $\sum_{m=1}^{\infty}\lambda_m < \infty$ and $\lambda_m \geq 0$, so that $\lambda_m = o(m^{-1})$. 
    Hence, $\lambda_m^2 = o(m^{-2})$ so that for some $M$ and all $m \geq M$ it holds that
     \begin{IEEEeqnarray}{rCl}
        \sum_{m=M+1}^{\infty}\lambda_m^2
        & \leq &
        \int_{M}^{\infty} x^{-2} dx
        = M^{-1},
        \nonumber
    \end{IEEEeqnarray}
    which concludes the proof.
\end{proof}

\subsection{Proof of Lemma \ref{lemma:Gaussian-case-optimal}}


\begin{proof}
According to Theorem \ref{thm:optimal-measure-calQ} the potential of $\tau^*$ is given as ($h_m = e_m$)
\begin{align}
   V^*(u) &=  \bbE_{\xi} \Big[ \ell_N \big( \mu_u(x_{1:N}) + \sqrt{ \Sigma(x_{1:N})} \xi \big) \Big]  + \frac{1}{2} u^T \Big( \langle C h, h \rangle \Big)^{-1}u \\
   &=\frac{1}{2 \sigma^2} \bbE_{\xi} \left[||Y_{1:N} -\mu_u(x_{1:N}) + \sqrt{ \Sigma(x_{1:N})} \xi ||^2  \right] + \frac{1}{2} u^T \Lambda_M^{-1} u
\end{align}
where $\xi \sim \mathcal{N}(0,I_M)$ and $\Lambda_M =\text{diag}(\lambda_1, \hdots,\lambda_M)$. Note that 
\begin{align}
    Y_{1:N} -\mu_u(x_{1:N}) + \sqrt{ \Sigma(x_{1:N})} \xi \sim \mathcal{N} \big( y_{1:N}- \mu_u(x_{1:N}), \Sigma(x_{1:N}) \big)
\end{align}
and so by Equation 378 in \citet{petersen2008matrix} we obtain
\begin{align}
     V^*(u) &=  \text{tr}[\Sigma(x_{1:N}) ] + \frac{1}{2 \sigma^2} ||Y_{1:N}- \mu_u(x_{1:N}) ||^2 + u^T \Lambda_M^{-1} u \\
     &= \frac{1}{2 \sigma^2} ||Y_{1:N}- \mu_u(x_{1:N}) ||^2 + u^T \Lambda_M^{-1} u + \text{ const.} \\
     &= \frac{1}{2 \sigma^2} \sum_{n=1}^N (y_n - \mu_u(x_n))^2  + u^T \Lambda_m^{-1} u + \text{ const.} \\
     &= V_\infty(u) + \text{ const.}, 
\end{align}
where the last equality follows from $\mu_u(x_{1:N}) = e^{1:M}(X)^Tu$ which we calculated in \eqref{eq:mu_u-is-ortho}. We therefore conclude that $V^*$ and $V_\infty$ are the same up to an additive constant in dependent of $u$ which implies that $\tau^* = \Pi^{1:M}_{\infty}$. Since both $\Pi^*_M$ and $\Pi_{\infty}$ are contained in $\calQ_M$ and are parameterised by $\tau^*$ and $\Pi^{1:M}_{\infty}$, respectively, this immediately implies that $\Pi_{\infty}= \Pi^*_M$.
\end{proof}

\section{Projected Langevin Sampling for other inducing functions}\label{ap:sec:fidi-dynamics-projected}
The idea of PLS can be generalised to inducing functions other than $\widehat{e}^{1:M}$. To this end, let now $P : H_k \to H_M$ be the orthogonal projection onto $H_M := \text{span}\big( \{h_1, \hdots, h_M \}\big) \subset H_k$, $h=(h_1,\hdots,h_M)$ given as
\begin{align}
    P f =   h(\cdot)^T \langle h,h \rangle^{-1} \langle f, h \rangle. \label{eq:ortho-proj-op},
\end{align}
where $ ( \langle h, h \rangle )_{m,m'}=\langle h_m, h_{m'} \rangle$ and $(\langle f, h \rangle )_m = \langle f, h_m \rangle$ for all $m=1,\hdots,M$. We first prove that $P$ is indeed given by expression \eqref{eq:ortho-proj-op}. 
\begin{lemma}\label{Lemma:ortho-projection}
 The orthogonal operator $P:H_k \to H_M$ is given by expression \eqref{eq:ortho-proj-op}.
\end{lemma}
\begin{proof}
By definition of the orthogonal projection $P$ onto $H_M$ satisfies
\begin{align}
    \| f - P(f) \|_k = \argmin_{g \in H_M} \| f -g \|_k, 
\end{align}
where $\| \cdot \|_k$ is the norm induced by the RKHS inner product. Every element in $g \in H_M$ can, by definition, be written as 
\begin{align}
    g(x) = \sum_{m=1}^M \alpha_m h_m(x)
\end{align}
for $\alpha=(\alpha_1,\hdots, \alpha_M) \in \bbR^M$. Exploiting this fact leads to an finite-dimensional optimisation problem for $\alpha \in \bbR^M$ with solution 
\begin{align}
    \alpha^* =  \langle h ,h \rangle^{-1} \langle f, h \rangle.
\end{align}
and hence $P(f) = h(\cdot)^T \alpha^*$. 
\end{proof}

The next step is to derive an evolution equation for the inducing features $U:= (U_1, \hdots, U_M )^T$ with $U_m := \langle F ,h_m \rangle$ ($m=1,\hdots,M)$ from the projected Langevin SDE which is given as
\begin{align}
     d F(t) = -\left(  D \ell \big( P(F(t)) \big)  + C^{-1} F(t) \right) dt + \sqrt{2} d W(t) \label{eq:langevin-projected-infinite}.
\end{align}

\begin{theorem}\label{thm:SDE-for-U}
 Define $U(t) := \langle F(t) , h_m \rangle$ then $U(t)$ satisfies the SDE in $\bbR^M$ given as
\begin{align}
    d U(t) = &-   h(X) \big(\partial_2 c\big)\big(Y, P[F_t](X) \big) dt \nonumber 
    - \langle  C^{-1} F(t), h \rangle  dt +  \sqrt{2\langle h,h \rangle} d B(t).
\end{align}
\end{theorem}
Here, we denote $h(X) :=  \big( h_m(x_n) \big)_{m,n=1}^{M,N} \in \bbR^{M \times N}$, 
$\big(\partial_2(Y, P[F_t](X))\big)_n := \partial_2c(y_n, (P[F_t](X))_n) $, 
$P [F_t](X) = h(X)^T \langle h,h \rangle^{-1} U(t)$,
$(\langle C^{-1} F(t), h \rangle)_m = \langle C^{-1} F(t), h_m \rangle) $ for $m=1,\hdots,M$, $n=1,\hdots,N$. Furthermore $B(t)$ is a standard Brownian motion in $\bbR^M$ and $\sqrt{\langle h, h \rangle}$ the square-root of the matrix $\langle h , h \rangle$.

\ \\

\begin{proof}
We apply Ito's Rule to \eqref{eq:langevin-projected-infinite} \citep[Chapter 4.4]{da2014stochastic} with $\phi: H \to \bbR^M$ defined as $\phi(f) := \langle f, h \rangle$ which gives the dynamics
\begin{align}
    d U(t) &= d \phi\big( F(t) \big) \\
    &= - \left(  \Big\langle D \ell\big( P[F(t)] \big), h \Big\rangle + \big\langle C^{-1}F(t), h \big\rangle \right) dt + \sqrt{2\langle h,h \rangle} d B(t). \label{eq:ito-raw}
\end{align}
For a loss $\ell$ of the form \eqref{eq:loss-pointwise}, we can simplify further by applying the chain-rule for Fréchet derivatives
\begin{align}
     \Big\langle D \ell\big( P[F(t)] \big), h \Big\rangle  = h(X)^T \big(\partial_2 c\big)\big(Y, (P [F_t])(X) \big) \label{eq:rewrite1}
\end{align}
which concludes the proof.
\end{proof}

The SDE for $U(t)$ in Theorem \eqref{thm:SDE-for-U} can not be simulated, since we do not have access to $\langle C^{-1} F(t), h \rangle$. However, for the specific choice $h_m = k(z_m, \cdot)$ where $Z=(z_1,\hdots,z_M)$ are samples from $\nu$, we can find an approximation for this term.

\begin{lemma}\label{eq:approximation-cov-op}
Let $z_1,\hdots, z_m$ be iid samples from $\nu$, $h = k_Z(\cdot)$ and let $C$ be the covariance operator defined in \eqref{eq:def-cov-op}. Then 
    \begin{align}
        \langle C^{-1} f, k_Z \rangle \approx M k(Z, Z)^{-1}  f(Z)
    \end{align}
for all $f \in H_k$.
\end{lemma}

\begin{proof}
Let now $h_m = k(z_m, \cdot)$ and hence $h= k_Z(\cdot)$. By reproducing property we have 
\begin{align}
    \langle C^{-1} f , h \rangle = C^{-1}f(Z).
\end{align}
We now use a standard Monte Carlo approach by assuming that
\begin{align}
    \frac{1}{M}\sum_{m=1}^{M} \delta_{z_m} \approx \nu.
\end{align}
Consequently, we approximate $C: H \to H$ via for all $x \in \calX$ via
\begin{align}
    Cf(x) &= \int k(x,x') f(x') \, d\nu(x') \\
    &\approx \frac{1}{M} \sum_{m=1}^M k(x, z_m) f(z_m) \\
    &=:\frac{1}{M} k(x, Z) f(Z)
\end{align}
which leads to $Cf(Z) \approx \frac{1}{M} k(Z,Z) f(Z)$ and ultimately to
\begin{align}
    (C^{-1}f)(Z) \approx M k(Z,Z)^{-1} f(Z) \label{eq:C-inv-approx}.
\end{align}
This concludes the proof.
\end{proof}

The combination of Theorem \ref{thm:SDE-for-U} and Lemma \ref{eq:approximation-cov-op} leads to a (approximate) SDE for $U(t)$ given as 
\begin{align}
    d U(t) = &-   k(Z,X) \big(\partial_2 c\big)\big(Y, k(X,Z) k(Z,Z)^{-1} U(t) \big) dt \nonumber \\
    &- M k(Z,Z)^{-1} U(t) dt +  \sqrt{2k(Z,Z)} dB(t) \label{eq:SDE-U-vector-appendix}
\end{align}
which can be simulated in $\bbR^M$. For large $t>0$ we hope that $U(t) \approx U | Y$. Furthermore, by Matheron's Rule, we can transform the samples from $U$ into posterior samples (cf. Appendix \ref{sec:ap:matherons-rule}). The covariance matrices are given as
\begin{align}\label{eq:cov-calcus-IP}
&\mathbb{C}( \langle F , k_{x^*}\rangle, \langle F , k_{x^*}\rangle) = r(x^*,x^*), 
&&\mathbb{C}(\langle F , k_{x^*}\rangle, U) = r(x^*,Z)   &&&\mathbb{C}[U,U] = r(Z,Z),
\end{align}
since $G(x):=\langle F , k(x,\cdot) \rangle$ is a GP with kernel $r$ (cf. Lemma \ref{lemma:GRE-GP-analogy}). 

Furthermore, we can also obtain the asymptotic distribution of the SDE $U(t)$ in closed form, since it is a preconditioned Langevin equation.

\begin{theorem}
Let $\big(U(t)\big)_{t \ge 0}$ be the solution to
the SDE in \eqref{eq:SDE-U-vector-appendix} with $U(0)=U_0$ for a given initial value $U_0 \in \bbR^M$. Then 
$
    U(t) \overset{\mathcal{D}}{\longrightarrow} \widehat{Q}_U
$
for $t \to \infty$ where $\widehat{Q}_U(du) = \widehat{q}_U(u) du$ with
\begin{align}
 \widehat{q}_U(u) \propto \exp\left( - \sum_{n=1}^N c\big(y_n, k_Z(x_n)^T k(Z,Z)^{-1}  u \big) - \frac{M}{2} u^T k(Z,Z)^{-2} u  \right),    
\end{align}
for $u \in \bbR^M$. 
\end{theorem}

\begin{proof}
First, define $A:=  k(Z,Z) \in \bbR^{M \times M}$. Then the SDE in \eqref{eq:SDE-U-vector-appendix} can be rewritten as 
\begin{align}
    dU(t) = - A \Big( A^{-1} k_Z(X) \partial_2 \big(Y, k_Z(x_n)^T A^{-1} U(t)\big) + M A^{-1}  
    k(Z,Z)^{-1}
    U(t)  \Big) dt  +  \sqrt{2 A} d \beta(t), \label{eq:rewritten-sde}
\end{align}
Further, we define the potential $V : \bbR^M \to \bbR$ as 
\begin{align}
    V(u) = \sum_{n=1}^N c(y_n, k_Z(x_n)^T A^{-1}u) + \frac{M}{2} u^T A^{-2} u
\end{align}
and calculate the gradient $\nabla V$ as 
\begin{align}
    \nabla V (u) =A^{-1} k(Z,X)  (\partial_2 c)(Y, k(X,Z) A^{-1} u ) + M A^{-2} u.
\end{align}
The SDE \eqref{eq:SDE-U-vector-appendix} can therefore be rewritten as 
\begin{align}
    dU(t) =  - A \nabla V \big( U(t) \big) dt + \sqrt{2A} d\beta(t)
\end{align}
which is the preconditioned Langevin diffusion with potential $V$. It is known \citep{bhattacharya2023fast} that 
\begin{align}
    U(t) \overset{\mathcal{D}}{\longrightarrow} \frac{1}{\kappa} \exp( - V(u) )
\end{align}
for $t \to \infty$ where $\kappa:= \int \exp(-V(u)) \, du $ is the normalising constant.
\end{proof}

\section{Implementation Details}\label{ap:sec:implementation-details}

\subsection{Hyperparameter selection}

For fair comparison, we shared the same hyperparameters for $r$ across PLS and SVGP. We used an ARD kernel for $k$ when constructing $r$, tuning the hyperparameters of $k$. For data sets with $N\leq2000$ observations (or $N\leq1000$ for classification), we learned the hyperparameters of $k$ by maximising the exact marginal log-likelihood of a GP with $k$ and a mean zero prior. For a data set with more than $2000$ ( or $1000$ for classification) observations we use the following heuristic taken from \citet{{lin2024sampling}}:

\begin{enumerate}
    \item Randomly select a centroid uniformely at random the training data.
    \item Select a subset of size $2000$ (or $1000$ for classification) with the smallest Euclidean distance to the centroid.
    \item Learn kernel hyperparmeters on the data through regression GP marginal likelihood with kernel $k$ and mean zero prior for this subset.
\end{enumerate}
Repeat the above procedure $10$ (or $5$ for classification) and average the learned hyperparameters

The inducing points $z_1, \dots z_M \in \calX$ used for the Nyström method in Appendix \ref{sec:ap:nystroem-method} were selected following the greedy variance selection method in \cite{burt2020gpviconv} and \cite{chen2018fast}.

\subsection{Projected Langevin Sampling Algorithm}
Let $\eta >0$ be the step-size and $T >0$ be the the end time of our simulation. Define now $U(t):= \widehat{F}^{1:M}(t)$ where $\widehat{F}^{1:M}(t)$ is the solution to the SDE \eqref{eq:approx-langevin-components:ap} and $t_i := i \eta$ for $i=0,\hdots, I$ with $I = \lfloor T/\eta \rfloor$ and further
\begin{align}
    &\widehat{U}(0) \sim Q_0 \\
    &\widehat{U}(t_{i+1}) : = \widehat{U}(t_i) 
    -  \eta \widehat{e}^{1:M}(X) \big(\partial_2 c\big)\big(Y, \widehat{e}^{1:M}(X)^T  \widehat{U}(t_i) \big) -\eta  \widehat{\Lambda}^{-1} \widehat{U}(t_i)  +  \sqrt{2\eta} \xi_i \label{eq:SDE-U-discrete},
\end{align}
where $\eta >0$ is the step size and $\{\xi_i\}_{i=1}^I$ are i.i.d. $\mathcal{N}(0, I_M)$. Here, we define 
\begin{align}
    (\partial_2 c)(Y, \widehat{Y}) &:= \Big(\partial_2 c( y_n , \widehat{y}_n) \Big)_{n=1}^N \in \bbR^N, && \widehat{e}^{1:M}(X):= \Big( \widehat{e}^m(x_n) \Big)_{m,n=1}^{M,N} \in \bbR^{M \times N}
\end{align}
for all $Y,\widehat{Y} \in \bbR^N$ and $X=(x_1,\hdots,x_N) \in \calX^N$.

Notice that \eqref{eq:SDE-U-discrete} is the Euler-Maruyama discretisation of the SDE \eqref{eq:approx-langevin-components:ap} and therefore $\widehat{U}(t_i) \approx U(t_i)=\widehat{F}^{1:M}(t_i)$ for small enough $\eta$.

\subsection{Likelihood functions}\label{likelihood-functions}

\paragraph{Regression} As discussed in Section \ref{sec:BKR-BKC} the Gaussian likelihood $p(y|f) = \mathcal{N}(f(X), \sigma^2 I_N)$ corresponds to the choice $c: \bbR \times \bbR \to \bbR$ with $c(y,\widehat{y}) = \frac{1}{2\sigma^2} (y-\widehat{y})^2$ and consequently
\begin{align}
    \partial_2 c(y, \widehat{y}) = \frac{1}{\sigma^2}( \widehat{y}-y).
\end{align}

\paragraph{Classification} In binary classification, we assume a Bernoulli data, i.e.
\begin{align}
    Y_n | F= f \sim \text{Bernoulli}\big(\phi(f(x_n) \big).
\end{align}
This gives rise to the cost $c: \{0,1\} \times \bbR \to \bbR$ with
\begin{align}
    c(y, \widehat{y}) = - \log p(y|f) = - y \log \phi(\widehat{y}) - (1-y) \log\big(1- \phi( \widehat{y})\big),
\end{align}
where $\widehat y = f(x)$. For our experiments we use the logistic function $\phi(\widehat{y}) = \big(1+ \exp(-\widehat{y}))^{-1} $. Due to the well-known property $\phi'(\widehat{y}) = \phi(\widehat{y}) \big( 1- \phi(\widehat{y}) \big)$ we obtain
\begin{align}
    \partial_2 c(y, \widehat{y}) =  - y \big(1- \phi(\widehat{y})\big) + (1-y) \phi(\widehat{y})
\end{align}
for $y\in\{0,1\}$, $\widehat{y} \in \bbR$.

\paragraph{Poisson Model} 
We choose a Poisson regression model where we assume 
\begin{align}
   Y_n | F=f \sim \text{Poisson}\big(( f(x_n)^2 \big).
\end{align}
This gives rise to the cost 
\begin{align}
    c(y_n, f(x_n) ) &= - 2 y_n \log |f(x_n) | + \big( f(x_n) \big)2 \\
    \partial_2 c(y_n, f(x_n))&= - 2 y_n \frac{\text{sign}(f(x_n))}{| f(x_n)|} + 2 f(x_n) \\
    &=- \frac{2y_n}{f(x_n)} + 2 f(x_n)
\end{align}
for $y_n \in \bbN_0$ and $f(x_n) \in \bbR$.

\paragraph{Multimodal Regression with unknown shift}

Assume that the data $Y$ is generated as 
$$
Y_n | F= f  \sim f(x_n) + Z \cdot c + \sigma \epsilon
$$
where $c \in \mathbb{R}$, $Z \sim B(\alpha)$, $\epsilon \sim \mathcal{N}(0,1)$ and $\sigma > 0$. Essentially, there is a latent variable $Z$ which is unbobserved that cases an offset $c$. This leads to a mixture likelihood for $Y$ given as
$$
p(y_n | f(x_n) )= \alpha \mathcal{N}(y_n | f(x_n)+c, \sigma^2) + (1-\alpha) \mathcal{N}(y_n | f(x_n), \sigma^2)
$$
Consequently, we obtain the cost
$$
c(y_n,f(x_n)) = - \log p(y_n|f(x_n)) = - \log  \Big( \alpha \mathcal{N}(y_n | f(x_n)+c, \sigma^2) + (1-\alpha) \mathcal{N}(y_n | f(x_n), \sigma^2) \Big).
$$
and the partial derrivative with respect to second component is
\begin{align*}
\partial_2 c(y_n, f(x_n)) = & - \frac{1}{\sqrt{2\pi} \sigma^3 p(y_n|f(x_n))} \Big( 
    \alpha (y_n - f(x_n) - c) \exp\Big(- \frac{(y_n - f(x_n) - c)^2}{2 \sigma^2}\Big) \\
    & + (1-\alpha) (y_n - f(x_n)) \exp\Big(- \frac{(y_n - f(x_n))^2}{2 \sigma^2}\Big)
\Big).
\end{align*}

\SetKwInput{Initialise}{Initialise} 
\SetKwInput{Input}{Input}

\begin{algorithm}
  \caption{Projected Langevin Sampling}
  \Input{input data $x_{1:N}$, targets $y_{1:N}$, 
  kernel $k$, 
  inducing points $z_{1:M}$,
  step size $\eta >0$,
  time horizon $T$,
  initialisation prob. measure $Q_0 \in \mathcal{P}(\calX)$,
  number of samples $J$,
  new points $x_{1:N_*}$
  }
  \KwResult{Samples $F_1(x_{1:N_*}), \hdots, F_J(x_{1:N_*}) \approx F(x_{1:N_*})|y_{1:N}$}
  \For{$j=1,\hdots,J$}{
    Initialise $\widehat{U}_j(t_0) \sim Q_0$ \\
    \For{$i=0,\hdots,T/\eta$}{
     Generate $\widehat{U}_j(t_{i+1}) $ from $\widehat{U}_j(t_i)$ according to the update rule in \eqref{eq:SDE-U-discrete}
    }
    Sample $\big( G_j(X^*), \langle G_j,h \rangle\big) \sim \mathcal{N}(0, R_{N_*,M})$ with 
    \begin{align} \label{eq:R-matrix}
    R_{N_*, M} :=
        \begin{bmatrix}
            r(x_{1:N_*}, x_{1:N_*}) &  \widehat{e}^{1:M}(x_{1:N_*})^T \widehat{\Lambda}_M  \\
            \widehat{\Lambda}_M \widehat{e}^{1:M}(x_{1:N_*}) &  \widehat{\Lambda}_M
        \end{bmatrix}   \in \bbR^{(N_*+M) \times (N_*+M)}
    \end{align}
    Calculate
     \begin{IEEEeqnarray}{rCl}
        F_j(x^*_{1:N_*}) = G_j(x^*_{1:N_*}) + \widehat{e}^{1:M}(x^*_{1:N_*})^{\top} \Big( \widehat{F}^{1:M}_j(T) - \langle G_j, \widehat{e}^{1:M} \; \rangle \Big)
        \nonumber
    \end{IEEEeqnarray}
    Here,  $\widehat{e}^{1:M}(x^*_{1:N_*}) \in \bbR^{M \times N_*}$  is the matrix whose entry at  $(m,n)$ is  $\widehat{e}_m(x^*_n)$ and $\widehat{\Lambda}_M := \text{diag}( \widehat{\lambda}_1, \hdots, \widehat{\lambda}_M) \in \bbR^{M \times M}$ is the diagonal matrix with entries $\widehat{\lambda}_1, \hdots, \widehat{\lambda}_M$ (see Appendix \ref{sec:ap:nystroem-method} for a definition).
  }
  \label{alg:particle-method}
\end{algorithm}

\subsection{Time and Space Complexity} 
\label{sec:time-and-space-complexity}

This section discusses the time and space complexity requirements of producing $J \in \bbN$ posterior samples from $F|Y$ with our method.

\paragraph{Training} 
We calculate the spectral decomposition of $\frac{1}{M} k(z_{1:M},z_{1:M})$ and store the result which can be done in $\mathcal{O}(M^3)$. The update in \eqref{eq:SDE-U-discrete} requires only matrix multiplications which are each dominated by $\mathcal{O}(NM)$.

The total costs therefore are $\mathcal{O}\big(M^3 + J NM\big)$. These costs could be reduced further by batch-approximations of the gradient. However, the data set considered in this paper serve illustration purposes only and a batch-approximation was not required.

\paragraph{Prediction} Let $x_{1:N_*} \in \calX^{N_*}$ be a set of input points for which we want to generate posterior samples, i.e. $F_j(x_{1:N_*}) \approx F(x_{1:N_*}) | y_{1:N}$. Generating one sample in \eqref{eq:R-matrix} requires jointly sampling $\big( G_j(X^*), \langle G_j, h \rangle \big) $ from a multivariate Gaussian which is $\mathcal{O}\big((N_* + M)^3\big)$  for calculating the Cholesky (or the spectral) decomposition once and then $\mathcal{O}\big((N_* + M)^2\big)$ for generating samples. The matrix multiplication is $\mathcal{O} (N_* M )$. Hence we pay $\mathcal{O}\big((N_*+M)^3\big)$ upfront for the Cholesky decomposition and then $\mathcal{O}\big( J (N_*+M)^2 )$.

For special kernels, this could be further improved by using the exploiting the ideas discussed in \citet{wilson2020efficiently}.

\paragraph{Space Complexity} Space complexity in our implementation is $\mathcal{O}(NM + M^2)$ to store the $k(z_{1:M},x_{1:N})$ and $k(z_{1:M},z_{1:M})$ matrix.

\bibliography{biblio.bib}

\end{document}